\documentclass[accepted]{style}
\usepackage[numbers]{natbib}

\newcommand{\score}{\textsc{SCORE}}

\usepackage[utf8]{inputenc}
\usepackage[T1]{fontenc}
\usepackage{url}
\usepackage{booktabs}
\usepackage{amsfonts}
\usepackage{nicefrac}
\usepackage{microtype}
\usepackage{lipsum}
\usepackage{graphicx}
\usepackage{float}
\usepackage{float}
\usepackage{dsfont}
\usepackage{doi}
\usepackage{subcaption}
\usepackage{xcolor}
\usepackage{array}
\usepackage{makecell}
\usepackage{longtable}
\usepackage{enumitem}
\usepackage{wrapfig}
\usepackage{multirow}
\usepackage{amsmath}
\usepackage{amssymb, amsthm}
\usepackage{mathtools}
\usepackage{bm}
\usepackage{natbib}
\usepackage{siunitx}
\usepackage{placeins}
\usepackage{adjustbox}

\theoremstyle{plain}
\newtheorem{theorem}{Theorem}[section]
\newtheorem{proposition}[theorem]{Proposition}
\newtheorem{lemma}[theorem]{Lemma}

\newtheorem{remark}{Remark}[section]

\theoremstyle{definition}

\newtheorem{assumption}{Assumption}[section]

\newcommand{\R}{\mathbb{R}}
\newcommand{\C}{\mathbb{C}}
\newcommand{\Gauss}{\mathcal{N}}
\newcommand{\E}{\mathbb{E}}
\newcommand{\bmu}{\bm{\mu}}
\newcommand{\bSigma}{\bm{\Sigma}}
\newcommand{\bx}{\bm{x}}

\newcommand{\bt}{\bm{t}}
\newcommand{\be}{\bm{e}}

\newcommand{\br}{\bm{r}}
\newcommand{\bw}{\bm{w}}

\newcommand{\bv}{\bm{v}}
\newcommand{\bB}{\bm{B}}
\newcommand{\bX}{\bm{X}}
\newcommand{\by}{\bm{y}}
\newcommand{\bz}{\bm{z}}
\newcommand{\bU}{\bm{U}}
\newcommand{\bM}{\bm{M}}
\newcommand{\bI}{\bm{I}}
\newcommand{\bL}{\bm{L}}
\newcommand{\bD}{\bm{D}}
\newcommand{\bF}{\bm{F}}

\newcommand{\bA}{\bm{A}}
\newcommand{\bW}{\bm{W}}
\newcommand{\bY}{\bm{Y}}
\newcommand{\bT}{\bm{T}}
\newcommand{\bS}{\bm{S}}

\newcommand{\tr}{\operatorname{tr}}
\newcommand{\diag}{\operatorname{diag}}

\newcommand{\bE}{\mathbb{E}}

\newcommand{\bP}{\mathbb{P}}

\newcommand{\cN}{\mathcal{N}}
\newcommand{\cP}{\mathcal{P}}

\newcommand{\cF}{\mathcal{F}}
\newcommand{\cY}{\mathcal{Y}}
\newcommand{\cX}{\mathcal{X}}
\newcommand{\btheta}{\bm{\theta}}
\newcommand{\bDelta}{\bm{\Delta}}
\newcommand{\bLambda}{\bm{\Lambda}}
\newcommand{\bone}{\bm{1}}  \newcommand{\Unif}{\mathcal{U}}

\newcommand{\prerank}{\rho}

\usepackage{etoc}

\usepackage{tikz}

\title{SCORE: Spectral Correlation Estimation for Multivariate Gaussians}

\author[1,2]{Christopher Bülte\thanks{Equal contribution}}
\author[1,2]{Emil Partow$^*$}
\author[1]{Astha Gupta}
\author[1,2]{Pascal Esser}
\author[1,2,3,4]{Gitta Kutyniok}

\affil[1]{
LMU Munich
}
\affil[2]{%
Munich Center for Machine Learning (MCML)
}
\affil[3]{%
DLR-German Aerospace Center
}
\affil[4]{%
University of Troms\o
}

\begin{document}
\maketitle
\etocdepthtag.toc{main}

\begin{abstract}
Neural network-based predictive modeling with high-dimensional structured Gaussian targets requires an efficient and numerically stable, yet expressive approximation of the covariance matrix. We propose SCORE: a scalable framework, combining scoring rule training with an expressive covariance approximation learned in spectral space.
For $d$-dimensional data, the learning task is decomposed into learning the marginal distributions and learning a structured correlation matrix, which enables dense dependencies with linear storage and $\mathcal{O}(d\log d)$ cost.
We utilize the closed form Gaussian kernel score for training, which remains defined even for degenerate covariances and admits bounded gradients during optimization. We characterize kernel scores under invertible transforms and prove exact invariance under unitary transforms. At population level, our two-level objective recovers the true marginals and projects the target correlation onto the representable class; finite-sample PAC bounds show that the errors of the two stages enter additively. We evaluate our model on a variety of tasks with a commonly assumed Gaussian domain: Time-series forecasting, monocular depth estimation, and spatial weather prediction, showing improved performance at lower computational cost.
\end{abstract}

\section{INTRODUCTION}
Probabilistic predictions are essential when models inform decisions under uncertainty and have become increasingly important in machine learning \citep{gneiting2014, 3295222.3295309}. This is particularly important for high-dimensional structured outputs such as spatial fields or multivariate time series, whose uncertainties are correlated across coordinates \citep{dorta2018structured,salinas2019high,liu2023single}.
\begin{figure*}[t!]
    \centering
    \includegraphics[width=\linewidth]{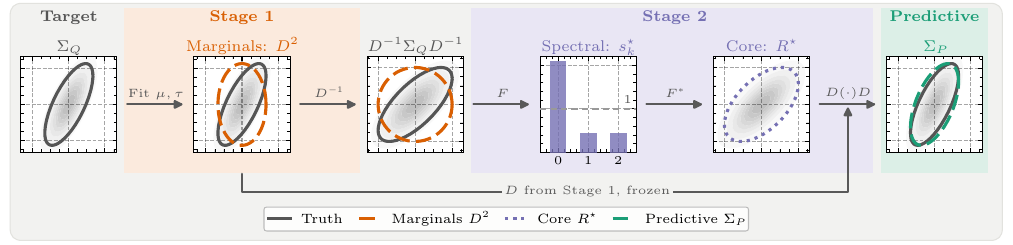}
\caption{Schematic overview of \score{}. Stage~1 fits the mean $\bmu$ and the marginal variances $\bm\tau$ (orange) and standardizes the target, leaving only its correlation structure. Stage~2 models this correlation by a normalized spectrum $s_k^\star$ in Fourier coordinates, whose back-transform $F^*$ is the correlation core $R^\star$ (purple). Rescaling by the frozen marginals yields the predictive covariance $\bSigma_P$ (green).}
    \label{fig:method_overview}
\end{figure*}
This setting arises naturally from measurement noise, aggregation effects, or from inherently stochastic systems. Therefore, it is widely used in scientific domains, including atmospheric forecasting \citep{gneiting2005calibrated}, oceanographic fields \citep{kuusela2018locally}, biological measurements \citep{schafer2005empirical}, and astronomical time series \citep{foremanmackey2017fast}. 

In practice, one often assumes a Gaussian predictive, which presents two major challenges in high dimensions.
First, a dense covariance matrix requires $\mathcal{O}(d^2)$ parameters and generic matrix operations up to $\mathcal{O}(d^3)$ time. Diagonal models scale linearly but discard any dependencies, while low-rank and structured alternatives restrict the representable correlations. Second, the commonly used negative log-likelihood (NLL) \citep{374138, lakshminarayananSimpleScalablePredictive2017} is undefined for singular covariances and ill-conditioned near the boundary of the positive-definite cone, where its gradients may become arbitrarily large \citep{seitzer2022pitfalls}.

We address these complexities by combining a robust proper scoring rule with a covariance approximation that separates marginal uncertainty from dependence structure. As the multivariate objective, we employ the Gaussian kernel score, which is directly related to the maximum mean discrepancy \citep{GneitingRaftery2007,steinwart2021strictly,gretton2012kernel}, and available in closed form for Gaussian predictions. Unlike the NLL, it remains well-defined for degenerate distributions and admits bounded gradients. At the same time, it is invariant under unitary transformations of the domain, allowing us to learn the dependence structure in a spectral basis, e.g., the Fourier basis, leading to a more computationally efficient representation without altering the optimization objective. 
Further, we employ a staged estimation approach, which first learns the marginals and recovers the optimal marginal distributions at a population level. 
In the second stage, we employ the spectral transformation to efficiently approximate the remaining correlation structure, for which we provide an explicit error bound.

Our staged procedure is summarized in \autoref{fig:method_overview}. First, the marginal distributions are estimated and used to normalize the target. In this transformed domain, we fit a circulant correlation matrix, using an efficient representation in the Fourier domain. Finally, rescaling the learned correlation by the marginal scales yields the approximate predictive covariance matrix.

For the discrete Fourier transform (DFT) and a diagonal covariance in the spectral domain, referred to as \emph{core}, the model learns a spectrum whose inverse transform is a dense circulant correlation matrix. Combining the core with the learned marginals admits an expressive dependence structure with linear covariance storage and FFT-based $\mathcal{O}(d\log d)$ operations. More generally, this applies to any unitary transform and admissible core class. We show that the population objective recovers the true marginals and selects the best representable correlation, characterize its approximation error, and derive finite-sample PAC bounds in which the errors from the two stages enter additively.

\textbf{Contributions.} Our main contributions are:

\emph{1.} We propose to use the Gaussian kernel score as a closed form training objective and establish corresponding boundedness and Lipschitz properties, contrasting its regularity with the singular behavior of the NLL.

\emph{2.} We introduce \score{} (Spectral CORrelation Estimation), a covariance parameterization combining learned marginal scales with a structured correlation core in a spectral domain, characterize its expressivity and population optimum, and provide conditions for exact recovery of the true distribution.

\emph{3.} We evaluate the framework on real-world high-dimensional spatiotemporal prediction tasks, examining predictive performance, calibration, numerical stability, and computational efficiency.


\section{BACKGROUND}
Throughout the paper, let $\mathcal{X}\subseteq\R^{k}$ and $\mathcal{Y}\subseteq\R^{d}$ denote the input and output spaces, respectively, and let $\mathcal{P}(\mathcal{Y})$ be a convex set of probability measures on $\mathcal{Y}$. 
We consider a training dataset $\mathcal{D}=\{(\bx_n,\by_n)\}_{n=1}^{N}\subset(\mathcal{X}\times\mathcal{Y})^N$ consisting of i.i.d.\ samples from an unknown joint distribution $\bP=\bP_{\bX,\bY}$ with marginal $\bP_{\bX}$ and conditional law $\bP_{\bY|\bX}(\cdot\mid\bx)$. 
Expectations with respect to $\bP$ are written as $\bE_{\bX,\bY\sim\bP}[\cdot]$, while $P,Q\in\mathcal{P}(\mathcal{Y})$ denote generic probability distributions. Our predictive model is a neural network $f_\phi:\mathcal{X}\rightarrow\mathcal{P}(\mathcal{Y})$ with parameters $\phi\in\mathcal{W}\subseteq\R^q$, which maps an input $\bx$ to a probability distribution $f_\phi(\bx)$.
Unless stated otherwise, we restrict our attention to Gaussian predictive distributions, $f_\phi(\bx)=\mathcal{N}\!\left(\bmu_\phi(\bx),\,\bSigma_\phi(\bx)\right)$, where $\bmu_\phi(\bx)\in\R^d$ and $\bSigma_\phi(\bx)\succ0$ denote the predicted mean vector and positive definite covariance matrix, respectively. 
A full  overview is given in Appendix~\ref{app:notation}.

\subsection{Scoring rule minimization}
Predictive models that output distributional parameters are typically fitted by minimizing a proper scoring rule \citep{https://doi.org/10.1111/sjos.12168}, most commonly the log score (NLL) \citep{374138, lakshminarayananSimpleScalablePredictive2017}. 
A scoring rule is a function $S: \cP(\cY) \times \cY \to \overline{\R}$ that assigns a numerical score $S(P, \by)$ to a probabilistic prediction
$P \in \cP(\cY)$ when $\by$ is observed. Following
\citet{GneitingRaftery2007}, $S$ is \emph{proper} relative to a class
$\cF \subseteq \cP(\cY)$ if 
\begin{equation}
\E_{\bY \sim Q}[S(Q, \bY)] \le \E_{\bY \sim Q}[S(P, \bY)],\ \forall P, Q \in \cF,
\end{equation}
and \emph{strictly proper} if equality holds only when $P = Q$. Defining $S(P,Q) := \bE_{\bY \sim Q}[S(P,\bY)]$, any scoring rule $S$ gives rise to a corresponding entropy $H(P) := S(P, P)$ and divergence $D(P,Q) := S(P, Q) - H(Q)$.

Utilizing scoring rules, the training objective of a neural network can be expressed as
\begin{align}
\label{eq:sr_objective}
&\mathrm{argmin}_\phi \bE_{\bX,\bY \sim \bP}\left[S(f_\phi(\bX), \bY) \right]\nonumber \\
=& \mathrm{argmin}_\phi \bE_{\bX \sim \bP_{\bX}} \left[ D(f_\phi(\bX), \bP_{\bY|\bX}) \right].
\end{align}
The above states that optimizing a scoring rule is equivalent to minimizing the score divergence between the predictive model and the true conditional distribution. As mentioned before, this setting includes the commonly used strictly proper log score ${S_\mathrm{log}(P, \by) \coloneq - \log p(\by)}$, or NLL. In this case, the corresponding divergence is the KL divergence, recovering the cross-entropy minimization objective.

\subsection{Parameterizing a Multivariate Gaussian}
\label{sec:param_multi_gauss}
Several works focus on ways to approximate a covariance matrix such that it can be parameterized and optimized using a neural network. We revisit three important and commonly used parameterizations here:

The simplest choice is just a diagonal covariance $\bSigma$ (Diag), which disregards any correlation structure, with ${\bSigma = \diag(\sigma_1^2, \ldots, \sigma_d^2)}$, and $\sigma_k^2 > 0$. Since this corresponds to $d$ independent univariate Gaussians, one per dimension, the approach is equivalent to the mean-variance networks proposed by \citet{374138}.

On the other hand, one can consider the Cholesky decomposition (Chol), which decomposes any $\bSigma\succ0$ into $\bSigma=\bL\bL^\top$, where $\bL\in\R^{d\times d}$ is lower triangular with a positive diagonal. While this has been successfully combined with neural networks \citep{10.1162/neco.1996.8.4.843, 9457247} and offers full expressivity of $\bSigma$, it comes at a very high computational cost and is infeasible in high dimensions.

A middle ground is given by the low-rank-plus-diagonal approximation (LorD), given as $\bSigma
    \;=\;
    \bD + \bU\bU^\top$, where $\bD = \diag(\sigma_1^2, \ldots, \sigma_d^2) \in \R_{>0}^{d \times d}$
and $\bU \in \R^{d \times r}$ with $1 \leq r \ll d$.
While often used in combination with variational inference \citep{Ong03072018, NEURIPS2020_310cc7ca}, it has also been used in conditional prediction settings \citep{salinas2019high, 3495724.3496794}.

\begin{table}[ht]
\centering\scriptsize
\caption{Cost of covariance parameterizations: number of covariance parameters, one log-score evaluation, one sample, and estimation of one covariance entry $\hat{\Sigma}_{ij}$.}
\label{tab:complexity}
\smallskip
\begin{tabular}{@{}lcccc@{}}
\toprule
    & \textbf{Params}
    & \textbf{Score eval.}
    & \textbf{Sampling}
    & $\hat{\Sigma}_{ij}$\\
\midrule
Chol
    & $d(d+1)/2$
    & $\mathcal{O}(d^{2})$
    & $\mathcal{O}(d^{2})$
    & $\mathcal{O}(d)$\\
Diag
    & $d$
    & $\mathcal{O}(d)$
    & $\mathcal{O}(d)$
    & $\mathcal{O}(1)$\\
LorD
    & $d + dr$
    & $\mathcal{O}(dr^{2} + r^{3})$
    & $\mathcal{O}(dr)$
    & $\mathcal{O}(r)$\\
\midrule
SCORE (Ours)
    & $2d$
    & $\mathcal{O}(d\log d)$
    & $\mathcal{O}(d\log d)$
    & $\mathcal{O}(d)$\\
\bottomrule
\end{tabular}
\end{table}

Table~\ref{tab:complexity} summarizes the computational cost of these methods and of our proposed method \score{}, introduced in Section~\ref{sec:methodology}. Further approximations exist, which we do not consider here. A noteworthy one is the Kronecker product, common in curvature approximation \citep{pmlr-v37-martens15, ritter2018a}. It factorizes a structured domain, for example, spatiotemporal, into its coordinates, but each factor again requires one of the above approximations, which is the reason we do not consider it here.

\subsection{Training with Negative Log-likelihood}
\label{sec:nll}
Regardless of the parameterization, previous work focuses almost exclusively on optimizing or adapting $S_\mathrm{log}$, which admits a closed form
\begin{align}
    \label{eq:log_score_gaussian}
S_\mathrm{log}(P, \by) =\frac{d}{2}\log(2\pi)+\frac{1}{2}\,\log\det\bSigma_P
+ \frac{1}{2}\bM, 
\end{align}
with $\bM \coloneq(\by - \bmu_P)^\top \bSigma_P^{-1}\,(\by - \bmu_P)$. However, already in the univariate case, the corresponding optimization problem is ill-conditioned and leads to diverging gradients \citep{seitzer2022pitfalls, pmlr-v206-stirn23a}. In the multivariate case, the log score becomes even more unstable.

First, note that for existence $S_\mathrm{log}$ requires ${\mathrm{rank}(\bSigma) = d}$. However, in high dimensions, we can expect $\mathrm{rank}(\bSigma)<d$, since the true distribution might live on a proper subspace. In that case, while the Gaussian distribution can still be defined meaningfully, \eqref{eq:log_score_gaussian} is undefined. This can occur along any rank-deficient direction, not only in the marginals.
Furthermore, $S_\mathrm{log}$ suffers from the additional drawback that it admits unbounded gradients and therefore is not Lipschitz continuous, which leads to unstable training with neural networks.
\begin{lemma}[Unbounded gradients of $S_\mathrm{log}$]
\label{lemma:log_score_gradients}
    Consider a predictive Gaussian $\cN(\bmu, \bSigma)$, with Cholesky decomposition $\bSigma = \bL\bL^\top$, and fixed $\by \in \R^d$. 
    Then $\nabla_{\bmu} S_\mathrm{log}$ and $\nabla_{\bL} S_\mathrm{log}$ are
    unbounded over $(\bmu, \bL)$ as $\lambda_\mathrm{min}(\bSigma) \to 0$, and
    $S_\mathrm{log}$ is not (globally) Lipschitz continuous in $(\bmu,\bm{L})$.
\end{lemma}
The proof is given in Appendix~\ref{app:loss_gradients}. In general, even for $\bSigma = \sigma^2 \bI$, $S_\mathrm{log}$ is not Lipschitz, since the univariate log score is not Lipschitz in $\sigma$.
While one can circumvent these issues, for example, by adding $\varepsilon \bI$ with $\varepsilon>0$, this still leads to numerical issues for small $\varepsilon$ and to an overestimation of the aleatoric uncertainty in the predictive distribution for large $\varepsilon$, as highlighted in \autoref{fig:toy_example1_full} and Appendix~\ref{app:toy_example}.
%
%
\section{METHODOLOGY}
\label{sec:methodology}

\subsection{Robust Loss Functions}
To circumvent the caveats of the log score, we propose to instead use kernel scores as loss functions, which offer a more robust alternative. They have been extensively used in forecast evaluation \citep{GneitingRaftery2007} and more recently as neural network loss functions~\citep{pacchiardiProbabilisticForecastingGenerative2021, kelenDistributionFreeDataUncertainty2024}. Let $k \colon \R^d \times \R^d \to \R$ be a positive definite kernel. The \emph{kernel score} of a distribution $P \in\cP(\cY)$ evaluated at observation
    $\by \in \cY$
    \begin{align*}
        S_k(P,\by)= \tfrac12\left(\E_{\bz, \bz' \sim P}\bigl[k(\bz, \bz')\bigr] 
        + c\right) -\E_{\bz \sim P}\bigl[k(\by, \bz)\bigr],
    \end{align*}
    with $c=k(\by, \by)$, is strictly proper when $k$ is characteristic \citep{steinwart2021strictly}. Kernel scores are directly related to maximum mean discrepancies (MMDs) \citep{gretton2012kernel}, in fact, the divergence in \eqref{eq:sr_objective} becomes $\frac{1}{2}\mathrm{MMD}^2$. Here, we focus on the Gaussian kernel $k(\by, \by') = \exp\!\bigl(-\tfrac{1}{\gamma^2}\|\by - \by'\|_2^2\bigr)$ with bandwidth $\gamma > 0$. For a Gaussian predictive, it admits a closed form expression, as stated in Appendix~\ref{app:ks_closed_form}, and it is bounded, which is the key property behind the following result.
\begin{theorem}[Lipschitz continuity]
\label{thm:ks-lipschitz}
Under the setting of Lemma~\ref{lemma:log_score_gradients}, the gradients of $S_k$ are bounded and the score is globally Lipschitz continuous in $(\bmu, \bm{L})$.
\end{theorem}
The theorem is proved in Appendix~\ref{app:loss_gradients}. It highlights that the kernel score is more stable for gradient-based optimization. The toy example in \autoref{fig:gaussian_sampling_example} illustrates this further on a rank-deficient target.
\begin{figure}[t]
    \centering
    \includegraphics[width=\linewidth]{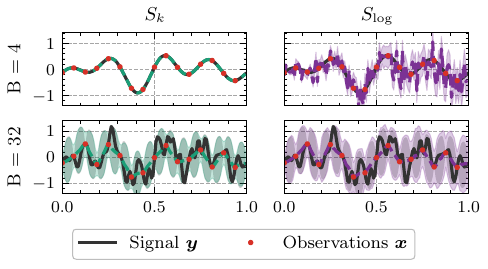}
    \caption{Toy example. We draw a signal $\by\sim\cN(\bm0,\bSigma_\mathrm{pr})$ with circulant $\bSigma_\mathrm{pr}$ and impose a cut-off on the eigenvalues, i.e., $\lambda_i>0$ iff $|i|\le B$. Now, we observe a noisy subsample $\bx=\bS\by+\bm\eta$ with observational noise $\bm\eta\sim\cN(\bm0,\sigma^2\bI_n)$, where $\bS\in\R^{n\times d}$ selects every $(d/n)$-th coordinate, and train a neural network to reconstruct $\by$ given $\bx$. The posterior has a closed form with $\operatorname{rank}(\bSigma_\mathrm{post})=\min(2B{+}1,d)$. By varying the band limit $B$, we interpolate between a low-rank and a higher-rank target covariance matrix.
The figure shows the estimated mean (colored dashed line) and standard deviation (shaded colored area), as well as the true posterior standard deviation (shaded grey area). In the low-rank case, $S_\mathrm{log}$ fails to recover the posterior, while $S_k$ does. More details are given in Appendix~\ref{app:toy_example}.}
    \label{fig:gaussian_sampling_example}
\end{figure}

\subsection{Spectral Parameterization}
Apart from the loss instability, we also address the problem of approximating the high-dimensional predictive covariance matrix. Here, we propose to model the marginals and correlation structure separately, optimizing the latter in spectral space. This is motivated by the fact that a Gaussian is closed under linear transformations. In particular, let $\bY \sim \cN(\bmu, \bSigma)$ and $\bB \in \C^{d \times d}$. Then, ${\tilde{\bY} = \bB \bY \sim \cN(\bB\bmu, \bB \bSigma \bB^*)}$.
Further, the kernel score transforms consistently under such maps and is invariant for unitary $\bB$.
\begin{proposition}[Kernel score invariance]
\label{prop:score-invariance}
Let $k(\by,\by')=\kappa(\|\by-\by'\|_2^2)$ be isotropic, with $\|\cdot\|_2$ the Hermitian norm on $\C^d$, and let $\bB_\# P$ denote the pushforward of $P=\cN(\bmu,\bSigma)$ under $\bB\in\C^{d\times d}$.
\begin{enumerate}
\item[(i)] If $\bB$ is invertible, then $S_k(\bB_{\#}P,\bB\by)=S_{k_{\bB}}(P,\by)$ with $k_{\bB}(\by,\by'):=\kappa\big(\|\by-\by'\|_{\bB^*\bB}\big)$, where $\|\bm v\|_{\bB^*\bB}^2=\bm v^*\bB^*\bB\bm v$. Hence $S_{k_{\bB}}$ is strictly proper whenever $S_k$ is.
\item[(ii)] If $\bB$ is unitary, then $k_{\bB}=k$ and $S_k(\bB_{\#}P,\bB\by)=S_k(P,\by)$.
\end{enumerate}
\end{proposition}
We provide a proof in Appendix~\ref{app:score_proofs}. 
%
Part (ii) is what allows us to fit the correlation structure in the spectral domain without changing the objective.

In this work, we mainly consider the DFT, for which we write $\bB=\bF$ in the following, since it enables fast computation.
Let ${\bF \in \C^{d \times d}}$ denote the unitary DFT matrix with entries $F_{jk} = d^{-1/2}\,e^{-2\pi i\, jk/d}$, so that $\bF^{-1}=\bF^{*}$. Throughout, a hat, as in $\hat\bSigma$, marks quantities in the transformed coordinates $\hat{\bz}=\bB\bz$. 
Consider a diagonal covariance matrix in Fourier space, $\hat{\bSigma} = \diag(s_1, \ldots, s_d)$. Its counterpart in physical space $\bSigma = \bF^{*} \hat{\bSigma}\, \bF$
is circulant, hence the covariance matrix of a stationary process on the periodic grid, and $\{s_k\}$ is its power spectral density by the Wiener--Khinchin theorem. It is dense in physical space, where relevant operations would cost $\mathcal O(d^3)$. Yet they reduce to $\mathcal O(d\log d)$ via the FFT.


\subsection{Spectral Correlation Estimation}
\label{sec:score}
So far we have shown how unitary transforms can be used to approximate a dense covariance matrix with few parameters. However, this trades off dependence structure against a correct marginal fit, since, for example, a circulant matrix implies equal marginal variances in every dimension.
To circumvent this issue, and to further enhance expressivity, we propose \score{}, a parameterization that separates the marginals from the correlation structure.
Throughout, we suppress the dependence on $\bx$ in the notation.

Denote $\mathrm{corr}(\bm \Sigma)
  :=
  \diag(\bm \Sigma)^{-1/2}\,\bm \Sigma\,\diag(\bm \Sigma)^{-1/2}$ for the correlation matrix of $\bm \Sigma$.
Consider the Gaussian target $Q=\cN(\bmu_Q,\bSigma_Q)$ and write $R_Q:=\mathrm{corr}(\bSigma_Q)$ and $\bD_Q:=\diag(\bSigma_Q)^{1/2}$. Then its covariance factors into marginal scales and correlation as
\begin{equation}\label{eq:corr-fact}
  \bSigma_Q=\bD_Q\,R_Q\,\bD_Q .
\end{equation}
The factors are complementary: $\bD_Q$ holds the marginal scales and $R_Q$ the correlation between coordinates. We mirror this factorization in the model.

Fix a unitary matrix
$\bB\in\C^{d\times d}$ associated with the chosen transform and refer
to the coordinates $\bB\by$, $\by\in\R^{d}$, as the
\emph{$\bB$-domain}. A prediction for a given input is specified by a triple
$(\bmu,\bm\tau,\hat\bSigma)$, with mean vector $\bmu\in\R^{d}$, marginal variances
$\bm\tau\in\R_{>0}^{d}$ collected in
the diagonal scale matrix $\bD_{\bm \tau}:=\diag(\sqrt{\bm \tau})$, and a Hermitian positive
definite matrix $\hat\bSigma$ in the $\bB$-domain, which we call the \emph{core} of
the parameterization. The predictive distribution is
$\cN(\bmu,\bD_{\bm\tau}\,\bB^{*}\hat\bSigma\,\bB\,\bD_{\bm\tau})$ with $\bD_{\bm\tau}$ intended to carry the marginals and the transformed core $\bB^* \hat \bSigma \bB$ the
correlation of \eqref{eq:corr-fact}. Restricting the core is
what keeps the parameterization cheap, and we write
$\mathcal S\subseteq\{\hat\bSigma\in\C^{d\times d}:
\hat\bSigma=\hat\bSigma^{*}\succ0\}$ for the admissible class, e.g. the diagonal or low-rank-plus-diagonal matrices. Our quantitative results mostly concern the
case where $\bB = \bF$ is the DFT and $\mathcal S$
the class of diagonal cores.

Although \eqref{eq:corr-fact} suggests splitting the learning of $\bSigma_Q$ into marginals and correlation, the above parameterization does not enforce this. The transformed core is in general not a correlation. Its diagonal entries $\Delta_i(\hat\bSigma):=\big(\bB^{*}\hat\bSigma\bB\big)_{ii}$, the \emph{implied marginal variances} $\bm\Delta(\hat\bSigma)\in\R^{d}_{>0}$, need not equal one, so that $(\bSigma_P)_{ii}=\tau_i\,\Delta_i(\hat\bSigma)$ instead of $\tau_i$.

\textbf{Taking the correlation.}
The remedy is to pass the transformed core through the same operator that defines the target correlation in \eqref{eq:corr-fact}, so that the parameterization turns into
\begin{equation}\label{eq:core_param_corr}
  \bSigma_P \;:=\; \bD_{\bm\tau}\;\mathrm{corr}\!\big(\bB^{*}\hat\bSigma\bB\big)\;\bD_{\bm\tau} .
\end{equation}
By definition, this only requires computing $\bm\Delta(\hat\bSigma)$, which for a diagonal core $\hat\bSigma=\diag(\bm s)$ with $\bm s\in\R_{>0}^{d}$ is the sum $\Delta_i=\sum_{k}s_k\lvert B_{ki}\rvert^{2}$. Replacing the transformed core with its correlation pins the marginals exactly.

\begin{lemma}[Correlation core]\label{lem:corr-core}
Let $\bB\in\C^{d\times d}$ be unitary, $\hat\bSigma\succ0$ Hermitian and $\bm\tau\in\R_{>0}^{d}$. Then the covariance \eqref{eq:core_param_corr} satisfies $(\bSigma_P)_{ii}=\tau_i$ for every $i=1,\dots,d$.
\end{lemma}
The proof is in Appendix~\ref{app:score_proofs}. For the DFT and diagonal cores, the parameterization simplifies.

\textbf{Fourier transform and diagonal core.}
\label{par:dft-diagonal}
Let $\bB=\bF$ be the DFT and $\hat\bSigma=\diag(\bm s)$ with $\bm s\in\R_{>0}^{d}$ and $s_k=s_{d-k}$, which makes $\bF^{*}\diag(\bm s)\bF$ real; cf. Appendix~\ref{app:nn_predictions}. This matrix is circulant, so its diagonal is the constant $\bar s=\tfrac1d\sum_k s_k$. Normalizing therefore only rescales the spectrum, and the prediction becomes
\[
  \bSigma_P=\bD_{\bm\tau}\,\bF^{*}\diag(\bm s/\bar s)\,\bF\,\bD_{\bm\tau}.
\]
We therefore let the network emit logits $\bm\ell\in\R^{d}$ and set $\bm s=d\,\mathrm{softmax}(\bm\ell)$, so that $\bar s=1$ and the transformed core is a correlation matrix already, i.e.,
\[
  \mathrm{corr}\big(\bF^{*}\diag(\bm s)\bF\big)\;=\;\bF^{*}\diag(\bm s)\bF .
\]
This normalization also removes a redundancy. Since $\operatorname{corr}$ is invariant under $\hat{\boldsymbol\Sigma} \mapsto c\hat{\boldsymbol\Sigma}$ for every $c > 0$, the core is identified only up to a positive factor, the \emph{scale gauge}. For the DFT and a diagonal core, $\bar s = 1$ fixes it, so that each correlation matrix corresponds to a single spectrum; compare Appendix~\ref{app:nn_predictions}.

By Lemma~\ref{lem:corr-core}, the marginal variances equal $\bm\tau$ for every correlation core. \score{} exploits this with a staged objective that fits the marginals first and then the correlation with the marginals frozen. Both stages range over maps $\bx\mapsto(\bmu,\bm\tau,\hat\bSigma)$ that a network realizes.

\textbf{\score{}.}
Let $S_1$ be a univariate and $S_2$ a multivariate scoring rule. The first stage is
\begin{equation}\label{eq:stage1}
  (\bmu^{\star},\bm\tau^{\star})
  =\operatorname*{arg\,min}_{\bmu,\,\bm\tau}
   \E\sum\nolimits_{i=1}^{d} S_1\big(\cN(\mu_i,\tau_i),\,Y_i\big),
\end{equation}
fitting the marginal distributions only. Write $\bD^{\star}:=\bD_{\bm\tau^{\star}}$, $\bm m:=(\bD^{\star})^{-1}\bmu^{\star}$ and $\bm Z:=(\bD^{\star})^{-1}\bY$ for the standardized target. The second stage holds the marginals fixed and fits the core over the admissible class $\mathcal S$,
\begin{equation}\label{eq:stage2}
  \hat\bSigma^{\star}
  =\operatorname*{arg\,min}_{\hat\bSigma\in\mathcal S}
   \E\,S_2\Big(\cN\big(\bm m,\;
     \mathrm{corr}\!\big(\bB^{*}\hat\bSigma\bB\big)\big),\,\bm Z\Big) .
\end{equation}
Dividing out the marginal scales removes the per-coordinate scale learned in the first stage, so with $\mathcal M:=\big\{\mathrm{corr}\big(\bB^{*}\hat\bSigma\bB\big)\;:\;\hat\bSigma\in\mathcal S\big\}$ the prediction ranges exactly over $\{\cN(\bm m,R):R\in\mathcal M\}$. Writing $R^{\star}:=\mathrm{corr}(\bB^{*}\hat\bSigma^{\star}\bB)\in\mathcal M$ for the deployed correlation, the prediction is $\cN\big(\bmu^{\star},\bD^{\star}R^{\star}\bD^{\star}\big)$, whose marginal variances are $\bm\tau^{\star}$ by Lemma~\ref{lem:corr-core}. The overall workflow of \score{} is illustrated in \autoref{fig:method_overview}.
In practice, we take $S_1$ to be the continuous ranked probability score (CRPS) \citep{GneitingRaftery2007}, a parameter-free kernel score, and $S_2$ the Gaussian kernel score instead of the log score, motivated by Lemma~\ref{lemma:log_score_gradients} and Theorem~\ref{thm:ks-lipschitz}. After a single FFT, it evaluates in closed form at $\mathcal O(d)$ (Appendix~\ref{app:score_closed_forms}). For its bandwidth, the median heuristic \citep{gretton2012kernel}, resolves to $\gamma^{2}=4d$ (Lemma~\ref{lem:bandwidth}), which we average across different scalings. See Appendix~\ref{app:nn_predictions} for implementation details.

\begin{remark}[Why a staged fit]
\score{} is a deliberate design choice, for several reasons. Freezing the marginals keeps them exact even when the correlation
class is misspecified (Proposition~\ref{prop:two-stage-opt}(i)). A joint fit would achieve a better joint score only by distorting the marginals. For the kernel score, a joint fit is also not available
at $\mathcal O(d)$. The closed form (Appendix~\ref{app:score_closed_forms})
requires a circulant, which $\boldsymbol\Sigma_P =
\mathbf D_{\boldsymbol\tau} \mathbf F^* \operatorname{diag}(\mathbf s) \mathbf F \mathbf D_{\boldsymbol\tau}$ is not unless $\boldsymbol\tau$ is constant. Nor
can \eqref{eq:stage2} be optimized jointly. The marginals enter only through the residual $\mathbf D_{\boldsymbol\tau}^{-1}(\mathbf y - \boldsymbol\mu)$, so it degenerates as $\boldsymbol\tau \to \infty$.
\end{remark}

A sufficiently rich model attains the minimum of \eqref{eq:stage1} and \eqref{eq:stage2} at every input, as the following result describes.
\begin{proposition}[Population optimum]\label{prop:two-stage-opt}
At a fixed input $\bx$, let the target be $Q=\cN(\bmu_Q,\bSigma_Q)$ with $\bSigma_Q\succ0$, factored as $\bSigma_Q=\bD_QR_Q\bD_Q$. Let $S_1$ be strictly proper on $\cP(\R)$, let $S_2$ be strictly proper and translation invariant (e.g., a kernel score with isotropic kernel), and write $D$ for the divergence associated with $S_2$. Then
\begin{enumerate}[leftmargin=25pt]
\item[(i)] The first stage fits $\bmu^{\star}=\bmu_Q$ and $\tau_i^{\star}=(\bSigma_Q)_{ii}$ for $i=1,\dots,d$, hence $\bD^{\star}=\bD_Q$.
\item[(ii)] The second stage fits the correlation alone, $R^{\star}\in\arg\min_{R\in\mathcal M}D\big(\cN(\bm0,R),\cN(\bm0,R_Q)\big)$, and the deployed prediction is $\cN(\bmu_Q,\bD_QR^{\star}\bD_Q)$.
\end{enumerate}
\end{proposition}
The proof is given in Appendix~\ref{app:score_proofs}. Let $\varepsilon_{\mathrm{app}}(\bx):=D\big(\cN(\bm0,R^{\star}(\bx)),\cN(\bm0,R_Q(\bx))\big)$ denote the approximation error at input $\bx$ and $\bar\varepsilon_{\mathrm{app}}:=\E_{\bX}[\varepsilon_{\mathrm{app}}(\bX)]$ its average. We provide explicit bounds in Appendix~\ref{app:expressivity}.

\textbf{Finite-sample error.}
Proposition~\ref{prop:two-stage-opt} is a population statement. At finite sample size, both stages are trained on empirical counterparts of \eqref{eq:stage1} and \eqref{eq:stage2}, the second on residuals standardized by the \emph{estimated} marginals, and one might worry that this amplifies their error. The following shows that it does not.

\begin{theorem}[\score{} is PAC, informal]\label{thm:pac-main}
Train \score{} with the CRPS in the first stage and the Gaussian kernel score in the second, on $N$ and $M$ i.i.d.\ samples from $\bP$, within the network class of Assumption~\ref{ass:class}, and let $P_{N,M}$ be the deployed prediction. Under Assumption~\ref{ass:target}, for every $\varepsilon>0$ and $\delta\in(0,1)$, if $N\in\tilde{\mathcal O}(\varepsilon^{-4}\log\tfrac1\delta)$ and $M\in\tilde{\mathcal O}(\varepsilon^{-2}\log\tfrac1\delta)$, then with probability at least $1-\delta$
\[
  \E_{\bX}\Big[\tfrac{1}{2}\mathrm{MMD}^{2}\big(Q(\bX),\,P_{N,M}(\bX)\big)\Big]
  \le c\,\bar\varepsilon_{\mathrm{app}}+\varepsilon
\]
for a constant $c > 0$, where $\tilde{\mathcal O}$ hides logarithmic factors.
\end{theorem}
Formal statement and proof are given in Appendix~\ref{app:pac}.


\textbf{Beyond Fourier and diagonal cores.}
The parameterization \eqref{eq:core_param_corr}, Lemma~\ref{lem:corr-core} and Proposition~\ref{prop:two-stage-opt} hold for any unitary $\bB$ and admissible class $\mathcal S$. In particular, the first-stage marginals stay exact regardless of the chosen core and second-stage objective. For the log score, the closed forms even extend to arbitrary $\bB$ and low-rank cores, which is not the case for the Gaussian kernel score, as shown in Appendix~\ref{app:score_closed_forms}. 

\section{EXPERIMENTS}
\label{sec:experiments}

%
%
\begin{figure*}[t!]
\centering
\includegraphics[width=\linewidth]{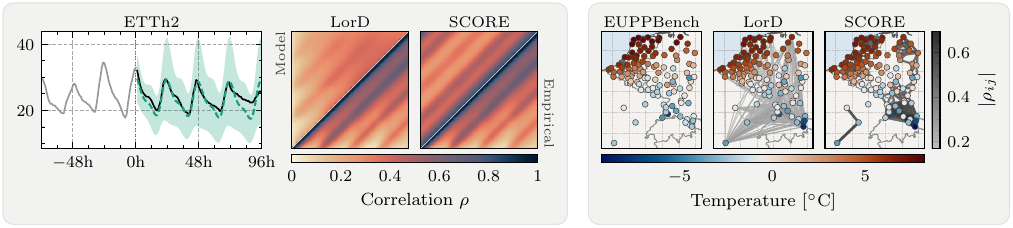}
    \caption{SCORE predictions. Left: Mean prediction and standard deviation of \score{} on the univariate ETTh2 dataset and a comparison with LorD regarding similarity to the empirical correlation structure. Right: observed temperature across stations and predictions and correlations using \score{} and LorD.}
    \label{fig:main_prediction_example}
\end{figure*}

\paragraph{Datasets \& baselines.}


\begin{table}[ht]
\centering
\small
\caption{Overview of tasks.}
\label{tab:datasets}
\adjustbox{max width=\linewidth}{%
\begin{tabular}{@{}llll}
\toprule
Task & Dataset & Backbone & Prior work \\
\midrule
Time series & \citep{wu2021autoformer} & PatchTST \citep{nie2023a} & \citep{salinas2019high, zheng2024multivariate, choi2025} \\
Depth est. & NYU \citep{Silberman:ECCV12} & DepthAnything \citep{depth_anything_v2} & \citep{liu2023single} \\
Gridded temp. & ERA5 \citep{hersbachERA5GlobalReanalysis2020} & U-Cast \citep{cachay2026ucast}& \citep{raspNeuralNetworksPostprocessing2018, muschinskiCholeskybasedMultivariateGaussian2024}  \\
Station temp. & EUPPBench \citep{essd-15-2635-2023} & GNN \citep{feik2024graphneuralnetworksspatial} & \citep{feik2024graphneuralnetworksspatial}  \\
\bottomrule
\end{tabular}
}
\end{table}
We consider four tasks where a Gaussian predictive is a reasonable assumption and has been used in prior work, summarized in \autoref{tab:datasets}. For each task, we fix a backbone and only replace the output layer with the respective covariance parameterization (Appendix~\ref{app:nn_predictions}), such that the comparison isolates the Gaussian representation and all methods are comparable. In all tasks \score{} uses the Fourier transform, except in the station-based task, which uses a graph neural network and models the covariance via the graph Fourier transform, illustrating how \score{} extends beyond the DFT.
Across all tasks, we compare \score{} against the same baselines: a deterministic model (Det), a diagonal Gaussian (Diag), and the low-rank-plus-diagonal parameterization (LorD) \citep{salinas2019high, zheng2024multivariate}, the most common scalable alternative. We set $r=1$, which matches the output-layer parameter count of \score{} (larger $r$ in Appendix~\ref{app:ablations}).
A sample-based generative model (SB) with noise injection in the hidden layers and trained with the Gaussian kernel score (KS) \citep{pacchiardiProbabilisticForecastingGenerative2021, kelenDistributionFreeDataUncertainty2024} serves as a nonparametric, non-Gaussian baseline.
Det is trained with the mean squared error (MSE), SB with KS, and \score{} with the CRPS in the first and the KS in the second
stage. Diag and LorD are trained twice, with the NLL as
the standard choice and with the KS as alternative, which separates the effect of the loss
from that of the covariance parameterization in \autoref{tab:main-rank}.
In Appendix~\ref{app:results} we also provide results for a model with a full Cholesky factor (Chol), as a fully expressive baseline. However, since this is only feasible for a few tasks, it is omitted in the main paper.

\paragraph{Evaluation protocol.}
\label{sec:eval}
All models are compared on a common set of metrics, mostly proper scoring
rules. Scores are test-set and five-seed averages.
The MSE and the CRPS
measure mean and marginal fit, respectively. The NLL, the energy
score (ES) and the KS are strictly proper for the joint
distribution and the variogram score (VS) is proper, shift-invariant,
and targets the dependence structure. Note that since the NLL, CRPS, and KS serve as training objectives, one would expect models using these as a loss to perform better in terms of the corresponding metric. Therefore, the performance should be considered with respect to all selected metrics to obtain an overall picture.
Sample-based metrics use an ensemble of size $M=100$; NLL and MSE are evaluated in closed form (cf. Appendix~\ref{app:eval}).
All implementations and reproducible experiments are available in a public repository at \url{https://github.com/cbuelt/score}

\paragraph{Results.}
\begin{table}[t]
\centering
\caption{Average rank of the predictive methods. Ranks are computed per task and seed, and then averaged over the eleven datasets. Best rank in bold.}
\label{tab:main-rank}
\small
\adjustbox{max width=\linewidth}{%
\begin{tabular}{llrrrrrrr}
\toprule
Core & Objective & MSE & CRPS & ES & VS & KS & NLL/$d$ & Avg. \\
\midrule
\multicolumn{9}{l}{\emph{Physical space}}\\
\quad Det & MSE & $4.0$ & $6.5$ & $6.6$ & $5.1$ & $5.7$ & --- & $5.6$ \\
\quad Diag & NLL & $4.1$ & $2.8$ & $3.6$ & $4.3$ & $4.2$ & $2.9$ & $3.7$ \\
\quad Diag & KS & $3.2$ & $3.5$ & $3.0$ & $4.1$ & $3.2$ & $4.0$ & $3.5$ \\
\quad LorD $r{=}1$ & NLL & $6.1$ & $5.6$ & $5.5$ & $5.0$ & $5.8$ & $\mathbf{2.1}$ & $5.0$ \\
\quad LorD $r{=}1$ & KS & $3.4$ & $3.4$ & $2.6$ & $3.4$ & $\mathbf{2.6}$ & $3.2$ & $3.1$ \\
\quad SB & KS & $4.2$ & $4.2$ & $4.3$ & $4.3$ & $3.7$ & --- & $4.1$ \\
\cmidrule(l){1-9}
\multicolumn{9}{l}{\emph{Spectral space}}\\
\quad SCORE & CRPS+KS & $\mathbf{3.0}$ & $\mathbf{2.1}$ & $\mathbf{2.4}$ & $\mathbf{1.8}$ & $2.8$ & $2.8$ & $\mathbf{2.5}$ \\
\bottomrule
\end{tabular}
}
\end{table}

\begin{table}[ht]
\centering
\caption{Number of wins when training the same predictive model with the Gaussian kernel score instead of the NLL. \textbf{Bold}: $p<0.05$ under the clustered sign-flip test, Holm-Bonferroni corrected over the four metrics.}

\label{tab:performance_significance}
\small
\begin{tabular}{rrrrr}
\toprule
  $n$ & MSE & CRPS & ES & VS \\
\midrule
167 & $\mathbf{150}$ & $\mathbf{140}$ & $\mathbf{156}$ & $\mathbf{139}$ \\
\bottomrule
\end{tabular}
\end{table}

\autoref{tab:main-rank} shows the average rank of each method on each metric. \score{} attains the best average rank ($2.5$) and the best rank on MSE, CRPS, ES and VS. The strongest baseline, LorD trained with the KS ($3.1$), is outranked by \score{} on these metrics at equal parameter count. The margin is largest on the VS, which targets the dependence structure, and on the CRPS, which \score{} inherits from its first stage. 
Unsurprisingly, the NLL is won by the model trained on it: LorD-NLL. The same holds for KS, but here \score{} ranks similarly quite ($\pm 0.2$).

We next separate the two contributions. 
\autoref{tab:performance_significance}
counts, over all data sets and parameterizations, how often training with the
kernel score beats the NLL on the evaluation metrics that do not serve as an objective. The kernel score wins the large majority, significantly under a data-set-clustered
sign-flip test.
Among the KS-trained models, \score{} attains the best average rank on every metric, most clearly on the VS, which no method trains on. This gap is a contribution of the parameterization (Appendix~\ref{app:ablations}).
\begin{figure}[ht]
    \centering
    \includegraphics[width=\linewidth]{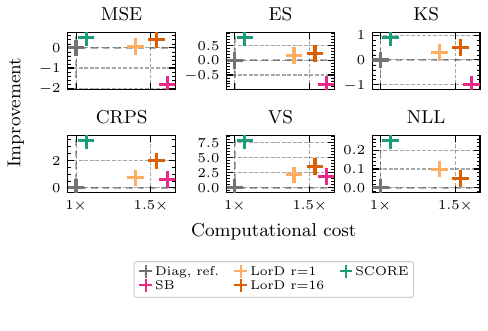}
    \caption{Performance improvement (difference in nats per coordinate for NLL, percentage improvement otherwise) vs.\ increase in compute, relative to the
Diag (median over seeds and data sets).}
    \label{fig:performance_runtime}
\end{figure}
This is further highlighted in \autoref{fig:performance_runtime}, which analyzes all methods trained with the KS with respect to relative improvement in performance and computational cost against Diag. As expected, \score{} is slower than Diag, since the latter is linear in computational complexity. However, it is the fastest among the baselines that model dependence and admits the best performance, with clear margins in all metrics except the MSE, where performance is similar to LorD.

\begin{figure}[ht]
    \centering
    \includegraphics[width=\linewidth]{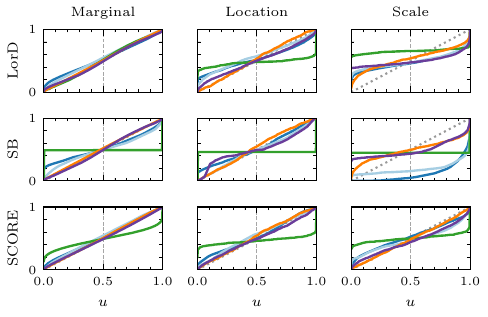}
    \caption{Different calibration measures for LorD, SB and \score{}, using the kernel score across all datasets. The dashed diagonal line indicates optimal calibration.}
    \label{fig:calibration_main}
\end{figure}
Further, we analyze the probabilistic calibration of the methods with respect to univariate marginals, as well as multivariate location and scale, as shown in \autoref{fig:calibration_main}. \score{} exhibits the best calibrated predictions overall, except on NYU (green), where the depth maps violate the stationarity assumption and \score{} falls behind (Appendix~\ref{app:results}). LorD shows good marginal calibration, while SB fails especially on the scale calibration.

Detailed results for each task, as well as regarding the runtime, calibration, and additional ablations, are available in Appendix~\ref{app:results}.
Overall, \score{} is the best-ranked method on every metric, with a large margin on the VS, indicating a superior approximation of the dependence structure. Both our design choices contribute to the performance: the KS improves upon the
NLL significantly, while with the same loss, \score{} beats established approaches. \score{} achieves this with near-linear cost, due to the FFT, making it a good choice for a predictive Gaussian model in high dimensions. We provide further ablations of the rank $r$ for the LorD method and the impact of the two-stage training process in Appendix~\ref{app:ablations}.

\section{RELATED WORK}
Neural networks predicting Gaussian distributions date back to mean-variance networks \citep{374138}, which output a univariate Gaussian trained with the negative log-likelihood (NLL); for multiple variables, this amounts to the diagonal covariance parameterization. This setup is known to train poorly because the NLL gradient of the mean is scaled by the predicted precision, motivating remedies such as reweighted losses, decoupled mean and variance heads, or natural parameterizations \citep{seitzer2022pitfalls, pmlr-v206-stirn23a, immer2023effective, vemuganti2026fisher}. These difficulties persist and are amplified for more expressive covariance structures. Full covariances were already learned by \citet{10.1162/neco.1996.8.4.843}, who train a network to output the Cholesky factor by likelihood maximization. \citet{9457247} compare this parameterization to the diagonal one across deep learning benchmarks, and \citet{muschinskiCholeskybasedMultivariateGaussian2024} use it for distributional regression in weather forecasting. The quadratic cost can be reduced via banded Cholesky factors \citep{dorta2018structured} or by training a Taylor expansion of the covariance \citep{pmlr-v235-shukla24a}.
Low-rank-plus-diagonal covariances originate in variational inference \citep{Ong03072018, NEURIPS2020_310cc7ca} but also appear in predictive settings such as uncertainty prediction in images \citep{3495724.3496794} and time series modeling with copulas \citep{salinas2019high} or Kronecker factorizations \citep{zheng2024better}. Kronecker-factored matrices themselves stem from curvature estimation \citep{pmlr-v37-martens15, ritter2018a} and variational inference \citep{pmlr-v48-louizos16, pmlr-v80-zhang18l} and have only recently been used as direct network outputs \citep{zheng2024multivariate}.

Gaussian processes (GPs) offer a conceptually different route to input-dependent multivariate Gaussian predictions: mean and covariance are induced by a kernel whose hyperparameters are fit by marginal likelihood optimization \citep{bonilla2008multitask, goldberg1998heteroscedastic, lazarogredilla2011variational}, and deep kernel learning parameterizes the kernel map by a neural network \citep{pmlr-v51-wilson16}. Exact GP inference costs $\mathcal{O}(d^2)$ memory and $\mathcal{O}(d^3)$ time, which is prohibitive for gridded targets and is typically addressed by inducing points or low-rank approximations \citep{titsias2009variational, lazarogredilla2010sparse, pmlr-v37-wilson15}. We instead follow the distributional regression paradigm \citep{gneiting2014, klein2015}, where the network parameterizes the conditional predictive law directly and describes aleatoric dependence in the target rather than posterior uncertainty about a latent function.

\section{CONCLUSION}
We introduced \score{}, a method for learning high-dimensional structured
Gaussian distributions with neural networks. Training with kernel scores
instead of the negative log-likelihood circumvents the instability of
likelihood-based optimization. In addition, we proposed to parameterize the
covariance matrix in a spectral unitary basis, which, owing to proven
properties of the kernel score, yields a dense covariance in physical space at
low computational cost. A staged objective that separates learning the
marginals from learning the correlation structure keeps training efficient,
while recovering the true marginals and an approximation of the true
correlations. In our experiments, \score{} outperforms established baselines
across a range of metrics, is the best calibrated among the compared methods,
and offers the best trade-off between compute and performance.

\paragraph{Limitations and future work.}
The main limitation of \score{} is the circulant covariance structure it implies, which may be unreasonable in some applications. While we quantify the resulting approximation error, and the method is deliberately designed as a fast yet simple approximation, certain predictive tasks will require moving beyond this structure.
This points to several directions for future work. On the modeling side, richer parameterizations in the transformed domain, such as low-rank or tridiagonal instead of diagonal, could increase expressivity while keeping computational cost low, and other unitary transforms, such as wavelets, could be an interesting object of study. Moreover, \score{} extends naturally to Gaussian mixtures and other elliptical distributions, such as the multivariate $t$-distribution, which would broaden its applicability to non-Gaussian and heavy-tailed data. On the theoretical side, sharpening the PAC bounds and deriving an optimal split between the two stages of the objective would further improve the usability of our method.

\begin{acknowledgements}
C. Bülte and G. Kutyniok acknowledge support by the DAAD programme Konrad Zuse Schools of Excellence in Artificial Intelligence, sponsored by the Federal Ministry of Research, Technology and Space and the German Research Foundation under the grant DFG-SPP-2298.

E. Partow and G. Kutyniok acknowledge support by the German Research Foundation (Deutsche Forschungsgemeinschaft,
DFG) – GRK 3081 – Project number 534429653.

G. Kutyniok also acknowledges support by the gAIn project, which is funded by the Bavarian Ministry of Science and the Arts (StMWK Bayern) and the Saxon Ministry for Science, Culture and Tourism (SMWK Sachsen). Furthermore, G. Kutyniok is supported by LMUexcellent, funded by the Federal Ministry of Education and Research (BMBF) and the Free State of Bavaria under the Excellence Strategy of the Federal Government and the Länder as well as by the Hightech Agenda Bavaria.

The authors acknowledge the computational and data resources
provided by the Leibniz Supercomputing Centre (\href{https://www.lrz.de}{www.lrz.de}).
\end{acknowledgements}
 
%
%
\bibliographystyle{unsrtnat}
\bibliography{references}

\clearpage
\onecolumn

\appendix
\etocdepthtag.toc{appendix}

\section*{Appendix}

\etocsettagdepth{main}{none}
\etocsettagdepth{appendix}{section}
\tableofcontents

\section{Notation}\label{app:notation}
Throughout the paper, let $\mathcal{X}\subseteq\R^{k}$ and $\mathcal{Y}\subseteq\R^{d}$ denote the input and output spaces, respectively, and let $\mathcal{P}(\mathcal{Y})$ be a convex set of probability measures on $\mathcal{Y}$.
We consider a training dataset $\mathcal{D}=\{(\bx_n,\by_n)\}_{n=1}^{N}\subset(\mathcal{X}\times\mathcal{Y})^N$ consisting of i.i.d.\ samples from an unknown joint distribution $\bP=\bP_{\bX,\bY}$ with marginal $\bP_{\bX}$ and conditional law $\bP_{\bY|\bX}(\cdot\mid\bx)$.
Expectations with respect to $\bP$ are written as $\bE_{\bX,\bY\sim\bP}[\cdot]$, while $P,Q\in\mathcal{P}(\mathcal{Y})$ denote generic probability distributions. Our predictive model is a neural network $f_\phi:\mathcal{X}\rightarrow\mathcal{P}(\mathcal{Y})$ with parameters $\phi\in\mathcal{W}\subseteq\R^q$, which maps an input $\bx$ to a probability distribution $f_\phi(\bx)$.
Unless stated otherwise, we restrict attention to Gaussian predictive distributions, $f_\phi(\bx)=\mathcal{N}\!\left(\bmu_\phi(\bx),\,\bSigma_\phi(\bx)\right)$, where $\bmu_\phi(\bx)\in\R^d$ and $\bSigma_\phi(\bx)\succ0$ denote the predicted mean vector and positive definite covariance matrix, respectively.
We write $\diag(\bv)$ for the diagonal matrix with diagonal entries $\bv$ and, for a matrix $\bM$, $\diag(\bM)$ for the diagonal matrix with the same diagonal as $\bM$. Moreover, $\bI$ denotes the identity matrix, $\bm1$ the vector of ones, and $\odot$ the entrywise product. We use bold lower-case letters for vectors and bold upper-case letters for matrices, with the exception of correlation matrices, which we write $R$ throughout to distinguish them from covariances. For a matrix $\bA$, the transpose and Hermitian transpose are denoted by $\bA^\top$ and $\bA^*$, respectively, $\lVert\bA\rVert_F$ is the Frobenius norm, $\lambda_{\min}$ and $\lambda_{\max}$ are the extreme eigenvalues, and $\bA\succ0$ ($\bA\succeq0$) means that a symmetric or Hermitian $\bA$ is positive (semi)definite. A hat, as in $\hat\bSigma$, marks quantities in the transformed coordinates $\hat{\bz}=\bB\bz$ of a unitary $\bB$.

\section{Loss gradients}
\label{app:loss_gradients}
This section derives the gradients of $S_\mathrm{log}$ and $S_k$ to prove  Lemma~\ref{lemma:log_score_gradients} and Theorem~\ref{thm:ks-lipschitz}.
Throughout, we consider a multivariate normal distribution $P=\cN(\bmu,\bSigma)$ with $\bSigma\succ0$, decomposed as $\bSigma=\bL\bL^\top$, where $\bL$ is a lower triangular matrix with positive diagonal. We analyze Lipschitz continuity of the scores in $(\bmu,\bL)$ through boundedness of their gradients. Let $\by\in\R^d$ denote the observation and $\br:=\bmu-\by$ the residual. Finally, let $\lambda_1\ge\cdots\ge\lambda_d$ denote the ordered eigenvalues of $\bSigma$ and $\be_k$, $k=1,\ldots,d$, the corresponding eigenvectors. 

\subsection{Log score}
Recall that the log score of a multivariate Gaussian is given by
\[
S_\mathrm{log}(P,\by)=\frac{1}{2}\log\det\bSigma+\frac{1}{2}(\by-\bmu)^\top\bSigma^{-1}(\by-\bmu)+\frac{d}{2}\log(2\pi).
\]
Let $\bw\coloneq\bSigma^{-1}\br$ and $\bz\coloneq\bL^{-1}\br$, so that $\bL^\top\bw=\bz$. The gradients of $S_\mathrm{log}$ take the form
\begin{align*}
    \nabla_{\bmu}S_\mathrm{log}=\bw,\qquad
    \big[\nabla_{\bL}S_\mathrm{log}\big]_{ii}=\frac{1}{L_{ii}}-w_i z_i,\qquad
    \big[\nabla_{\bL}S_\mathrm{log}\big]_{ij}=-w_i z_j,
\end{align*}
for all $i=1,\ldots,d$ and $j<i$.

\subsubsection{Proof of Lemma~\ref{lemma:log_score_gradients}}
\begin{proof}
Representing $\bSigma$ via its eigendecomposition leads to
\begin{align*}
    \bSigma = \sum_{k=1}^d \lambda_k \be_k \be_k^\top, \quad \bSigma^{-2} = \sum_{k=1}^d \lambda_k^{-2} \be_k \be_k^\top.
\end{align*}
Further, we rewrite $\br= \sum_{k = 1}^d  \Tilde{r}_k \be_k$ with $ \Tilde{r}_k = \be_k^\top \br$. Since $\bSigma^{-1}$ is symmetric, we obtain
\begin{align*}    
\| \nabla_{\bm \mu} S_\mathrm{log}\|^2 = \| \bSigma^{-1}\br\|^2 = \br^\top \bSigma^{-2}\br = \sum_{k=1}^d \lambda_k^{-2} \Tilde{r}_k^2,
\end{align*}
from which we obtain the lower bound
\begin{align*}    
\sum_{k=1}^d \lambda_k^{-2}\Tilde{r}_k^2 \geq \lambda_{\min}^{-2}\Tilde{r}_d^2 \implies \| \nabla_{\bm \mu} S_\mathrm{log}\| \geq \lambda_{\min}^{-1}|\Tilde{r}_d|. 
\end{align*}
Therefore, along any sequence of covariance matrices $\bSigma^{(n)}$ with $\lambda_\mathrm{min}(\bSigma^{(n)}) \to 0$ whose smallest eigenvector is not asymptotically orthogonal to the residual, i.e.\ $\liminf_n |\be_d^{(n)\top} \br| > 0$, we have $\| \nabla_{\bm \mu} S_\mathrm{log}\| \to \infty$.

For the second gradient, we have
\[
\|\nabla_{\bL}S_\mathrm{log}\|_F^2=\sum_i\Big(\frac{1}{L_{ii}}-w_iz_i\Big)^2+\sum_{i>j}w_i^2z_j^2 .
\]
Note that a vanishing diagonal entry $L_{ii} \to 0$ alone is not sufficient for divergence, since the two terms $1/L_{ii}$ and $w_i z_i$ may cancel. It suffices to exhibit a divergent sequence. Consider $\bL^{(n)} = \operatorname{diag}(\sigma_n, 1, \ldots, 1)$ with $\sigma_n \to 0$, so that $\lambda_{\min}^{(n)} = \sigma_n^2 \to 0$, which gives
\begin{align*}
    \left[\nabla_{\bL} S_\mathrm{log} \right]_{11} = \frac{1}{\sigma_n} - \frac{r_1^2}{\sigma_n^3},
\end{align*}
which diverges for every fixed $r_1$, including $r_1 = 0$.
Consequently, $\sup \|\nabla S_\mathrm{log}\|$ is unbounded in every neighborhood of $\lambda_{\min} = 0$, and the scoring rule is not (globally) Lipschitz continuous in $(\bmu,\bL)$.
\end{proof}
\subsection{Kernel score}
\label{app:ks_closed_form}
The closed form of the Gaussian kernel score \citep{kneissl2025improvedprobabilisticregressionusing} with bandwidth $\gamma>0$ is given by
\begin{align*}
S_k(P,\by)=\frac{1}{2}\underbrace{\det(\bA)^{-1/2}}_{c_1}-\underbrace{\det(\bm{G})^{-1/2}}_{c_2}\underbrace{\exp\Big(-\frac{1}{\gamma^2}\br^\top\bm{G}^{-1}\br\Big)}_{E}+\frac{1}{2},
\end{align*}
for $\br\coloneq\bmu-\by$, $\bA\coloneq\bI+4\bSigma/\gamma^2$ and $\bm{G}\coloneq\bI+2\bSigma/\gamma^2$. Further, define $\bw\coloneq\bm{G}^{-1}\br$ and note that $\nabla_{\bL}E=\frac{4E}{\gamma^4}\bw\bw^\top\bL$. Then, the gradients of the Gaussian kernel score are given by
\begin{align}
    \nabla_{\bmu}S_k=\frac{2c_2E}{\gamma^2}\bm{G}^{-1}\br,\qquad
    \nabla_{\bL}S_k=-\frac{2c_1}{\gamma^2}\bA^{-1}\bL+\frac{2c_2E}{\gamma^2}\bm{G}^{-1}\bL-\frac{4c_2E}{\gamma^4}\bw\bw^\top\bL.
\end{align}
Here $\nabla_{\bL}$ denotes the unconstrained derivative with respect to a full matrix argument; the gradient with respect to a lower-triangular $\bL$ is its lower-triangular part, whose Frobenius norm is bounded by that of the full matrix.

\subsubsection{Proof of Theorem~\ref{thm:ks-lipschitz}}
\begin{proof}
    Throughout, let $\bSigma = \sum_i \lambda_i \be_i \be_i^\top$ be the eigendecomposition of $\bSigma$ with the eigenvectors $\be_i$ introduced above, and let
    \begin{align*}
        \bA = \sum_i \alpha_i \be_i \be_i^\top, \ \alpha_i = 1 + \frac{4}{\gamma^2} \lambda_i \geq 1,\qquad
        \bm{G} = \sum_i \beta_i \be_i \be_i^\top, \ \beta_i = 1 + \frac{2}{\gamma^2} \lambda_i \geq 1,
    \end{align*}
    which are the eigendecompositions of $\bA$ and $\bm{G}$, since both are functions of $\bSigma$. Since all eigenvalues of $\bA$ and $\bm{G}$ are at least one, we have $\det(\bA), \det(\bm{G}) \geq 1$ and therefore
    \begin{align*}
        c_1 = \det(\bA)^{-1/2} \leq 1, \qquad c_2 = \det(\bm{G})^{-1/2} \leq 1.
    \end{align*}
    Write $\rho_i \coloneq \be_i^\top \br$ and define the normalized coefficients
    \begin{align*}
        t_i \coloneq \frac{\rho_i}{\gamma \sqrt{\beta_i}}, \qquad \text{such that} \qquad \rho_i^2 = \gamma^2 \beta_i t_i^2 .
    \end{align*}
    The exponent of $E$ is exactly $-\|\bt\|^2$, since
    \begin{align*}
    \frac{1}{\gamma^2} \br^\top \bm{G}^{-1}\br = \frac{1}{\gamma^2} \br^\top \Big( \sum_i \beta_i^{-1} \be_i \be_i^\top\Big) \br = \sum_i \frac{\rho_i^2}{\gamma^2 \beta_i} = \sum_i t_i^2 = \|\bt\|^2 ,
    \end{align*}
    hence $E = \exp(-\|\bt\|^2)$.

    \subparagraph{Gradient in $\bmu$.} With $\beta_i \geq 1$, we get
    \begin{align*}
        \|\bw\|^2 = \|\bm{G}^{-1}\br\|^2 = \sum_i \frac{\rho_i^2}{\beta_i^2} = \gamma^2 \sum_i \frac{t_i^2}{\beta_i} \leq \gamma^2 \|\bt\|^2 .
    \end{align*}
    Using the fact that $|x| \exp(-x^2) \leq 1/\sqrt{2e}$ for all $x \in \R$, we obtain
    \begin{align*}
        \|\bw\| E \leq \gamma \|\bt\| \exp(-\|\bt\|^2) \leq \frac{\gamma}{\sqrt{2e}} .
    \end{align*}
    Together with $c_2 \leq 1$ this yields
    \begin{align*}
        \|\nabla_{\bmu} S_k \| = \frac{2 c_2 E}{\gamma^2}\|\bw\| \leq \frac{2}{\gamma^2}\cdot \frac{\gamma}{\sqrt{2e}} = \frac{1}{\gamma}\sqrt{\frac{2}{e}} .
    \end{align*}

    \subparagraph{Gradient in $\bL$.} Since $\bL$ need not be symmetric, we use its singular value decomposition $\bL = \sum_i \sigma_i \be_i \bv_i^\top$; the left singular vectors of $\bL$ coincide with the eigenvectors of $\bSigma = \bL\bL^\top = \sum_i \sigma_i^2 \be_i \be_i^\top$, so that $\sigma_i^2 = \lambda_i$ is the $i$th eigenvalue of $\bSigma$. Consequently
    \begin{align*}
        \bA^{-1}\bL = \sum_i \frac{\sigma_i}{\alpha_i} \be_i \bv_i^{\top}, \qquad \bm{G}^{-1}\bL = \sum_i \frac{\sigma_i}{\beta_i} \be_i \bv_i^{\top},
    \end{align*}
    and, since $\be_i$ and $\bv_i$ are orthonormal, $\big\|\sum_i s_i \be_i \bv_i^\top\big\|_F^2 = \sum_i s_i^2$.

    For the first term $\bT_1 \coloneq -\frac{2c_1}{\gamma^2} \bA^{-1}\bL$, we obtain
    \begin{align*}
        \left\|\bT_1 \right\|_F^2 = \frac{4c_1^2}{\gamma^4} \sum_i \frac{\lambda_i}{\alpha_i^2} = \frac{4c_1^2}{\gamma^4} \sum_i \frac{\lambda_i}{(1+4\lambda_i / \gamma^2)^2}.
    \end{align*}
        Using $c_1 \leq 1$ and that differentiating $x \mapsto x/(1+cx)^2$ for $c>0$ shows that it is maximized at $x = 1/c$ with value $1/(4c)$, which gives $\frac{x}{(1+4x/\gamma^2)^2} \leq \gamma^2 / 16$ for all $x \geq 0$, we obtain
    \begin{align*}
        \left\|\bT_1 \right\|_F^2 \leq \frac{4}{\gamma^4}\, d\, \frac{\gamma^2}{16} = \frac{d}{4\gamma^2}\implies  \left\|\bT_1 \right\|_F \leq \frac{\sqrt{d}}{2\gamma}.
    \end{align*}

    For the second term, $\bT_2 \coloneq \frac{2c_2E}{\gamma^2}\bm{G}^{-1}\bL$, we similarly obtain
    \begin{align*}
        \|\bT_2\|_F^2 = \frac{4 c_2^2 E^2}{\gamma^4} \sum_i \frac{\lambda_i}{\beta_i^2}.
    \end{align*}
    Using $c_2E \leq 1$ and, by the same maximization with $c = 2/\gamma^2$, $\frac{\lambda_i}{(1+2\lambda_i/\gamma^2)^2}\leq \gamma^2/8$, we obtain
    \begin{align*}
        \left\|\bT_2 \right\|_F^2 \leq \frac{4}{\gamma^4}\,d\,\frac{\gamma^2}{8} = \frac{d}{2\gamma^2} \implies  \left\|\bT_2 \right\|_F \leq \frac{\sqrt{d}}{\sqrt{2}\,\gamma}.
    \end{align*}

    Finally, consider the third term $\bT_3 \coloneq - \frac{4c_2 E}{\gamma^4} \bw \bw^\top \bL$. Since $\bw\bw^\top \bL = \bw (\bL^\top \bw)^\top$ has rank one,
    \begin{align*}
        \|\bT_3\|_F &= \frac{4c_2E}{\gamma^4} \|\bw\| \, \|\bL^\top \bw\|, \quad
        \|\bw\|^2  = \br^\top \bm{G}^{-2}\br = \sum_i \frac{\rho_i^2}{\beta_i^2}, \quad
        \|\bL^\top \bw\|^2 = \bw^\top \bL \bL^\top \bw = \bw^\top \bSigma \bw = \sum_i \frac{\lambda_i \rho_i^2}{\beta_i^2}.
    \end{align*}
    Substituting $\rho_i^2 = \gamma^2 \beta_i t_i^2$ and using $\beta_i \geq 1$ for the first identity gives
    \begin{align*}
        \|\bw\|^2 = \gamma^2 \sum_i \frac{t_i^2}{\beta_i} \leq \gamma^2 \|\bt\|^2, \quad
        \|\bL^\top \bw\|^2 = \gamma^2 \sum_i \frac{\lambda_i}{\beta_i}\,t_i^2 \leq \frac{\gamma^4}{2} \|\bt\|^2 ,
    \end{align*}
    where the second bound uses
    \begin{align*}
        \frac{\lambda_i}{\beta_i} = \frac{\lambda_i}{1 + 2\lambda_i/\gamma^2} \leq \frac{\gamma^2}{2} \qquad \text{for all } \lambda_i \geq 0.
    \end{align*}
    Combining yields
    \begin{align*}
        \|\bw\| \, \|\bL^\top \bw\| \, E \leq \gamma \|\bt\| \cdot \frac{\gamma^2}{\sqrt{2}} \|\bt\| \exp(-\|\bt\|^2) = \frac{\gamma^3}{\sqrt{2}} \|\bt\|^2 \exp(-\|\bt\|^2) \leq \frac{\gamma^3}{\sqrt{2}\,e},
    \end{align*}
    since $\max_{s\geq 0} s \exp(-s) = e^{-1}$. In total, using $c_2 \leq 1$,
    \begin{align*}
        \|\bT_3\|_F \leq \frac{4}{\gamma^4} \frac{\gamma^3}{\sqrt{2}\,e} = \frac{2 \sqrt{2}}{e \gamma}.
    \end{align*}

    Combining the three terms with the triangle inequality, we obtain
    \begin{align*}
        \|\nabla_{\bL} S_k\|_F \leq \left(\frac{1}{2}+\frac{1}{\sqrt{2}} \right) \frac{\sqrt{d}}{\gamma} + \frac{2\sqrt{2}}{e \gamma}.
    \end{align*}
    All bounds hold for every $\bmu, \by \in \R^d$ and every lower-triangular $\bL$, including singular $\bL$, since no step requires invertibility. As the domain of $(\bmu,\bL)$ is convex, the bounded gradients imply that $(\bmu, \bL) \mapsto S_k(P,\by)$ is globally Lipschitz continuous.
\end{proof}

\section{Spectral Correlation Estimation}
\label{app:score}
This appendix collects the material behind Section~\ref{sec:score}.
\subsection{Proofs}
\label{app:score_proofs}

We provide the proofs of Proposition~\ref{prop:score-invariance}, Lemma~\ref{lem:corr-core} and Proposition~\ref{prop:two-stage-opt}.

\subsubsection{Proof of Proposition~\ref{prop:score-invariance}}
\begin{proof}
Let $\bm{Z},\bm{Z}'\sim P$ be independent. Then $\bB\bm{Z},\bB\bm{Z}'\sim\bB_\#P$ are independent, and for all $\bz,\bz'\in\R^d$,
\[
k(\bB\bz,\bB\bz')=\kappa\big(\|\bB(\bz-\bz')\|_2\big)=\kappa\big(\|\bz-\bz'\|_{\bB^*\bB}\big)=k_{\bB}(\bz,\bz'),
\]
since $\|\bB\bm v\|_2^2=\bm v^*\bB^*\bB\bm v$. By definition of the kernel score, for the prediction $\bB_\#P$ and the observation $\bB\by$,
\begin{align*}
S_k(\bB_\#P,\bB\by)
&=\tfrac12\,\E\big[k(\bB\bm{Z},\bB\bm{Z}')\big]+\tfrac12\,k(\bB\by,\bB\by)-\E\big[k(\bB\by,\bB\bm{Z})\big] \\
&=\tfrac12\,\E\big[k_{\bB}(\bm{Z},\bm{Z}')\big]+\tfrac12\,k_{\bB}(\by,\by)-\E\big[k_{\bB}(\by,\bm{Z})\big]
=S_{k_{\bB}}(P,\by),
\end{align*}
which is the identity in (i). For the second claim in (i), assume that $S_k$ is strictly proper, i.e., that $k$ is characteristic. Let $Q$ be a further distribution on $\R^d$ and $\bW,\bW'\sim Q$ independent. By the same computation,
\[
\mathrm{MMD}_{k_{\bB}}^2(P,Q)
=\E\big[k_{\bB}(\bm{Z},\bm{Z}')\big]+\E\big[k_{\bB}(\bW,\bW')\big]-2\,\E\big[k_{\bB}(\bm{Z},\bW)\big]
=\mathrm{MMD}_{k}^2(\bB_\#P,\bB_\#Q).
\]
If $P\neq Q$, then $\bB_\#P\neq\bB_\#Q$ because $\bB$ is invertible, so the right-hand side is positive as $k$ is characteristic. Hence $k_{\bB}$ is characteristic, and since the divergence of $S_{k_{\bB}}$ equals $\tfrac12\mathrm{MMD}^2_{k_{\bB}}$, the score $S_{k_{\bB}}$ is strictly proper.

For (ii), $\bB$ unitary means $\bB^*\bB=\bI$, so $\|\cdot\|_{\bB^*\bB}=\|\cdot\|_2$ and $k_{\bB}=k$. The claim follows from (i).
\end{proof}

\subsubsection{Proof of Lemma~\ref{lem:corr-core}}
\begin{proof}
Write $\bm C=\bB^{*}\hat\bSigma\bB$ and $\bm\Delta\in\R^{d}$ with $\Delta_i:=C_{ii}$.
Since $\bB$ is invertible and $\hat\bSigma$ is Hermitian positive definite,
$\bm C$ is Hermitian positive definite as well. Indeed, $\bm C^{*}=\bm C$, and
$\bm v^{*}\bm C\bm v=(\bB\bm v)^{*}\hat\bSigma(\bB\bm v)>0$ for
$\bm v\neq\bm0$. Its diagonal entries $\Delta_i$ are therefore real and strictly positive, so
$\mathrm{corr}(\bm C)=\diag(\bm\Delta)^{-1/2}\,\bm C\,\diag(\bm\Delta)^{-1/2}$
is well defined. By the same congruence argument, now with the invertible
diagonal matrix $\diag(\bm\Delta)^{-1/2}$, $\mathrm{corr}(\bm C)$ is again
Hermitian positive definite, and its diagonal entries are
$\Delta_i^{-1}C_{ii}=1$. With $\bD_{\bm\tau}=\diag(\sqrt{\bm\tau})$ for
$\bm\tau\in\R^{d}_{>0}$, it follows that
$\big(\bD_{\bm\tau}\,\mathrm{corr}(\bm C)\,\bD_{\bm\tau}\big)_{ii}=\tau_i$
for every $i$.
\end{proof}

\subsubsection{Proof of Proposition~\ref{prop:two-stage-opt}}
\begin{proof}
For (i), the objective of the first stage is separable and its $i$-th summand is the expected $S_1$-score of $\cN(\mu_i,\tau_i)$ under the $i$-th marginal of $Q$, which is $\cN(\mu_{Q,i},(\bSigma_Q)_{ii})$ and thus lies in the candidate family. Since $S_1$ is strictly proper, the summand is minimized by the true marginal and by no other, so $\mu^{\star}_i=\mu_{Q,i}$ and $\tau^{\star}_i=(\bSigma_Q)_{ii}$, hence $\bD^{\star}=\bD_Q$.

For (ii), insert $\bmu^{\star}=\bmu_Q$ and $\bD^{\star}=\bD_Q$, and write
\[
  \bm m:=\bD_Q^{-1}\bmu_Q,\qquad
  \bm{Z}:=(\bD^{\star})^{-1}\bY\sim\cN(\bm m,R_Q),\qquad
  \bW:=\bm{Z}-\bm m\sim\cN(\bm0,R_Q),
\]
the law of $\bm{Z}$ being given by \eqref{eq:corr-fact}. The candidates of the second stage are $\cN(\bm m,R)$ with $R\in\mathcal M=\{\mathrm{corr}(\bB^{*}\hat\bSigma\bB)\colon\hat\bSigma\in\mathcal S\}$. Then
\begin{align*}
  R^{\star}
  &\in\operatorname*{arg\,min}_{R\in\mathcal M}\;
      \E\,S_2\big(\cN(\bm m,R),\,\bm{Z}\big)\\
  &=\operatorname*{arg\,min}_{R\in\mathcal M}\;
      \E\,S_2\big(\cN(\bm0,R),\,\bW\big)
   &&\text{target and candidates share $\bm m$}\\
  &=\operatorname*{arg\,min}_{R\in\mathcal M}\;
      \Big[\E\,S_2\big(\cN(\bm0,R),\,\bW\big)
          -\E\,S_2\big(\cN(\bm0,R_Q),\,\bW\big)\Big]
   &&\text{subtracted term free of $R$}\\
  &=\operatorname*{arg\,min}_{R\in\mathcal M}\;
      D\big(\cN(\bm0,R),\,\cN(\bm0,R_Q)\big),
\end{align*}
where the second step is the translation invariance of $S_2$ and the last step is the definition of the divergence, the subtracted term being the score entropy. Hence $R^{\star}$ minimizes $D(\cN(\bm0,R),\cN(\bm0,R_Q))$ over $\mathcal M$. The deployed prediction is $\cN(\bmu_Q,\bD_QR^{\star}\bD_Q)$, whose marginal variances are $(\bSigma_Q)_{ii}$ by Lemma~\ref{lem:corr-core}.
\end{proof}

\subsection{Implementation}
\label{app:implementation}
We first describe how the network outputs are turned into a prediction, then
the closed forms for each stage.

\subsubsection{From network outputs to a prediction}
\label{app:nn_predictions}
All predictions are conditional on the input. A single network maps the input to all parameters at once,
\[
  \bx\;\longmapsto\;\big(\bmu(\bx),\,\bm\tau(\bx),\,\hat\bSigma(\bx)\big),
\]
and the predictive law at $\bx$ is $\cN\big(\bmu(\bx),\,\bSigma_P(\bx)\big)$
with $\bSigma_P$ assembled from these outputs by \eqref{eq:core_param_corr}.
Note that the staged objective of Section~\ref{sec:score} places no
constraint on how the network produces them. We specify the architecture
per benchmark in Appendix~\ref{app:experiments}. However, common to all of
them is the output head. For a diagonal core it emits three groups of $d$ values,
$(\bm h_\mu,\bm h_\tau,\bm h_s)$, mapped to the parameters by
\begin{equation}\label{eq:head-map}
  \bmu=\bm h_\mu,
  \qquad
  \bm\tau=\mathrm{softplus}(\bm h_\tau)+\epsilon,
  \qquad
  \bm s=d\,\mathrm{softmax}(\bm h_s),
\end{equation}
where $\mathrm{softplus}(x)=\log(1+e^{x})$ acts entrywise with a floor $\epsilon>0$ for numerical safety and the softmax runs over the coefficient index, $\mathrm{softmax}(\bm h)_k=e^{h_k}/\sum_l e^{h_l}$. The softplus makes the variances positive and the softmax makes the spectrum positive with $\sum_k s_k=d$, hence $\bar s=1$. For the DFT, the spectrum is subsequently symmetrized, $s_k\leftarrow\tfrac12(s_k+s_{d-k})$, which preserves $\sum_k s_k=d$ and makes $\bF^{*}\diag(\bm s)\bF$ real. The floor $\bm s\leftarrow(\bm s+\epsilon)/(1+\epsilon)$ keeps the core positive definite. The prediction is then assembled as in \eqref{eq:core_param_corr}, which only requires the implied marginal variances $\Delta_i=(\bB^{*}\hat\bSigma\bB)_{ii}$,
\begin{equation}\label{eq:head-assembly}
  \bm\Delta=\bm H^{\top}\bm s,
  \qquad
  \bm H:=\big(\lvert B_{ki}\rvert^{2}\big)_{k,i},
  \qquad
  \tilde\bD=\diag\big(\sqrt{\bm\tau/\bm\Delta}\big),
  \qquad
  \bSigma_P=\tilde\bD\,\bB^{*}\diag(\bm s)\bB\,\tilde\bD .
\end{equation}
For the DFT $\lvert F_{ki}\rvert^{2}=1/d$, so $\bm\Delta=\bar s\,\bm1=\bm1$, $\tilde\bD=\diag(\sqrt{\bm\tau})$, and the renormalization is the identity. When the target is a $d_1\times d_2$ grid and the transform is separable, $\bB=\bB_1\otimes\bB_2$ acting along its axes, then $\bm H=\bm H_1\otimes\bm H_2$ with $\bm H_j=(\lvert(B_j)_{ki}\rvert^{2})_{k,i}$, so $\bm\Delta$ reshaped to the grid equals $\bm H_1^{\top}\bm S\,\bm H_2$ with $\bm S$ the spectrum reshaped likewise. This is one small matrix product per axis, and $\bB$ itself is never materialized.

\paragraph{Low-rank correction.}
The diagonal core is our default case, although the construction of
\score{}, and in particular Lemma~\ref{lem:corr-core}, is more general.
A straightforward extension is to add a low-rank term. A correction of
rank $r\ll d$ buys $r$ directions beyond what $\bB$ diagonalizes. For the
DFT these are $r$ non-stationary directions on a stationary background.
The head emits $r$ further groups $\bm u_j=\bm h_{u,j}\in\R^{d}$,
$j=1,\dots,r$, collected in $\bU=(\bm u_1,\dots,\bm u_r)$, without any
constraint, and the transformed core becomes
\begin{equation}\label{eq:lora-core}
  \bB^{*}\hat\bSigma\bB
  \;=\;
  \bB^{*}\diag(\bm s)\bB+\sum_{j=1}^{r}\bm u_j\bm u_j^{\top},
\end{equation}
which is real by construction, the factors being real fields in the physical
domain. In particular, no conjugate symmetry constraint is needed here. The
assembly gains one term per factor, that is
\begin{equation}\label{eq:lora-assembly}
  \bm\Delta=\bm H^{\top}\bm s+\sum_{j=1}^{r}\bm u_j\odot\bm u_j,
  \qquad
  \bSigma_P=\tilde\bD\,\bB^{*}\hat\bSigma\bB\,\tilde\bD,
  \qquad
  \hat\bSigma=\diag(\bm s)+\hat\bU\hat\bU^{*},
  \quad
  \hat\bU:=\bB\bU .
\end{equation}
Note that the scale gauge now acts jointly,
$(\bm s,\bU)\mapsto(c\,\bm s,\sqrt c\,\bU)$, and is again fixed by $\bar s=1$.

\subsubsection{Scores and closed forms}
\label{app:score_closed_forms}

The first stage scores the $d$ marginals, the second
a $d$-dimensional Gaussian assembled from the core, which calls for different closed forms. Throughout, $\bD:=\diag(\sqrt{\bm\tau})$
denotes the marginal scales, and $\bm m=\bD^{-1}\bmu$; $\bm Z=\bD^{-1}\bY$ are
the standardized mean and target of \eqref{eq:stage2}.

\paragraph{First stage.}
As first stage default score we use the CRPS, which for a univariate Gaussian is available
in closed form. With $z_i:=(y_i-\mu_i)/\sqrt{\tau_i}$ and $\varphi,\Phi$ the
standard normal density and distribution function,
\begin{equation}\label{eq:crps-closed}
  \mathrm{CRPS}\big(\cN(\mu_i,\tau_i),y_i\big)
  \;=\;\sqrt{\tau_i}\,
  \Big[\,z_i\big(2\Phi(z_i)-1\big)+2\varphi(z_i)-\tfrac{1}{\sqrt\pi}\,\Big],
\end{equation}
summed over the $d$ coordinates at $\mathcal O(d)$. 

\paragraph{Second stage.}
Both the log score and the kernel score have closed forms for a $d$-dimensional Gaussian predictive. For $\cN(\bmu,\bSigma)$ observed at $\by$, and the log score,
\begin{equation}\label{eq:log-closed}
  2\,S_{\log}\big(\cN(\bmu,\bSigma),\by\big)
  \;=\;
  d\log2\pi+\log\det\bSigma
  +(\by-\bmu)^{\top}\bSigma^{-1}(\by-\bmu),
\end{equation}
and, for the Gaussian kernel $k(\bm u,\bm v)=\exp\!\big(-\lVert\bm u-\bm v\rVert_2^{2}/\gamma^{2}\big)$ with bandwidth $\gamma>0$ (cf.\ Appendix~\ref{app:ks_closed_form}),
\begin{equation}\label{eq:ks-closed}
  S_{k}\big(\cN(\bmu,\bSigma),\by\big)
  \;=\;
  \tfrac12\Big(1+\det\big(\bI+\tfrac{4}{\gamma^{2}}\bSigma\big)^{-1/2}\Big)
  \;-\;
  \det\big(\bI+\tfrac{2}{\gamma^{2}}\bSigma\big)^{-1/2}
  \exp\Big(\!-(\by-\bmu)^{\top}
    \big(\gamma^{2}\bI+2\bSigma\big)^{-1}(\by-\bmu)\Big).
\end{equation}
Each needs one or two determinants and one quadratic form. We now show that these pass to the core. Recall that the deployed covariance \eqref{eq:core_param_corr} reads
\begin{equation*}
\bSigma_P=\tilde\bD\,\bB^{*}\hat\bSigma\,\bB\,\tilde\bD,\qquad
  \tilde\bD=\diag\big(\sqrt{\bm\tau/\bm\Delta}\big),\qquad
  \hat\bSigma=\diag(\bm s)+\hat\bU\hat\bU^{*},\quad
  \hat\bU:=\bB\bU,
\end{equation*}
with the implied marginal variances $\Delta_i=(\bB^{*}\hat\bSigma\bB)_{ii}$ of \eqref{eq:head-assembly}, the physical factors $\bU=(\bm u_1,\dots,\bm u_r)$ of \eqref{eq:lora-core}, and $r=0$ for a diagonal core. For the DFT with a diagonal core, $\bm\Delta=\bar s\,\bm1=\bm1$ by the softmax normalization \eqref{eq:head-map}, so $\tilde\bD=\bD$, the core is $\bm C=\bF^{*}\diag(\bm s)\bF$, and the second stage scores $\cN(\bm m,\bm C)$ at $\bm Z$. Both quantities that \eqref{eq:log-closed} have closed forms, such that
\begin{equation}\label{eq:log-reduction}
  \log\det\bm C=\sum_k\log s_k,
  \qquad
  (\bm Z-\bm m)^{\top}\bm C^{-1}(\bm Z-\bm m)
  =\sum_k\frac{w_k}{s_k},
  \qquad
  w_k:=\big\lvert\big(\bF(\bm Z-\bm m)\big)_k\big\rvert^{2},
\end{equation}
with $w_k$ the periodogram of the standardized residual, so that
\begin{equation}\label{eq:log-decoupled}
  2\,S_{\log}\big(\cN(\bm m,\bm C),\bm Z\big)
  \;=\;d\log2\pi+\sum_k\Big(\log s_k+\frac{w_k}{s_k}\Big).
\end{equation}
Note that \eqref{eq:log-decoupled} is also known as the Whittle likelihood~\citep{Whittle1953}, which approximates the likelihood of a stationary Gaussian time series by replacing its Toeplitz covariance with a circulant one. In our setting it is exact, since the core is diagonal by construction, i.e., the back-transformed covariance is already circulant.

Since the standardization is a diagonal change of variables, scoring the deployed prediction instead of the core only adds the constant $\sum_i\log\tau_i$ fixed by the first stage, $2\,S_{\log}\big(\cN(\bmu,\bD\bm C\bD),\by\big)=2\,S_{\log}\big(\cN(\bm m,\bm C),\bm Z\big)+\sum_i\log\tau_i$. For the log score, the two choices are therefore equivalent, at $\mathcal O(d)$ beyond computing the DFT.

For the kernel score the situation is different. By Proposition~\ref{prop:score-invariance}\,(i), the two choices correspond to kernel scores with different metrics, since $\bD$ is not unitary. The second stage therefore scores the core at the standardized residual. Still, it does reduce, as the shift passes through the unitary $\bF$, that is
\begin{equation}\label{eq:ks-core}
  \gamma^{2}\bI+c\,\bm C
  =\bF^{*}\big(\gamma^{2}\bI+c\,\diag(\bm s)\big)\bF ,
  \qquad c\in\{2,4\}.
\end{equation}
Both determinants of \eqref{eq:ks-closed} become products and the linear solve becomes a division. Combining yields
\begin{equation}\label{eq:ks-decoupled}
  S_{k}\big(\cN(\bm m,\bm C),\bm Z\big)
  =\tfrac12\Big(1+\prod_k\big(1+\tfrac{4s_k}{\gamma^{2}}\big)^{-1/2}\Big)
  -\prod_k\big(1+\tfrac{2s_k}{\gamma^{2}}\big)^{-1/2}
   \exp\Big(\!-\sum_k\frac{w_k}{\gamma^{2}+2s_k}\Big),
\end{equation}
again $\mathcal O(d)$ beyond computing the DFT. Including the transform, one evaluation of either objective costs $\mathcal O(d\log d)$ via the FFT, with $\mathcal O(d)$ memory for the spectrum and the marginals, cf. \score{} in Table~\ref{tab:complexity}.

\paragraph{General transforms and low-rank cores.}
The log score keeps its closed form for every unitary $\bB$ and any rank $r$.
Again, the renormalization in \eqref{eq:lora-assembly} is a diagonal change of
variables, so it contributes only the additive term $2\sum_i\log\tilde D_{ii}$,
and unitarity of $\bB$ gives
\begin{equation}\label{eq:log-reduction-general}
  \log\det\bSigma_P=2\sum_i\log\tilde D_{ii}+\log\det\hat\bSigma,
  \qquad
\bSigma_P^{-1}=\tilde\bD^{-1}\bB^{*}\hat\bSigma^{-1}\bB\,\tilde\bD^{-1}.
\end{equation}
A low-rank core keeps that structure through the matrix determinant lemma and
the Woodbury identity,
\begin{align}
  \det\hat\bSigma
  &=\Big(\prod_{k}s_k\Big)\,
    \det\!\big(\bI_r+\hat\bU^{*}\diag(\bm s)^{-1}\hat\bU\big),
  \label{eq:lora-det}\\
  \hat\bSigma^{-1}
  &=\diag(\bm s)^{-1}-\diag(\bm s)^{-1}\hat\bU\big(\bI_r
    +\hat\bU^{*}\diag(\bm s)^{-1}\hat\bU\big)^{-1}
    \hat\bU^{*}\diag(\bm s)^{-1},
  \label{eq:lora-woodbury}
\end{align}
at the price of $r$ further transforms for $\hat\bU$, $\mathcal O(r^{2}d)$ to
assemble the $r\times r$ matrix, and $\mathcal O(r^{3})$ to factorize it. The kernel score does not extend likewise. 

\paragraph{Sampling.}
In the DFT case with diagonal core, i.e. $\bB = \bF$ and $r = 0$, a draw is formed by
\[
  \by=\bmu+\bD\,\bF^{*}\diag(\sqrt{\bm s})\bF\,\bm\xi,
  \qquad
  \bm\xi\sim\cN(\bm0,\bI_d),
\]
at $\mathcal O(d\log d)$ with the FFT. The symmetrization in Appendix~\ref{app:nn_predictions} enforces $s_{d-k}=s_k$, which makes the square root, and hence $\by$, real. In general, draws come from \eqref{eq:lora-assembly} the same way,
\[
  \by=\bmu+\tilde\bD\big(\bB^{*}\diag(\sqrt{\bm s})\bB\,\bm\xi_1
  +\bU\bm\xi_2\big),
  \qquad
  \bm\xi_1\sim\cN(\bm0,\bI_d),\quad
  \bm\xi_2\sim\cN(\bm0,\bI_r)\ \text{independent},
\]
at $\mathcal O(d\log d+dr)$ each for the DFT.

\paragraph{Bandwidth selection.}
The kernel score leaves the choice of the bandwidth $\gamma>0$ open. We set it by (the mean version of) the median heuristic \citep{gretton2012kernel,garreau2018largesampleanalysismedian}, which matches $\gamma^{2}/2$ to the expected squared distance between two independent draws from the predictive distribution at a fixed input. In the second stage the scored quantity is the standardized residual, whose covariance is a correlation matrix, and the heuristic then resolves without data dependence.

\begin{lemma}[Bandwidth in standardized coordinates]\label{lem:bandwidth}
Let $R$ be a correlation matrix on $\R^{d}$ and let $\bm Z,\bm Z'$ be independent draws from $\cN(\bm m,R)$. Then, with the expectation taken over $(\bm Z,\bm Z')$,
\[
  \E\lVert\bm Z-\bm Z'\rVert^{2}=2\tr R=2d ,
\]
so matching $\gamma^{2}/2$ to this expectation gives $\gamma^{2}=4d$, independently of $R$, of the target, and of the input.
\end{lemma}
\begin{proof}
$\bm Z-\bm Z'\sim\cN(\bm0,2R)$, hence $\E\lVert\bm Z-\bm Z'\rVert^{2}=\tr(2R)=2d$ since $R_{ii}=1$, as it is a correlation matrix.
\end{proof}
The median heuristic only sets an order of magnitude. In practice, we therefore average the second-stage loss over a ladder $\gamma^{2}=f\,d$, $f\in\{0.25,0.5,1,2,5\}$, around this value, which keeps the score proper. We always evaluate at $\gamma^{2}=d$ in Appendix~\ref{app:experiments}.

\section{Expressivity of \score{}}
\label{app:expressivity}

Proposition~\ref{prop:two-stage-opt} states that under exact marginals \score{} fits the correlation alone, as a projection of $R_Q$ onto the representable class in the respective divergence. Here, we quantify the misspecification that remains within the Gaussian class. Everything is stated at a fixed input $\bx\in\mathcal X$, as elsewhere, since the representable class is the same at every input, and we take $\bB=\bF$ throughout. Recall that $Q=\cN(\bmu_Q,\bSigma_Q)$ with $\bSigma_Q\succ0$ denotes the target, factored as $\bSigma_Q=\bD_Q R_Q\bD_Q$ as in \eqref{eq:corr-fact}.

Besides our default diagonal core, we also consider the low-rank extension
of \eqref{eq:lora-core}. Diagonal and diagonal-plus-low-rank cores of rank
$r$ give the correlation classes
\[
  \mathcal M^{(r)}
  \;\coloneq\;
  \big\{\mathrm{corr}\big(\bF^{*}\hat\bSigma\bF\big)
  \;:\;\hat\bSigma\in\mathcal S^{(r)}\big\},
  \qquad
  \mathcal S^{(r)}\;\coloneq\;\big\{\diag(\bm s)+\bU\bU^{*}:
  \bm s\in\R_{>0}^{d},\ \bU\in\C^{d\times r}\big\},
\]
restricted to their real members. We write
$\mathcal M:=\mathcal M^{(0)}$ for the diagonal case.

We call the quantity of interest the \emph{approximation error},
\begin{equation}\label{eq:model-error}
  \varepsilon^{(r)}_{D}(Q)
  \;:=\;
  \inf_{R\in\mathcal M^{(r)}}
  D\big(\cN(\bm0,R),\,\cN(\bm0,R_Q)\big),
\end{equation}
which is the error $\varepsilon_{\mathrm{app}}$ of Section~\ref{sec:score} at the input where the target is $Q$, with the divergence $D$ of the respective score and the rank $r$ of the core made explicit. We drop the superscript $(r)$ whenever $r=0$. As before, we treat the Gaussian kernel score and the log score, whose divergences are $\tfrac12\mathrm{MMD}^{2}$ for the Gaussian kernel $k(\bm u,\bm v)=\exp\!\big(-\lVert\bm u-\bm v\rVert^{2}/\gamma^{2}\big)$ and the Kullback--Leibler divergence $\mathrm{KL}\big(\cN(\bm0,R_Q)\,\Vert\,\cN(\bm0,R)\big)$, respectively.

\paragraph{Gaussian kernel score.}
Consider the Gaussian kernel with bandwidth $\gamma>0$. It turns out that its approximation error is controlled by how well the class approximates the true correlation $R_Q$ in Frobenius norm,
\[
  \Psi^{(r)}(Q)
  \;:=\;
  \inf_{R\in\mathcal M^{(r)}}
  \big\lVert R-R_Q\big\rVert_{F}^{2},
\]
and in the diagonal case this infimum is the off-diagonal energy of the target in the Fourier domain.

\begin{lemma}[Expressivity under the kernel score]
\label{lem:mmd-frobenius}
Let $0\le r<d$.
\begin{enumerate}[leftmargin=25pt]
\item[(i)] For every $\gamma>0$, the approximation error is equivalent, up to constants, to the Frobenius projection error $\Psi^{(r)}(Q)$,
      \begin{equation}\label{eq:mmd-sandwich}
        \big(1+\tfrac{8d}{\gamma^{2}}\big)^{-\frac{d+4}{2}}
        \,\frac{\Psi^{(r)}(Q)}{\gamma^{4}}
        \;\le\;
        \varepsilon^{(r)}_{\mathrm{MMD}}(Q)
        \;\le\;
        \frac{\Psi^{(r)}(Q)}{\gamma^{4}},
      \end{equation}
      so it vanishes if and only if the projection error does, and $\gamma^{4}\varepsilon^{(r)}_{\mathrm{MMD}}(Q)\to\Psi^{(r)}(Q)$ as $\gamma\to\infty$.
\item[(ii)] For $r=0$ the projection is attained at $R^{\star}_{F}=\bF^{*}\diag(\bm t^{\star})\bF$ with $t^{\star}_k=(\bF R_Q\bF^{*})_{kk}$, and
      \[
        \Psi(Q)=\sum_{k\neq l}\big\lvert(\bF R_Q\bF^{*})_{kl}\big\rvert^{2}.
      \]
\end{enumerate}
\end{lemma}

\begin{proof}
\emph{(i)} By Bochner's theorem, $k(\bx,\by)=\E\,e^{i\bm\omega^{\top}(\bx-\by)}$ with
$\bm\omega\sim\cN(\bm0,\tfrac{2}{\gamma^{2}}\bI)$, so
$\mathrm{MMD}_k^{2}=\E\,\lvert\phi_R(\bm\omega)-\phi_{R_Q}(\bm\omega)\rvert^{2}$
for the characteristic functions $\phi_R(\bm\omega)=e^{-\bm\omega^{\top}R\bm\omega/2}$
of the two centered Gaussians. Substituting $\bm\omega=\sqrt2\,\bm z/\gamma$,
\[
  \mathrm{MMD}_k^{2}
  =\E_{\bm z\sim\cN(\bm0,\bI)}\big(e^{-a}-e^{-b}\big)^{2},
  \qquad
  a=\frac{\bm z^{\top}R\bm z}{\gamma^{2}},\quad
  b=\frac{\bm z^{\top}R_Q\bm z}{\gamma^{2}} .
\]
Both $R$ and $R_Q$ are real correlation matrices, so $\bm N:=R-R_Q$ is
symmetric and
$\E(\bm z^{\top}\bm N\bm z)^{2}=2\lVert\bm N\rVert_F^{2}+(\tr\bm N)^{2}
=2\lVert\bm N\rVert_F^{2}$. Since $a,b\ge0$, the mean value theorem gives
$e^{-a}-e^{-b}=-(a-b)e^{-\xi}$ with $\xi$ between $a$ and $b$, hence
\[
  e^{-2(a+b)}(a-b)^{2}\;\le\;\big(e^{-a}-e^{-b}\big)^{2}\;\le\;(a-b)^{2}
  =\frac{\big(\bm z^{\top}\bm N\bm z\big)^{2}}{\gamma^{4}}.
\]
Taking expectations, the right inequality yields
$\mathrm{MMD}_k^{2}\le2\lVert\bm N\rVert_F^{2}/\gamma^{4}$. For the left one,
the largest eigenvalue of a correlation matrix is at most its trace $d$, so
$a+b\le2d\lVert\bm z\rVert^{2}/\gamma^{2}$ and, with $\beta:=1+8d/\gamma^{2}$,
\begin{align*}
  \E\big[e^{-2(a+b)}(a-b)^{2}\big]
  &\;\ge\;\frac{1}{\gamma^{4}}\int
    \frac{e^{-\lVert\bm z\rVert^{2}/2}}{(2\pi)^{d/2}}\,
    e^{-4d\lVert\bm z\rVert^{2}/\gamma^{2}}
    \big(\bm z^{\top}\bm N\bm z\big)^{2}\,\mathrm d\bm z \\
  &\;=\;\frac{\beta^{-d/2}}{\gamma^{4}}\,
    \E_{\bm z\sim\cN(\bm0,\beta^{-1}\bI)}\big(\bm z^{\top}\bm N\bm z\big)^{2}
  \;=\;\frac{2\,\beta^{-\frac{d+4}{2}}\lVert\bm N\rVert_F^{2}}{\gamma^{4}},
\end{align*}
since $e^{-\lVert\bm z\rVert^{2}/2}e^{-4d\lVert\bm z\rVert^{2}/\gamma^{2}}
=e^{-\beta\lVert\bm z\rVert^{2}/2}$ is $\beta^{-d/2}$ times the density of
$\cN(\bm0,\beta^{-1}\bI)$, and under that law $\bm z=\beta^{-1/2}\bm z'$ with
$\bm z'$ standard normal, so the second moment of the quadratic form is
$\beta^{-2}\cdot2\lVert\bm N\rVert_F^{2}$.

Hence
$\mathrm{MMD}_k^{2}\ge2\big(1+\tfrac{8d}{\gamma^{2}}\big)^{-\frac{d+4}{2}}\lVert\bm N\rVert_F^{2}/\gamma^{4}$.
Both bounds are monotone in $\lVert\bm N\rVert_F^{2}$, so taking the infimum
over $\mathcal M^{(r)}$ and the factor $\tfrac12$ of
$D=\tfrac12\mathrm{MMD}_k^{2}$ proves \eqref{eq:mmd-sandwich}. Finally, with
$c(\gamma):=\beta^{-\frac{d+4}{2}}$, multiplying \eqref{eq:mmd-sandwich} by
$\gamma^{4}$ gives
\[
  c(\gamma)\,\Psi^{(r)}(Q)\;\le\;\gamma^{4}\varepsilon^{(r)}_{\mathrm{MMD}}(Q)
  \;\le\;\Psi^{(r)}(Q),
\]
and $c(\gamma)\to1$ as $\gamma\to\infty$, which proves the limit.

\emph{(ii)} Since $\lvert F_{ki}\rvert^{2}=1/d$, the deployed core is
$\bF^{*}\diag(\bm t)\bF$ with $\bm t=\bm s/\bar s$, and by unitary invariance
\[
  \lVert R-R_Q\rVert_{F}^{2}
  =\lVert\diag(\bm t)-\bF R_Q\bF^{*}\rVert_{F}^{2}
  =\sum_k\big(t_k-(\bF R_Q\bF^{*})_{kk}\big)^{2}
  +\sum_{k\neq l}\lvert(\bF R_Q\bF^{*})_{kl}\rvert^{2},
\]
whose first sum vanishes at $t_k=(\bF R_Q\bF^{*})_{kk}$. Note that this choice is
admissible. Indeed, $t^{\star}_k>0$ since $R_Q\succ0$,
$\bar t^{\star}=\tfrac1d\tr(\bF R_Q\bF^{*})=\tfrac1d\tr R_Q=1$, and
$t^{\star}_{d-k}=t^{\star}_k$ because $R_Q$ is real. What remains is the
off-diagonal energy.
\end{proof}

\noindent Part (i) uses nothing about $\bF$ and holds for every unitary $\bB$. Its lower constant decays exponentially in $d$. At the training bandwidth $\gamma^{2}=4d$ it equals $3^{-(d+4)/2}$, already of order $10^{-24}$ at $d=96$. Although it can be refined to a dimension-independent constant at this bandwidth (Lemma~\ref{lem:pac-frobenius}\,(ii)), the sandwich should be read as a qualitative statement, establishing that approximation error and Frobenius projection vanish together.

\paragraph{Logarithmic score.}
For the log score, it turns out that the approximation error comes down to the spread of the
eigenvalues of the target correlation $R_Q$ after transforming and renormalizing it to
$\mathrm{corr}(\bF R_Q\bF^{*})$.
\begin{lemma}[Diagonal core under the log score]
\label{lem:diag-kl}
Denote by $\lambda_1,\dots,\lambda_d>0$ the eigenvalues of
$\mathrm{corr}(\bF R_Q\bF^{*})$ and by $\mathrm{GM}:=(\prod_k\lambda_k)^{1/d}$
their geometric mean. Then
\begin{equation}\label{eq:kl-gap-flatness}
  \frac1d\,\varepsilon_{\mathrm{KL}}(Q)
  \;=\;-\frac{1}{2d}\log\det\mathrm{corr}(\bF R_Q\bF^{*})
  \;=\;-\frac12\log\mathrm{GM}.
\end{equation}
The infimum is attained at $R^{\star}=R^{\star}_{F}=\bF^{*}\diag(\bm t^{\star})\bF$ with $t^{\star}_k=(\bF R_Q\bF^{*})_{kk}$, the Frobenius projection of Lemma~\ref{lem:mmd-frobenius}\,(ii), and this minimizer is unique. In particular, the error vanishes if and only if $\bF$ diagonalizes $R_Q$.
\end{lemma}
\begin{proof}
The correlation normalization maps a spectrum $\bm s$ to $\bm t=\bm s/\bar s$,
so up to the scale gauge $\bm s\mapsto c\,\bm s$ the class is
$\mathcal M=\{\bF^{*}\diag(\bm t)\bF:\bm t>0,\ \bar t=1\}$. The divergence of
the log score is
$\mathrm{KL}\big(\cN(\bm0,R_Q)\Vert\cN(\bm0,R)\big)
=\tfrac12\big[\tr(R^{-1}R_Q)-d+\log\det R-\log\det R_Q\big]$, and trace and
determinant are invariant under conjugation by $\bF$, so for
$R=\bF^{*}\diag(\bm t)\bF$ this equals
\[
  \tfrac12\sum_k\left[\frac{(\bF R_Q\bF^{*})_{kk}}{t_k}+\log t_k\right]
  -\tfrac12\big[d+\log\det R_Q\big].
\]
The sum is separable and strictly convex in $\log t_k$, with minimizer
$t^{\star}_k=(\bF R_Q\bF^{*})_{kk}$. It is positive since $R_Q\succ0$,
satisfies $\bar t^{\star}=\tfrac1d\tr(\bF R_Q\bF^{*})=\tfrac1d\tr R_Q=1$, and
obeys $t^{\star}_{d-k}=t^{\star}_k$ since $R_Q$ is real. Hence $R^{\star}$ is
admissible, and by strict convexity it is the unique minimizer.

At $t^{\star}$ each term $(\bF R_Q\bF^{*})_{kk}/t^{\star}_k$ equals one, so the
first sum contributes $d+\sum_k\log(\bF R_Q\bF^{*})_{kk}$ and the value is
$\tfrac12\big[\sum_k\log(\bF R_Q\bF^{*})_{kk}-\log\det R_Q\big]
=-\tfrac12\log\det\mathrm{corr}(\bF R_Q\bF^{*})=-\tfrac12\sum_k\log\lambda_k
=-\tfrac d2\log\mathrm{GM}$, which is \eqref{eq:kl-gap-flatness}. Since
$\mathrm{corr}(\bF R_Q\bF^{*})$ has unit diagonal, $\tfrac1d\sum_k\lambda_k=1$,
so $\mathrm{GM}\le1$ with equality if and only if all $\lambda_k=1$, that is,
if and only if $\bF R_Q\bF^{*}$ is diagonal.
\end{proof}
Summarizing, in the diagonal case both approximation errors penalize the off-diagonal energy of $\bF R_Q\bF^{*}$, the kernel score directly through \eqref{eq:mmd-sandwich} and the log score through the spectral flatness $\mathrm{GM}$ of $\mathrm{corr}(\bF R_Q\bF^{*})$. They agree on when the error vanishes, namely when $\bF$ diagonalizes $R_Q$, but differ in how they scale what remains.

\section{PAC Learnability}
\label{app:pac}
Proposition~\ref{prop:two-stage-opt} describes the population limit of \score{} under an exact first stage. At finite sample size the first stage
is not exact, and stage 2 is fitted to residuals standardized by the
learned marginals. One might worry that the second stage amplifies the
first-stage error. This section shows that it does not. Instead, the error
enters additively, through terms that are linear in the marginal errors and
independent of the second-stage sample and of the core class.
We restrict ourselves to the
case the paper focuses on, which is the DFT with a diagonal core
($\bB=\bF$, $r=0$), the CRPS in the first stage, and the Gaussian kernel score
in the second. As is usual in PAC analyses, we work under idealized conditions. In particular, we assume exact empirical risk minimization and disjoint
samples of sizes $N$ and $M$ for the two stages. The latter is a standard device for
two-step estimators whose second step use the output of a learned first
step~\citep{klaassen1987consistent,chernozhukov2018double}. 

\subsection{Setting and Assumptions}
Let $N,M\in\mathbb N$ be the sample budgets of the first and second stage, respectively, and let $(\bX_n,\bY_n)$, $n\le N+M$, be independent draws from $\bP=\bP_{(\bX,\bY)}$. We write $\mathcal T_N:=(\bX_n,\bY_n)_{n=1}^{N}$ and $\mathcal T_M:=(\bX_n,\bY_n)_{n=N+1}^{N+M}$ for the samples used in the two stages. We assume that the conditional law is Gaussian,
\[
  Q(\bx):=\bP_{\bY\mid\bX=\bx}=\cN\big(\bmu_Q(\bx),\bSigma_Q(\bx)\big),
  \qquad \bSigma_Q(\bx)\succeq0,\quad \bx\in\cX ,
\]
with marginal variances $\bm\tau_Q(\bx):=\big((\bSigma_Q(\bx))_{ii}\big)_{i=1}^{d}$. As in \eqref{eq:corr-fact}, we write $\bSigma_Q(\bx)=\bD_Q(\bx)R_Q(\bx)\bD_Q(\bx)$ with marginal scales $\bD_Q(\bx)=\diag\big(\sqrt{\bm\tau_Q(\bx)}\big)$ and correlation $R_Q(\bx)$, and omit $\bx$ from the notation whenever a statement is pointwise in the input.
Fitted quantities carry the size of the sample they were fitted on as a
subscript, so $(\bmu_N,\bm\tau_N)$ come from the first stage and $R_M$ from
the second.
For $\gamma>0$ let $k_\gamma$ be the Gaussian kernel
$k_\gamma(\bm u,\bm v):=\exp\big(-\lVert\bm u-\bm v\rVert_2^{2}/\gamma^{2}\big)$,
and for probability measures $P,P'$ on $\R^{d}$ write
\[
  \mathrm{MMD}^{2}_{\gamma}(P,P')
  :=\E\,k_\gamma(\bm U,\bm U')+\E\,k_\gamma(\bm V,\bm V')
    -2\,\E\,k_\gamma(\bm U,\bm V),
\]
with $\bm U,\bm U'\sim P$ and $\bm V,\bm V'\sim P'$ all independent, so that
the divergence of the kernel score $S_{k_\gamma}$ is
$D_\gamma=\tfrac12\mathrm{MMD}^{2}_{\gamma}$. In our analysis we use the bandwidth $\tilde\gamma^{2}=4d$, as motivated by
Lemma~\ref{lem:bandwidth}, to fit the core to the standardized residual in the
second stage. We restrict the target to a
bounded range.
\begin{assumption}[Target]\label{ass:target}
There are $B>0$ and $0<\tau_{\min}\le\tau_{\max}<\infty$ such that
$\bmu_Q(\bx)\in[-B,B]^{d}$ and
$\bm\tau_Q(\bx)\in[\tau_{\min},\tau_{\max}]^{d}$ for every $\bx\in\cX$.
\end{assumption}
\paragraph{Admissible cores.} For the DFT and a diagonal core, the second stage optimizes over spectra $\bm s\in\Theta$. Indexing the DFT frequencies by $k=0,\dots,d-1$, we define the set of admissible spectra by
\[
  \Theta:=\Big\{\bm s\in\R^{d}\ :\ s_k\ge 0,\quad
  \tfrac1d\textstyle\sum_{k}s_k=1,\quad
  s_{d-k}=s_k\ \text{ for } k=1,\dots,d-1\Big\}.
\]
Note that $\Theta$ is the intersection of the scaled simplex $\{\bm s\ge\bm0,\ \sum_k s_k=d\}$ with the linear subspace of reflection-symmetric vectors, hence \emph{convex and compact}. With $\bF$ the unitary DFT on the $d$ grid points, the core with spectrum $\bm s$ is
\begin{equation}\label{eq:pac-core}
  R_{\bm s}:=\bF^{*}\diag(\bm s)\bF,\qquad \bm s\in\Theta,
\end{equation}
and the reflection condition makes $R_{\bm s}$ real. Moreover, the condition $\tfrac1d\sum_k s_k=1$ together with $\lvert F_{ki}\rvert^{2}=1/d$ gives it unit diagonal, so $R_{\bm s}$ is a correlation matrix.
\paragraph{The network.} Both stages optimize over maps realized by a single network, $\phi\mapsto(\bmu_\phi,\bm\tau_\phi,\bm s_\phi)$ with $\phi\in\mathcal W$ the parameters. We write
\[
  \mathcal G_1=\big\{(\bmu_\phi,\bm\tau_\phi):\phi\in\mathcal W\big\},
  \qquad
  \mathcal G_2=\big\{\bm s_\phi:\phi\in\mathcal W\big\},
\]
for the two classes, with $(\bmu_\phi,\bm\tau_\phi):\cX\to\R^{d}\times\R^{d}_{>0}$ and $\bm s_\phi:\cX\to\Theta$. The restriction of the image of $\bm s_\phi$ to $\Theta$ can be implemented by a softmax normalization, and the reflection condition by emitting a reflection-symmetric spectrum, as discussed in Appendix~\ref{app:nn_predictions}. The analysis treats $\mathcal G_1$ and $\mathcal G_2$ as separate classes, which corresponds to separate heads for $(\bmu,\bm\tau)$ and $\bm s$. We assume the following regularity conditions on the network.

\begin{assumption}[Network class]\label{ass:class}
Let the parameter space be $\mathcal W:=[-W,W]^{q}$ for some $W>0$ and
$q\in\mathbb N$, and let every $\phi\in\mathcal W$ give measurable maps
$\bmu_\phi,\bm\tau_\phi,\bm s_\phi$ on $\cX$. We assume:
\begin{enumerate}[leftmargin=25pt]
\item[(i)] the network stays within the range of the target,
\[
  \bmu_\phi(\cX)\subset[-B,B]^{d},
  \qquad
  \bm\tau_\phi(\cX)\subset[\tau_{\min},\tau_{\max}]^{d}
  \quad\text{for every }\phi\in\mathcal W,
\]
with $B,\tau_{\min},\tau_{\max}$ as in Assumption~\ref{ass:target};
\item[(ii)] the class $\mathcal G_1$ contains the first-stage optimum,
\[
  (\bmu_Q,\bm\tau_Q)\in\mathcal G_1;
\]
\item[(iii)] the marginals are Lipschitz in the weights, that is, for
some $L_1>0$ and all $\phi,\phi'\in\mathcal W$,
\[
  \max\Big\{\sup_{\bx}\big\lVert\bmu_\phi(\bx)-\bmu_{\phi'}(\bx)\big\rVert_\infty,\
  \sup_{\bx}\big\lVert\bm\tau_\phi(\bx)-\bm\tau_{\phi'}(\bx)\big\rVert_\infty\Big\}
  \ \le\ L_1\lVert\phi-\phi'\rVert_\infty ;
\]
\item[(iv)] the spectrum is Lipschitz in the weights, that is, for
some $L_2>0$ and all $\phi,\phi'\in\mathcal W$,
\[
  \sup_{\bx}\big\lVert\bm s_\phi(\bx)-\bm s_{\phi'}(\bx)\big\rVert_\infty
  \ \le\ L_2\lVert\phi-\phi'\rVert_\infty ;
\]
\item[(v)] the class contains a second-stage optimum, that is, at the
bandwidth $\tilde\gamma^{2}=4d$ a measurable selection
\[
  \bm s^{\star}(\bx)\;\in\;\operatorname*{arg\,min}_{\bm s\in\Theta}
  \tfrac12\mathrm{MMD}^{2}_{\tilde\gamma}
  \big(\cN(\bm0,R_Q(\bx)),\,\cN(\bm0,R_{\bm s})\big),
  \qquad \bx\in\cX ,
\]
lies in $\mathcal G_2$.
\end{enumerate}
\end{assumption}
For a class $\mathcal G$ of maps on $\cX$ and $\eta>0$, we call a finite set
$\mathcal C\subset\mathcal G$ an $\eta$\emph{-net} if for every
$g\in\mathcal G$ there is $g'\in\mathcal C$ with
$\sup_{\bx}\lVert g(\bx)-g'(\bx)\rVert_\infty\le\eta$, where
$\lVert\cdot\rVert_\infty$ is the maximum over all output coordinates. The
\emph{covering number} $\mathfrak N(\mathcal G,\eta)$ is defined as the
smallest cardinality of such a net. We abbreviate
$\mathfrak N_j(\eta):=\mathfrak N(\mathcal G_j,\eta)$ for stages $j=1,2$.
\begin{lemma}[Covering numbers of the two classes]\label{lem:covering}
Under Assumption~\ref{ass:class}, for every $\eta>0$,
\[
  \log\mathfrak N_1(\eta)\le q\,\log\Big(1+\frac{2WL_1}{\eta}\Big),
  \qquad
  \log\mathfrak N_2(\eta)\le q\,\log\Big(1+\frac{2WL_2}{\eta}\Big).
\]
\end{lemma}
\begin{proof}
The cube $\mathcal W=[-W,W]^{q}$ admits an $\varepsilon$-net in $\lVert\cdot\rVert_\infty$ of cardinality at most $(1+2W/\varepsilon)^{q}$ with centres in $\mathcal W$. Indeed, the grid $\{-W+\varepsilon,-W+3\varepsilon,\dots\}$ of spacing $2\varepsilon$ has at most $\lceil W/\varepsilon\rceil\le1+2W/\varepsilon$ points per coordinate, and every point of $[-W,W]$ lies within $\varepsilon$ of one of them. If the last grid point exceeds $W$, we replace it by $W$, which preserves this property. Take $\varepsilon=\eta/L_j$ and map the grid through $\phi\mapsto g_\phi$. By the Lipschitz property of Assumption~\ref{ass:class}\,(iii)--(iv), the images form an $\eta$-net of $\mathcal G_j$, whose cardinality gives the stated bound on $\mathfrak N_j(\eta)$.
\end{proof}
\paragraph{Stage 1.} For parameter values $(\bmu,\bm\tau)\in[-B,B]^{d}\times[\tau_{\min},\tau_{\max}]^{d}$, the range of Assumptions~\ref{ass:target} and~\ref{ass:class}\,(i), and an observation $\by\in\R^{d}$ let
\[
  \ell_1(\bmu,\bm\tau;\by):=\tfrac1d\sum_{i=1}^{d}
  \mathrm{CRPS}\big(\cN(\mu_i,\tau_i),\,y_i\big)
\]
be the stage-1 loss. For $g\in\mathcal G_1$ we evaluate it at $g(\bx)=(\bmu(\bx),\bm\tau(\bx))$ and write $\ell_1(g(\bx);\by)$. Risk and empirical risk over $\mathcal T_N$ are
\[
  \mathcal R_1(g):=\E_{(\bX,\bY)}\,\ell_1\big(g(\bX);\bY\big),
  \qquad
  \mathcal R_{1,N}(g):=\frac1N\sum_{n=1}^{N}\ell_1\big(g(\bX_n);\bY_n\big),
\]
and, under the idealization of exact empirical risk minimization, stage 1 returns a minimizer
\[
  (\bmu_N,\bm\tau_N)\in\operatorname*{arg\,min}_{g\in\mathcal G_1}\mathcal R_{1,N}(g),
\]
which exists because $\mathcal G_1$ is compact, being the image of the compact $\mathcal W$ under the map $\phi\mapsto(\bmu_\phi,\bm\tau_\phi)$, which is continuous by Assumption~\ref{ass:class}\,(iii).
Over all measurable maps the population minimizer of $\mathcal R_1$ is $(\bmu_Q,\bm\tau_Q)$, by Proposition~\ref{prop:two-stage-opt}\,(i) applied at each input, and by Assumption~\ref{ass:class}\,(ii) it lies in $\mathcal G_1$ and is therefore the minimizer over the class as well.
At an input $\bx$ we measure the distance of the fit $(\bmu_N,\bm\tau_N)$ in the original coordinates by
\begin{equation}\label{eq:pac-stage1-orig}
  \bar e(\bx)^{2}:=\tfrac1d\Big(
    \big\lVert\bmu_N(\bx)-\bmu_Q(\bx)\big\rVert_2^{2}
    +\big\lVert\sqrt{\bm\tau_N(\bx)}-\sqrt{\bm\tau_Q(\bx)}\big\rVert_2^{2}\Big).
\end{equation}
In the standardized coordinates, with $\bD_N(\bx):=\diag\big(\sqrt{\bm\tau_N(\bx)}\big)$ and $\bm\lambda_N(\bx):=\sqrt{\bm\tau_Q(\bx)/\bm\tau_N(\bx)}$, we use the standardized mean error and the relative scale error
\begin{equation}\label{eq:pac-stage1}
  e_\mu(\bx)^{2}:=\tfrac1d
    \big\lVert\bD_N(\bx)^{-1}\big(\bmu_N(\bx)-\bmu_Q(\bx)\big)\big\rVert_2^{2},
  \qquad
  e_\tau(\bx)^{2}:=\tfrac1d\big\lVert\bm\lambda_N(\bx)-\bm1\big\rVert_2^{2} .
\end{equation}
Both standardize by the \emph{fitted} scales $\bD_N$, since $\bD_N$ is the standardization the second stage actually applies. Their population versions drop the argument and average over $\bX\sim\bP_{\bX}$ with the fitted maps held fixed,
\[
  \bar e_N^{2}:=\E_{\bX}\big[\bar e(\bX)^{2}\big],
  \qquad
  e_{\mu,N}^{2}:=\E_{\bX}\big[e_\mu(\bX)^{2}\big],
  \qquad
  e_{\tau,N}^{2}:=\E_{\bX}\big[e_\tau(\bX)^{2}\big].
\]
They are functions of the fitted maps $\bmu_N$, $\bm\tau_N$, hence random through $\mathcal T_N$. Conditionally on $\mathcal T_N$, as in the second stage, they are fixed numbers.

\paragraph{Stage 2.} Stage 2 freezes $(\bmu_N,\bm\tau_N)$ and works in standardized coordinates. At input $\bx$ and observation $\by$ the standardized residual is $\tilde\by:=\bD_N(\bx)^{-1}\big(\by-\bmu_N(\bx)\big)$, and $\tilde\bY:=\bD_N(\bX)^{-1}\big(\bY-\bmu_N(\bX)\big)$ is the standardized target. Its conditional law given $\bX=\bx$ is
\begin{equation}\label{eq:pac-tilde}
  \tilde Q(\bx)=\cN\big(\tilde\bmu(\bx),\ \bm\Lambda(\bx)R_Q(\bx)\bm\Lambda(\bx)\big),
  \qquad
  \tilde\bmu:=\bD_N^{-1}(\bmu_Q-\bmu_N),\quad
  \bm\Lambda:=\bD_N^{-1}\bD_Q=\diag(\bm\lambda_N),
\end{equation}
so $e_\mu^{2}=\tfrac1d\lVert\tilde\bmu\rVert_2^{2}$ and $e_\tau^{2}=\tfrac1d\lVert\bm\lambda_N-\bm1\rVert_2^{2}$ measure its deviation from the law under an exact first stage, $\tilde Q^{\star}(\bx):=\cN\big(\bm0,R_Q(\bx)\big)$.

The second stage fits $\cN(\bm0,R_{\bm s})$ to $\tilde\bY$ with the Gaussian
kernel score at the bandwidth $\tilde\gamma^{2}=4d$ fixed above, so its loss is
\[
  \ell_2(\bm s;\tilde\by)
  :=S_{k_{\tilde\gamma}}\big(\cN(\bm0,R_{\bm s}),\tilde\by\big),
\]
available in closed form by \eqref{eq:ks-decoupled} with $\bm m=\bm0$. Risk and
empirical risk of the second stage are
\[
  \mathcal R_2(\bm s):=\E_{(\bX,\bY)}\,\ell_2\big(\bm s(\bX);\tilde\bY\big),
  \qquad
  \mathcal R_{2,M}(\bm s):=\frac1M\sum_{n=N+1}^{N+M}
    \ell_2\big(\bm s(\bX_n);\tilde\bY_n\big),
\]
both conditional on the first $N$ draws.

Under the same idealization as in stage 1, stage~2 returns a minimizer $\bm s_M\in\operatorname*{arg\,min}_{\bm s\in\mathcal G_2}\mathcal R_{2,M}(\bm s)$, which exists because $\mathcal G_2$ is compact, being the image of the compact $\mathcal W$ under the map $\phi\mapsto\bm s_\phi$, which is continuous by Assumption~\ref{ass:class}\,(iv). The analysis compares it against the pointwise best core $\bm s^{\star}$ of Assumption~\ref{ass:class}\,(v).

Note that $R_Q$ itself will in general not lie in $\{R_{\bm s}:\bm s\in\Theta\}$. Write
\begin{equation}\label{eq:pac-model-error}
  \bar\varepsilon_{\mathrm{app}}
  :=\E_{\bX}\Big[\min_{\bm s\in\Theta}
    \tfrac12\mathrm{MMD}^{2}_{\tilde\gamma}
    \big(\tilde Q^{\star}(\bX),\,\cN(\bm0,R_{\bm s})\big)\Big]
\end{equation}
for the $\bP_{\bX}$-average of the approximation error $\varepsilon_{\mathrm{app}}(\bX)$ of Section~\ref{sec:methodology} at bandwidth $\tilde\gamma$. It coincides with the average of the approximation error of Appendix~\ref{app:expressivity}, where the infimum over strictly positive spectra and the minimum over their closure $\Theta$ agree by continuity.
Writing $R_M(\bx):=R_{\bm s_M(\bx)}$, the deployed prediction is the affine image of $\cN(\bm0,R_M(\bx))$ under $\tilde\by\mapsto\bmu_N(\bx)+\bD_N(\bx)\tilde\by$,
\[
  P_{N,M}(\bx):=\cN\big(\bmu_N(\bx),\ \bD_N(\bx)\,R_M(\bx)\,\bD_N(\bx)\big).
\]
\subsection{The first stage is PAC}
Concentration bounds the excess risk of the empirical minimizer, and the following lemma turns such a bound into one on the parameter error. Together they make the first stage PAC. Throughout, a centred random variable $Z$ is \emph{sub-Gaussian with variance proxy} $\sigma^{2}$ if $\log\E\,e^{\beta Z}\le\beta^{2}\sigma^{2}/2$ for all $\beta\in\R$. Averages of $n$ independent such variables are sub-Gaussian with proxy $\sigma^{2}/n$, and the Chernoff bound gives $\bP(\lvert Z\rvert>t)\le2e^{-t^{2}/(2\sigma^{2})}$ \citep[Sec.~2.3]{boucheron2013concentration}.
\begin{lemma}[Curvature of the CRPS]\label{lem:crps-curvature}
Denote the standard normal PDF and CDF by $\varphi$ and $\Phi$. Let $m\in\R$,
$\nu>0$, and $Y\sim\cN(m,\nu^{2})$, and for $\mu\in\R$, $\sigma>0$ write
\[
  \mathrm{div}(\mu,\sigma):=\E_{Y}\big[\mathrm{CRPS}\big(\cN(\mu,\sigma^{2}),Y\big)
  -\mathrm{CRPS}\big(\cN(m,\nu^{2}),Y\big)\big].
\]
If $\lvert\mu-m\rvert\le2B$ and
$\sigma,\nu\in[\sqrt{\tau_{\min}},\sqrt{\tau_{\max}}]$, then
\[
  \mathrm{div}(\mu,\sigma)\ \ge\ c_\star\big((\mu-m)^{2}+(\sigma-\nu)^{2}\big),
  \qquad
  c_\star:=\min\Big\{\tfrac{1}{\sqrt{2\tau_{\max}}}\,
  \varphi\big(\tfrac{2B}{\sqrt{2\tau_{\min}}}\big),\
  \tfrac{1}{2}\sqrt{\tfrac2\pi}\,\tfrac{\tau_{\min}}{(2\tau_{\max})^{3/2}}\Big\}.
\]
\end{lemma}
\begin{proof}
Define
\[
  g:\R\times(0,\infty)\longrightarrow(0,\infty),
  \qquad
  g(u,w):=\E\lvert Z\rvert\ \text{ for }Z\sim\cN(u,w^{2}),
\]
so that $g(u,w)=w\sqrt{2/\pi}\,e^{-u^{2}/(2w^{2})}+u\,\big(2\Phi(u/w)-1\big)$,
and abbreviate $a:=\mu-m$ and $t:=\sqrt{\sigma^{2}+\nu^{2}}$.
Let $X,X'\sim\cN(\mu,\sigma^{2})$ be independent of $Y$ and of each other. The
kernel form of the CRPS gives
\[
  \E_{Y}\,\mathrm{CRPS}\big(\cN(\mu,\sigma^{2}),Y\big)
  =\E\lvert X-Y\rvert-\tfrac12\E\lvert X-X'\rvert
  =g(a,t)-\tfrac{\sigma}{\sqrt\pi},
\]
since $X-Y\sim\cN(a,t^{2})$ and $X-X'\sim\cN(0,2\sigma^{2})$. At
$(\mu,\sigma)=(m,\nu)$ the right-hand side equals $\nu/\sqrt\pi$, and
$g(0,t)=t\sqrt{2/\pi}$, so
\[
  \mathrm{div}(\mu,\sigma)=\underbrace{g(a,t)-g(0,t)}_{\text{mean term}}
  +\underbrace{h(\sigma)}_{\text{scale term}},
  \qquad
  h(\sigma):=\sqrt{\tfrac2\pi}\,\sqrt{\sigma^{2}+\nu^{2}}
  -\tfrac{\sigma}{\sqrt\pi}-\tfrac{\nu}{\sqrt\pi}.
\]
\emph{The mean term.} Differentiating $g$ in its first argument, the two
exponential contributions cancel and $\partial_ug(u,t)=2\Phi(u/t)-1$, so
$\partial_ug(0,t)=0$ and $\partial_u^{2}g(u,t)=\tfrac2t\varphi(u/t)$. On the
stated ranges $t\in[\sqrt{2\tau_{\min}},\sqrt{2\tau_{\max}}]$ and
$\lvert u\rvert\le\lvert a\rvert\le2B$, whence
$\partial_u^{2}g(u,t)\ge\tfrac{2}{\sqrt{2\tau_{\max}}}\,
\varphi\big(\tfrac{2B}{\sqrt{2\tau_{\min}}}\big)$ on the segment from $0$ to
$a$, and second-order Taylor expansion at $u=0$ gives
\[
  g(a,t)-g(0,t)\ \ge\ \tfrac{1}{\sqrt{2\tau_{\max}}}\,
  \varphi\big(\tfrac{2B}{\sqrt{2\tau_{\min}}}\big)\,a^{2}.
\]
\emph{The scale term.} One checks $h(\nu)=h'(\nu)=0$ and
\[
  h''(\sigma)=\sqrt{\tfrac2\pi}\,\nu^{2}\big(\sigma^{2}+\nu^{2}\big)^{-3/2}
  \ \ge\ \sqrt{\tfrac2\pi}\,\tau_{\min}\,(2\tau_{\max})^{-3/2},
\]
so second-order Taylor expansion at $\sigma=\nu$ gives
\[
  h(\sigma)\ \ge\ \tfrac12\sqrt{\tfrac2\pi}\,
  \tau_{\min}\,(2\tau_{\max})^{-3/2}\,(\sigma-\nu)^{2}.
\]
Taking $c_\star$ to be the smaller of the two constants and adding the two
bounds proves the claim.
\end{proof}
With the curvature in hand, a uniform deviation bound over $\mathcal G_1$ gives the following.
\begin{lemma}[The first stage is PAC]\label{lem:pac-stage1-rate}
Let $(\bmu_N,\bm\tau_N)\in\operatorname*{arg\,min}_{g\in\mathcal G_1}
\mathcal R_{1,N}(g)$ be the stage-1 empirical risk minimizer, let $\eta_1>0$
and $\delta_1\in(0,1)$, let $c_\star$ be the curvature constant of
Lemma~\ref{lem:crps-curvature} and $\mathfrak N_1(\eta_1)$ the covering
number of $\mathcal G_1$, bounded in Lemma~\ref{lem:covering}, and put
\[
  L_1':=1+\frac{1}{2\sqrt{\pi\tau_{\min}}},
  \qquad
  \sigma_1^{2}:=\tau_{\max}+\Big(B+\sqrt{\tau_{\max}/\pi}\Big)^{2},
  \qquad
  \rho_N(\eta_1):=2L_1'\eta_1
  +\sigma_1\sqrt{\frac{2\log\big(2\mathfrak N_1(\eta_1)/\delta_1\big)}{N}}.
\]
Under Assumptions~\ref{ass:target} and~\ref{ass:class}\,(i)--(iii), with
probability at least $1-\delta_1$ over $\mathcal T_N$, the first-stage error
is bounded both in the \hyperref[eq:pac-stage1-orig]{original} and in the
\hyperref[eq:pac-stage1]{standardized} coordinates,
\[
  \bar e_N^{2}\;\le\;\frac{2\rho_N(\eta_1)}{c_\star}
  \qquad\text{and}\qquad
  e_{\mu,N}^{2}+e_{\tau,N}^{2}
  \;\le\;\frac{2\rho_N(\eta_1)}{c_\star\,\tau_{\min}} .
\]
\end{lemma}
\begin{proof}
We proceed as follows. A uniform deviation bound over
$\mathcal G_1$ (Steps 1--2) controls the excess risk of the empirical
minimizer, and the curvature of the CRPS (Step 3) converts this excess risk into the
parameter error.
\emph{Step 1: for every fixed $g \in \mathcal{G}_1$, the centred loss
$\ell_1\big(g(\bX);\bY\big)-\mathcal R_1(g)$ is sub-Gaussian in $(\bX,\bY)$
with proxy $\sigma_1^{2}$.}
Fix $g=(\bmu,\bm\tau)\in\mathcal G_1$. Recall that given $\bX=\bx$ the target
is Gaussian, $\bY\sim Q(\bx)=\cN\big(\bmu_Q(\bx),\bSigma_Q(\bx)\big)$, and
write
\[
   m_g(\bx):=\E_{\bY\sim Q(\bx)}\big[\ell_1\big(g(\bx);\bY\big)\big]
\]
for the conditional risk. We split the fluctuation of
$\ell_1\big(g(\bX);\bY\big)$ around $\mathcal R_1(g)$ into the fluctuation of
$\bY$ given $\bX$, i.e.\ around $m_g(\bX)$, and the fluctuation of
$m_g(\bX)$ across inputs $\bX$. Indeed, adding and subtracting $m_g(\bX)$,
\[
  \ell_1\big(g(\bX);\bY\big)-\mathcal R_1(g)
  =\underbrace{\Big[\ell_1\big(g(\bX);\bY\big)-m_g(\bX)\Big]}
   _{\text{centred given }\bX}
  \;+\;
  \underbrace{\Big[m_g(\bX)-\mathcal R_1(g)\Big]}
   _{\bX\text{-measurable, centred}}.
\]
For the first part, each summand of $\ell_1$ is $1$-Lipschitz in its
observation, so $\by\mapsto\ell_1(\bmu,\bm\tau;\by)$ is $d^{-1/2}$-Lipschitz in
the Euclidean norm, uniformly in the parameters. Given $\bX=\bx$, write
$\bY=\bmu_Q(\bx)+\bSigma_Q(\bx)^{1/2}\bm\xi$ with $\bm\xi\sim\cN(\bm0,\bI)$.
The composition
\[
  \bm\xi\;\longmapsto\;\ell_1\big(g(\bx);\,\bmu_Q(\bx)+\bSigma_Q(\bx)^{1/2}\bm\xi\big)
\]
is then, as a composition of Lipschitz maps, Lipschitz in $\bm \xi$ with constant
$d^{-1/2}\,\lVert\bSigma_Q(\bx)^{1/2}\rVert_{\mathrm{op}}
\le d^{-1/2}\big(\tau_{\max}d\big)^{1/2}=\sqrt{\tau_{\max}}$, since
$\lVert\bSigma_Q^{1/2}\rVert_{\mathrm{op}}=\lVert\bSigma_Q\rVert_{\mathrm{op}}^{1/2}$
and $\lVert\bSigma_Q\rVert_{\mathrm{op}}\le\tr\bSigma_Q\le\tau_{\max}d$ by
Assumption~\ref{ass:target}. Gaussian concentration for Lipschitz functions
\citep[Thm.~5.5]{boucheron2013concentration} applies to this composition, whose
mean is $m_g(\bx)$ by the definition of the conditional risk, and gives for
every $\bx$ and every $\beta\in\R$
\[
  \log\E_{\bY\sim Q(\bx)}
    \Big[e^{\beta(\ell_1(g(\bx);\bY)-m_g(\bx))}\Big]
  \;\le\;\frac{\beta^{2}\tau_{\max}}{2}.
\]
The conditional fluctuation is therefore sub-Gaussian with variance proxy
$\tau_{\max}$, uniformly in $\bx$, the Lipschitz constant above not depending
on $\bx$.
For the second part, we show that $0\le m_g(\bx)\le2B+2\sqrt{\tau_{\max}/\pi}$
for every $\bx$. Since $\ell_1$ is a coordinate average and the CRPS sees only
the marginals,
\[
  m_g(\bx)=\tfrac1d\sum_{i=1}^{d}
  \E_{Y_i}\,\mathrm{CRPS}\big(\cN(\mu_i(\bx),\tau_i(\bx)),Y_i\big),
  \qquad
  Y_i\sim\cN\big(\mu_{Q,i}(\bx),\tau_{Q,i}(\bx)\big).
\]
Nonnegativity is that of the CRPS. For the upper bound, recall from the proof
of Lemma~\ref{lem:crps-curvature} that
$\E_{Y_i}\mathrm{CRPS}\big(\cN(\mu_i,\tau_i),Y_i\big)=g(a_i,t_i)-\sqrt{\tau_i/\pi}$
with $a_i:=\mu_i-\mu_{Q,i}$, $t_i^{2}:=\tau_i+\tau_{Q,i}$ and
$g(a,t)=\E\lvert\cN(a,t^{2})\rvert\le\lvert a\rvert+t\sqrt{2/\pi}$. The ranges
of Assumptions~\ref{ass:target} and~\ref{ass:class}\,(i) give
$\lvert a_i\rvert\le2B$ and $t_i\le\sqrt{2\tau_{\max}}$, hence
\[
  \E_{Y_i}\mathrm{CRPS}\big(\cN(\mu_i,\tau_i),Y_i\big)
  \;\le\;2B+\sqrt{2\tau_{\max}}\sqrt{2/\pi}
  \;=\;2B+2\sqrt{\tau_{\max}/\pi},
\]
and averaging over the $d$ coordinates bounds $m_g$.
By Hoeffding's lemma \citep[Ch.~2]{boucheron2013concentration} a random
variable with range $\ell$ is sub-Gaussian with variance proxy $(\ell/2)^{2}$,
so $m_g(\bX)-\mathcal R_1(g)$, whose range has length
$2B+2\sqrt{\tau_{\max}/\pi}$, has proxy
$\big(B+\sqrt{\tau_{\max}/\pi}\big)^{2}$.
The two proxies add. Indeed, the tower rule applies the conditional bound
inside and the two log-moment generating functions sum. That is, for every
$\beta\in\R$,
\begin{align*}
  \log\E_{(\bX,\bY)}\Big[e^{\beta(\ell_1(g(\bX);\bY)-\mathcal R_1(g))}\Big]
  &\;\le\;\frac{\beta^{2}\tau_{\max}}{2}
   +\log\E_{\bX}\Big[e^{\beta(m_g(\bX)-\mathcal R_1(g))}\Big]\\
  &\;\le\;\frac{\beta^{2}}{2}\Big(\tau_{\max}
   +\big(B+\sqrt{\tau_{\max}/\pi}\big)^{2}\Big)
   \;=\;\frac{\beta^{2}\sigma_1^{2}}{2}.
\end{align*}
Hence $\ell_1\big(g(\bX);\bY\big)-\mathcal R_1(g)$ is sub-Gaussian in
$(\bX,\bY)$ with proxy $\sigma_1^{2}$. Note that this statement holds uniformly in $g \in \mathcal{G}_1$.
\emph{Step 2: uniform deviation, and the excess risk of the minimizer.} Recall
from \eqref{eq:crps-closed}, written here with $\sigma=\sqrt\tau$, that with
$z:=(y-\mu)/\sigma$ the CRPS of a Gaussian has the closed form
\[
  \mathrm{CRPS}\big(\cN(\mu,\sigma^{2}),y\big)=\sigma f(z),
  \qquad
  f(z):=z\big(2\Phi(z)-1\big)+2\varphi(z)-\tfrac{1}{\sqrt\pi},
\]
and $\varphi'(z)=-z\varphi(z)$ gives $f'(z)=2\Phi(z)-1$. Hence
\[
  \partial_\mu\mathrm{CRPS}
  =-f'(z)=1-2\Phi(z),
  \qquad
  \partial_\sigma\mathrm{CRPS}
  =f(z)-zf'(z)=2\varphi(z)-\tfrac{1}{\sqrt\pi},
\]
so $\lvert\partial_\mu\mathrm{CRPS}\rvert\le1$ and, since
$0<2\varphi(z)\le\sqrt{2/\pi}<2/\sqrt\pi$,
$\lvert\partial_\sigma\mathrm{CRPS}\rvert\le1/\sqrt\pi$; by
$\partial_\tau=\partial_\sigma/(2\sigma)$ with $\sigma\ge\sqrt{\tau_{\min}}$ on
the range of Assumption~\ref{ass:class}\,(i),
$\lvert\partial_\tau\mathrm{CRPS}\rvert\le(2\sqrt{\pi\tau_{\min}})^{-1}$.
The admissible range is a product of intervals, so the segment between any two
admissible parameter values stays inside it, and the mean value theorem gives,
coordinate by coordinate and hence for the average $\ell_1$,
\[
  \big\lvert\ell_1(\bmu,\bm\tau;\by)-\ell_1(\bmu',\bm\tau';\by)\big\rvert
  \;\le\;
  \max_i\lvert\mu_i-\mu_i'\rvert
  +\frac{\max_i\lvert\tau_i-\tau_i'\rvert}{2\sqrt{\pi\tau_{\min}}}
  \;\le\;L_1'\,\big\lVert(\bmu,\bm\tau)-(\bmu',\bm\tau')\big\rVert_\infty ,
\]
uniformly in $\by$. Evaluated at $g(\bx)$ and $g'(\bx)$ this gives
$\lvert\ell_1(g(\bx);\by)-\ell_1(g'(\bx);\by)\rvert
\le L_1'\,\sup_{\bx}\lVert g(\bx)-g'(\bx)\rVert_\infty$ for all
$g,g'\in\mathcal G_1$ and all $(\bx,\by)$.
Let $\mathcal C$ be an $\eta_1$-net of $\mathcal G_1$ of cardinality
$\mathfrak N_1(\eta_1)$, and for $g\in\mathcal G_1$ pick $h\in\mathcal C$ with
$\sup_{\bx}\lVert g(\bx)-h(\bx)\rVert_\infty\le\eta_1$. Both $\mathcal R_{1,N}$
and $\mathcal R_1$ are averages of $\ell_1$ and therefore inherit its Lipschitz
bound, so each of the outer terms in
\[
  \big\lvert\mathcal R_{1,N}(g)-\mathcal R_1(g)\big\rvert
  \;\le\;
  \big\lvert\mathcal R_{1,N}(g)-\mathcal R_{1,N}(h)\big\rvert
  +\big\lvert\mathcal R_{1,N}(h)-\mathcal R_1(h)\big\rvert
  +\big\lvert\mathcal R_1(h)-\mathcal R_1(g)\big\rvert
\]
is at most $L_1'\eta_1$. Taking the supremum over $g$,
\[
  \sup_{g\in\mathcal G_1}
  \big\lvert\mathcal R_{1,N}(g)-\mathcal R_1(g)\big\rvert
  \;\le\;2L_1'\eta_1
  +\max_{h\in\mathcal C}
  \big\lvert\mathcal R_{1,N}(h)-\mathcal R_1(h)\big\rvert .
\]
Fix $h\in\mathcal C$. By Step 1, $\mathcal R_{1,N}(h)-\mathcal R_1(h)$ is an
average of $N$ i.i.d.\ centred terms with variance proxy $\sigma_1^{2}$, hence
sub-Gaussian with proxy $\sigma_1^{2}/N$. A union bound over $\mathcal C$ then
gives, with probability at least $1-\delta_1$,
\[
  \max_{h\in\mathcal C}
  \big\lvert\mathcal R_{1,N}(h)-\mathcal R_1(h)\big\rvert
  \;\le\;\sigma_1\sqrt{\frac{2\log\big(2\mathfrak N_1(\eta_1)/\delta_1\big)}{N}} .
\]
On that event the net term $2L_1'\eta_1$ and this deviation add to exactly
$\rho_N(\eta_1)$, that is,
$\sup_{\mathcal G_1}\lvert\mathcal R_{1,N}-\mathcal R_1\rvert\le\rho_N(\eta_1)$.
On the same event the excess risk of the empirical minimizer is controlled.
Indeed, $(\bmu_Q,\bm\tau_Q)\in\mathcal G_1$ by
Assumption~\ref{ass:class}\,(ii), so
$\mathcal R_{1,N}(\bmu_N,\bm\tau_N)\le\mathcal R_{1,N}(\bmu_Q,\bm\tau_Q)$ by
the definition of the empirical risk minimizer, and therefore
\begin{align*}
  \mathcal R_1(\bmu_N,\bm\tau_N)-\mathcal R_1(\bmu_Q,\bm\tau_Q)
  &=\big[\mathcal R_1(\bmu_N,\bm\tau_N)-\mathcal R_{1,N}(\bmu_N,\bm\tau_N)\big]
   +\big[\mathcal R_{1,N}(\bmu_N,\bm\tau_N)-\mathcal R_{1,N}(\bmu_Q,\bm\tau_Q)\big]\\
  &\quad+\big[\mathcal R_{1,N}(\bmu_Q,\bm\tau_Q)-\mathcal R_1(\bmu_Q,\bm\tau_Q)\big]
  \;\le\;\rho_N(\eta_1)+0+\rho_N(\eta_1)\;=\;2\rho_N(\eta_1).
\end{align*}
This is the only place where the ERM property and
Assumption~\ref{ass:class}\,(ii) enter.
\emph{Step 3: curvature turns excess risk into parameter error.} Since $\ell_1$
is a coordinate average, conditioning on
$\bX$ decomposes,
\[
  \mathcal R_1(\bmu_N,\bm\tau_N)-\mathcal R_1(\bmu_Q,\bm\tau_Q)
  \;=\;\E_{\bX}\Big[\tfrac1d\sum_{i=1}^{d}
    \mathrm{div}\big(\mu_{N,i}(\bX),\sqrt{\tau_{N,i}(\bX)}\big)\Big],
\]
with $\mathrm{div}$ as in Lemma~\ref{lem:crps-curvature}, taken against the
$i$-th marginal of the target, that is at
$(m,\nu)=\big(\mu_{Q,i}(\bX),\sqrt{\tau_{Q,i}(\bX)}\big)$. Its hypotheses hold
by Assumptions~\ref{ass:target} and~\ref{ass:class}\,(i), so each summand is at
least
$c_\star\big[(\mu_{N,i}-\mu_{Q,i})^{2}
+(\sqrt{\tau_{N,i}}-\sqrt{\tau_{Q,i}})^{2}\big]$, whose coordinate average is
$c_\star\,\bar e(\bX)^{2}$ by \eqref{eq:pac-stage1-orig}. Taking expectations
and inserting the excess-risk bound of Step 2 gives
$c_\star\bar e_N^{2}\le2\rho_N(\eta_1)$, the first claim.
The second follows by dividing each summand by $\tau_{N,i}\ge\tau_{\min}$. Indeed, it holds
\[
  \bar e(\bx)^{2}
  \;\ge\;\tau_{\min}\cdot\tfrac1d\sum_{i=1}^{d}
  \frac{(\mu_{N,i}-\mu_{Q,i})^{2}
        +\big(\sqrt{\tau_{N,i}}-\sqrt{\tau_{Q,i}}\big)^{2}}{\tau_{N,i}}
  \;=\;\tau_{\min}\big(e_\mu(\bx)^{2}+e_\tau(\bx)^{2}\big),
\]
the last equality being the definitions \eqref{eq:pac-stage1} together with
$\big(\sqrt{\tau_{N,i}}-\sqrt{\tau_{Q,i}}\big)^{2}/\tau_{N,i}
=(1-\lambda_{N,i})^{2}$. Taking expectations over $\bX$ and inserting the first
claim completes the proof.
\end{proof}
With Lemma~\ref{lem:pac-stage1-rate} we have bounds on the mean and scale
errors in terms of the sample size $N$, which is what we need to trace the
first-stage error through the second stage.
\subsection{The second stage in standardized coordinates is PAC}
The bound we are after concerns the deployed prediction $P_{N,M}$ in the original coordinates. The second stage, however, is fitted in the standardized ones, so we address the problem there first and return in the next subsection. Recall from \eqref{eq:pac-tilde} that the standardized target has conditional law $\tilde Q(\bx)=\cN\big(\tilde\bmu(\bx),\bm\Lambda(\bx)R_Q(\bx)\bm\Lambda(\bx)\big)$, which equals its counterpart $\tilde Q^{\star}(\bx)=\cN\big(\bm0,R_Q(\bx)\big)$ under an exact first stage exactly when $\tilde\bmu(\bx)=\bm0$ and $\bm\Lambda(\bx)=\bI$, and that the second stage uses the Gaussian kernel at the bandwidth $\tilde\gamma^{2}=4d$.
\begin{lemma}[The implied first-stage error]\label{lem:pac-stage1}
Fix the first stage by conditioning on $\mathcal T_N$, so that
$(\bmu_N,\bm\tau_N)$ and with them $\tilde Q$, $e_\mu$ and $e_\tau$ are
determined. Then for every input $\bx\in\cX$ of the second stage and every
spectrum $\bm s\in\Theta$, replacing the standardized target by its ideal
counterpart costs at most
\[
  \Big\lvert
  \tfrac12\mathrm{MMD}^{2}_{\tilde\gamma}\big(\tilde Q(\bx),\,\cN(\bm0,R_{\bm s})\big)
  -\tfrac12\mathrm{MMD}^{2}_{\tilde\gamma}\big(\tilde Q^{\star}(\bx),\,\cN(\bm0,R_{\bm s})\big)
  \Big\rvert
  \;\le\;\sqrt{\tfrac{2}{e}}\,\big(e_\mu(\bx)+e_\tau(\bx)\big).
\]
The bound does not depend on $\bm s$, so it holds in particular at the fitted
$\bm s_M$. Now condition on $\mathcal T_M$ as well, that is, fix the map
$\bm s_M\in\mathcal G_2$. Averaging over (a fresh) $\bX\sim\bP_{\bX}$ then gives
\[
  \E_{\bX}\Big\lvert
  \tfrac12\mathrm{MMD}^{2}_{\tilde\gamma}
    \big(\tilde Q(\bX),\,\cN(\bm0,R_{\bm s_M(\bX)})\big)
  -\tfrac12\mathrm{MMD}^{2}_{\tilde\gamma}
    \big(\tilde Q^{\star}(\bX),\,\cN(\bm0,R_{\bm s_M(\bX)})\big)
  \Big\rvert
  \;\le\;\Delta_N:=\sqrt{\tfrac{2}{e}}\,\big(e_{\mu,N}+e_{\tau,N}\big).
\]
\end{lemma}
\begin{proof}
For the pointwise bound, fix $\bx$ and drop it from the notation. Couple the two targets by letting $\bm U_0\sim\tilde Q^{\star}=\cN(\bm0,R_Q)$ and setting
$\bm U:=\tilde\bmu+\bm\Lambda\bm U_0$, which has law $\tilde Q$ by
\eqref{eq:pac-tilde}. Let $(\bm U_0',\bm U')$ be an independent copy of
$(\bm U_0,\bm U)$, and $\bm V,\bm V'\sim\cN(\bm0,R_{\bm s})$ independent of both
and of each other.
Now take the squared MMDs and subtract them. That is, writing $k:=k_{\tilde\gamma}$ throughout,
\begin{align*}
  \mathrm{MMD}^{2}_{\tilde\gamma}\big(\tilde Q,\cN(\bm0,R_{\bm s})\big)
  -\mathrm{MMD}^{2}_{\tilde\gamma}\big(\tilde Q^{\star},\cN(\bm0,R_{\bm s})\big)
  &=\Big(\E\,k(\bm U,\bm U')+\E\,k(\bm V,\bm V')-2\,\E\,k(\bm U,\bm V)\Big)\\
  &\quad-\Big(\E\,k(\bm U_0,\bm U_0')+\E\,k(\bm V,\bm V')
    -2\,\E\,k(\bm U_0,\bm V)\Big)\\
  &=\E\big[k(\bm U,\bm U')-k(\bm U_0,\bm U_0')\big]
   -2\,\E\big[k(\bm U,\bm V)-k(\bm U_0,\bm V)\big].
\end{align*}
The kernel is
$L_{\tilde\gamma}$-Lipschitz in each argument, with
$L_{\tilde\gamma}:=\sqrt2/(\tilde\gamma\sqrt e)$ the maximal gradient norm of
$\bm u\mapsto e^{-\lVert\bm u\rVert^{2}/\tilde\gamma^{2}}$. In the first bracket
both arguments change, one at a time,
\[
  \big\lvert k_{\tilde\gamma}(\bm U,\bm U')-k_{\tilde\gamma}(\bm U_0,\bm U_0')\big\rvert
  \;\le\;L_{\tilde\gamma}\big(\lVert\bm U-\bm U_0\rVert
  +\lVert\bm U'-\bm U_0'\rVert\big),
\]
which contributes $2L_{\tilde\gamma}\E\lVert\bm U-\bm U_0\rVert$,
the two pairs being identically distributed. In the second bracket both terms have $\bm V$ as
their second argument, so the same inequality gives
$L_{\tilde\gamma}\lVert\bm U-\bm U_0\rVert$. The
difference is therefore at most
$4L_{\tilde\gamma}\E\lVert\bm U-\bm U_0\rVert$ in absolute value.
The coupling makes that distance small. From
$\bm U-\bm U_0=\tilde\bmu+(\bm\Lambda-\bI)\bm U_0$ and $\diag(R_Q)=\bm1$,
\[
  \E\lVert\bm U-\bm U_0\rVert
  \;\le\;\lVert\tilde\bmu\rVert
  +\Big(\E\lVert(\bm\Lambda-\bI)\bm U_0\rVert^{2}\Big)^{1/2}
  =\lVert\tilde\bmu\rVert+\Big(\sum_i(\lambda_{N,i}-1)^{2}\Big)^{1/2}
  =\sqrt d\,\big(e_\mu+e_\tau\big),
\]
the two summands being exactly the mean and the scale error of
\eqref{eq:pac-stage1}. Halving, since the claim is about
$\tfrac12\mathrm{MMD}^{2}$, and using $2L_{\tilde\gamma}\sqrt d=\sqrt{2/e}$ at
$\tilde\gamma^{2}=4d$ gives the pointwise bound. Note that the dimension cancels here
because the bandwidth scales with $\sqrt d$.
For the averaged form, take $\E_{\bX}$ of the pointwise bound. The population
errors are quadratic means, $e_{\mu,N}=(\E_{\bX}\,e_\mu(\bX)^{2})^{1/2}$ and
likewise for $e_\tau$, so Jensen's inequality gives
$\E_{\bX}[e_\mu(\bX)]\le e_{\mu,N}$ and $\E_{\bX}[e_\tau(\bX)]\le e_{\tau,N}$.
\end{proof}
The first-stage error therefore enters the second stage only through
$\Delta_N$, linearly in $e_{\mu,N}$ and $e_{\tau,N}$ and independently of
$\mathcal G_2$ and of $M$. In particular, the first-stage error is not amplified by the second stage.
\begin{lemma}[Regularity of the stage-2 loss]\label{lem:pac-regularity}
For all spectra $\bm s,\bm s'\in\Theta$ and every standardized residual
$\tilde\by\in\R^{d}$,
\[
  0\;\le\;\ell_2(\bm s;\tilde\by)\;\le\;1,
  \qquad
  \big\lvert\ell_2(\bm s;\tilde\by)-\ell_2(\bm s';\tilde\by)\big\rvert
  \;\le\;L_2'\,\lVert\bm s-\bm s'\rVert_\infty,
  \qquad
  L_2':=\frac12+\frac{3}{4e},
\]
where $\lVert\cdot\rVert_\infty$ is the maximum over the $d$ coordinates. Both
bounds hold uniformly in $\tilde\by$.
\end{lemma}
\begin{proof}
By the closed form \eqref{eq:ks-decoupled} of the Gaussian kernel score,
evaluated at mean $\bm0$ and observation $\tilde\by$,
\[
  \ell_2(\bm s;\tilde\by)=\tfrac12(1+\alpha)-\beta e^{-\varrho},
  \qquad
  \alpha:=\prod_k\Big(1+\tfrac{4s_k}{\tilde\gamma^{2}}\Big)^{-1/2},
  \quad
  \beta:=\prod_k\Big(1+\tfrac{2s_k}{\tilde\gamma^{2}}\Big)^{-1/2},
  \quad
  \varrho:=\sum_k\frac{w_k}{\tilde\gamma^{2}+2s_k},
\]
with $w_k:=\lvert(\bF\tilde\by)_k\rvert^{2}$ the periodogram of the
standardized residual. Since $(1+2u)^{2}\ge1+4u$ we have $\alpha\ge\beta^{2}$,
hence $\ell_2\ge\tfrac12(1-\beta)^{2}\ge0$, and $\ell_2\le\tfrac12(1+\alpha)\le1$.
Write $\partial_k:=\partial/\partial s_k$ for the derivative in the $k$-th
spectral coordinate. Then
$\lvert\partial_k\alpha\rvert\le2/\tilde\gamma^{2}$,
$\lvert\partial_k\beta\rvert\le1/\tilde\gamma^{2}$ and
$\lvert\partial_k\varrho\rvert=2w_k/(\tilde\gamma^{2}+2s_k)^{2}$, so by the
product rule
\[
  \lvert\partial_k\ell_2\rvert
  \;\le\;\tfrac12\lvert\partial_k\alpha\rvert
  +\lvert\partial_k\beta\rvert e^{-\varrho}
  +\beta e^{-\varrho}\lvert\partial_k\varrho\rvert .
\]
On $\Theta$ we have $s_k\le d$, so $\varrho\ge w_k/(\tilde\gamma^{2}+2d)$ and
$w_ke^{-\varrho}\le\sup_{t\ge0}te^{-t/(\tilde\gamma^{2}+2d)}
=(\tilde\gamma^{2}+2d)/e$. At our bandwidth $\tilde\gamma^{2}=4d$, it therefore holds
\[
  \lvert\partial_k\ell_2\rvert
  \le\frac{2}{\tilde\gamma^{2}}+\frac{2(\tilde\gamma^{2}+2d)}{e\tilde\gamma^{4}}
  =\frac{1}{2d}+\frac{3}{4ed}=\frac{L_2'}{d}.
\]
Since $\Theta$ is convex, the segment between $\bm s$ and $\bm s'$ stays in
$\Theta$, and the mean value theorem along it gives
\[
  \big\lvert\ell_2(\bm s;\tilde\by)-\ell_2(\bm s';\tilde\by)\big\rvert
  \;\le\;\sum_{k=0}^{d-1}\sup_{\Theta}\lvert\partial_k\ell_2\rvert\,\lvert s_k-s_k'\rvert
  \;\le\;d\cdot\frac{L_2'}{d}\,\lVert\bm s-\bm s'\rVert_\infty
  \;=\;L_2'\,\lVert\bm s-\bm s'\rVert_\infty .
\]
\end{proof}

\begin{proposition}[PAC bound in standardized coordinates]\label{prop:pac-rate}
Let $\bm s_M\in\operatorname*{arg\,min}_{\bm s\in\mathcal G_2}
\mathcal R_{2,M}(\bm s)$ be the stage-2 empirical risk minimizer, let
$\eta_1,\eta_2>0$ and $\delta_1,\delta_2\in(0,1)$, and put
\[
  c_2:=2\sqrt{\frac{2}{e\,c_\star\tau_{\min}}},
  \qquad
  \rho_M(\eta_2):=4L_2'\eta_2
  +2\sqrt{\frac{\log\big(2\mathfrak N_2(\eta_2)/\delta_2\big)}{2M}},
\]
with $\rho_N(\eta_1)$ as in Lemma~\ref{lem:pac-stage1-rate}. Both $\rho_N$ and
$\rho_M$ suppress their dependence on the confidence levels, which enters
through $\log\big(2\mathfrak N_j(\eta_j)/\delta_j\big)$. Under
Assumptions~\ref{ass:target} and~\ref{ass:class}, with probability at least
$1-\delta_1-\delta_2$,
\begin{equation}\label{eq:pac-rate}
  \E_{\bX}\Big[\tfrac12\mathrm{MMD}^{2}_{\tilde\gamma}
  \big(\tilde Q^{\star}(\bX),\,\cN(\bm0,R_{\bm s_M(\bX)})\big)\Big]
  \;\le\;
  \underbrace{\bar\varepsilon_{\mathrm{app}}}_{\text{appr.}}
  +\underbrace{2c_2\sqrt{\rho_N(\eta_1)}}_{\text{first stage}}
  +\underbrace{\rho_M(\eta_2)}_{\text{second stage}} .
\end{equation}
\end{proposition}
\begin{proof}
Condition on $\mathcal T_N$, that is $\bmu_N$ and $\bD_N=\diag\big(\sqrt{\bm\tau_N}\big)$ are fixed and the $M$ pairs
$(\bX_n,\tilde\bY_n)$, $n>N$, are i.i.d.\ with
\[
  \tilde\bY_n\mid\bX_n\;\sim\;\tilde Q(\bX_n)
  =\cN\big(\tilde\bmu(\bX_n),\,
    \bm\Lambda(\bX_n)R_Q(\bX_n)\bm\Lambda(\bX_n)\big)
\]
as in \eqref{eq:pac-tilde}, with $\bm\Lambda=\bD_N^{-1}\bD_Q$ the diagonal
of scale ratios and $\tilde\bmu=\bD_N^{-1}(\bmu_Q-\bmu_N)$.
For every $\bm s\in\mathcal G_2$, the expected score gap equals the divergence
of the kernel score, pointwise in $\bX$ and hence on average,
\[
  \mathcal R_2(\bm s)
  -\E\Big[S_{k_{\tilde\gamma}}\big(\tilde Q(\bX),\tilde\bY\big)\Big]
  =\E_{\bX}\Big[\tfrac12\mathrm{MMD}^{2}_{\tilde\gamma}
  \big(\tilde Q(\bX),\,\cN(\bm0,R_{\bm s(\bX)})\big)\Big].
\]
Since $\mathcal R_{2,M}(\bm s_M)\le\mathcal R_{2,M}(\bm s)$ for every
$\bm s\in\mathcal G_2$, an ERM decomposition analogous to Step~2 of
Lemma~\ref{lem:pac-stage1-rate} gives
\[
  \E_{\bX}\Big[\tfrac12\mathrm{MMD}^{2}_{\tilde\gamma}
  \big(\tilde Q(\bX),\cN(\bm0,R_{\bm s_M(\bX)})\big)\Big]
  \le\inf_{\bm s\in\mathcal G_2}
     \E_{\bX}\Big[\tfrac12\mathrm{MMD}^{2}_{\tilde\gamma}
     \big(\tilde Q(\bX),\cN(\bm0,R_{\bm s(\bX)})\big)\Big]
  +2\sup_{\bm s\in\mathcal G_2}
     \lvert\mathcal R_{2,M}(\bm s)-\mathcal R_2(\bm s)\rvert .
\]
We bound the two terms on the right in turn.
By the Lipschitz bound of Lemma~\ref{lem:pac-regularity}, passing to an
$\eta_2$-net of $\mathcal G_2$ of size $\mathfrak N_2(\eta_2)$ moves
$\mathcal R_{2,M}$ and $\mathcal R_2$ by at most $L_2'\eta_2$ each. At each net
point $\mathcal R_{2,M}-\mathcal R_2$ is an average of $M$ independent terms of
range $1$, so Hoeffding's inequality
\citep[Ch.~2]{boucheron2013concentration} and a union bound give, with
probability at least $1-\delta_2$,
\[
  \sup_{\bm s\in\mathcal G_2}
  \big\lvert\mathcal R_{2,M}(\bm s)-\mathcal R_2(\bm s)\big\rvert
  \;\le\;2L_2'\eta_2
  +\sqrt{\frac{\log\big(2\mathfrak N_2(\eta_2)/\delta_2\big)}{2M}}
  \;=\;\tfrac12\rho_M(\eta_2),
\]
so the deviation term contributes $\rho_M(\eta_2)$.
For the infimum, Lemma~\ref{lem:pac-stage1} replaces $\tilde Q$ by
$\tilde Q^{\star}$ inside it -- the first of two uses of that lemma, one on each
side of the ERM comparison. Its bound is uniform in $\bm s$, so it transfers to
the infimum, and Assumption~\ref{ass:class}\,(v) then identifies what remains
with the approximation error,
\begin{align*}
  \inf_{\bm s\in\mathcal G_2}
  \E_{\bX}\Big[\tfrac12\mathrm{MMD}^{2}_{\tilde\gamma}
    \big(\tilde Q(\bX),\cN(\bm0,R_{\bm s(\bX)})\big)\Big]
  &\;\le\;\inf_{\bm s\in\mathcal G_2}
  \E_{\bX}\Big[\tfrac12\mathrm{MMD}^{2}_{\tilde\gamma}
    \big(\tilde Q^{\star}(\bX),\cN(\bm0,R_{\bm s(\bX)})\big)\Big]
  +\Delta_N\\[2pt]
  &\;=\;\E_{\bX}\Big[\min_{\bm s\in\Theta}
    \tfrac12\mathrm{MMD}^{2}_{\tilde\gamma}
    \big(\tilde Q^{\star}(\bX),\cN(\bm0,R_{\bm s})\big)\Big]+\Delta_N\\[2pt]
  &\;=\;\bar\varepsilon_{\mathrm{app}}+\Delta_N .
\end{align*}
On the event of Lemma~\ref{lem:pac-stage1-rate}, of probability at least
$1-\delta_1$, the first-stage term is controlled. Indeed, using
$a+b\le\sqrt2\,(a^{2}+b^{2})^{1/2}$,
\[
  \Delta_N
  \;=\;\sqrt{\tfrac2e}\,\big(e_{\mu,N}+e_{\tau,N}\big)
  \;\le\;\frac{2}{\sqrt e}\big(e_{\mu,N}^{2}+e_{\tau,N}^{2}\big)^{1/2}
  \;\le\;\frac{2}{\sqrt e}\sqrt{\frac{2\rho_N(\eta_1)}{c_\star\tau_{\min}}}
  \;=\;c_2\sqrt{\rho_N(\eta_1)} .
\]
Collecting the three contributions in the ERM decomposition,
\[
  \E_{\bX}\Big[\tfrac12\mathrm{MMD}^{2}_{\tilde\gamma}
  \big(\tilde Q(\bX),\cN(\bm0,R_{\bm s_M(\bX)})\big)\Big]
  \;\le\;\bar\varepsilon_{\mathrm{app}}+c_2\sqrt{\rho_N(\eta_1)}+\rho_M(\eta_2),
\]
which bounds the core against $\tilde Q$, the law it is actually fitted to.
Statement \eqref{eq:pac-rate} is against the ideal $\tilde Q^{\star}$ instead,
so Lemma~\ref{lem:pac-stage1} is applied a second time, now on the left-hand
side at $\bm s=\bm s_M$, at the cost of a further
$\Delta_N\le c_2\sqrt{\rho_N(\eta_1)}$, which accounts for the factor $2$ in
\eqref{eq:pac-rate}. A union bound over the events of
Lemma~\ref{lem:pac-stage1-rate} and of the Hoeffding step completes the proof.
\end{proof}
\subsection{\score{} is PAC}
The deployed prediction $P_{N,M}$ and the kernel score it is evaluated with both
live in the original coordinates, at a bandwidth $\gamma>0$ that we now keep arbitrary. We carry \eqref{eq:pac-rate} over by a two-sided comparison of the squared MMD with the Frobenius norm and a coupling bound for the marginal error. Lemma~\ref{lem:mmd-frobenius}\,(i) already provides such a lower bound, but its constant decays exponentially in $d$, which would carry an exponential dependence on the dimension into the main theorem below. At the second-stage bandwidth $\tilde\gamma^{2}=4d$, part~(ii) of the following lemma replaces it by a constant that does not decay exponentially in $d$.
\begin{lemma}[Frobenius sandwich (refined)]\label{lem:pac-frobenius}
Let $\bSigma_1,\bSigma_2\succeq0$ be real symmetric with
$\tr\bSigma_1=\tr\bSigma_2$, and let $\bmu\in\R^{d}$ be the common mean.
\begin{enumerate}[leftmargin=25pt]
\item[(i)] \emph{Upper bound.} For every $\gamma>0$,
  \[
    \mathrm{MMD}_\gamma^{2}\big(\cN(\bmu,\bSigma_1),\cN(\bmu,\bSigma_2)\big)
    \;\le\;\frac{2\lVert\bSigma_1-\bSigma_2\rVert_F^{2}}{\gamma^{4}} .
  \]
\item[(ii)] \emph{Lower bound.} If in addition $\bSigma_1,\bSigma_2$ are
  correlation matrices, then at the second-stage bandwidth
  $\tilde\gamma^{2}=4d$,
  \[
    \mathrm{MMD}_{\tilde\gamma}^{2}\big(\cN(\bmu,\bSigma_1),\cN(\bmu,\bSigma_2)\big)
    \;\ge\;\frac{2e^{-10}\lVert\bSigma_1-\bSigma_2\rVert_F^{2}}
                {\tilde\gamma^{4}} .
  \]
\end{enumerate}
\end{lemma}
\begin{proof}
The kernel is translation invariant, so we may take $\bmu=\bm0$ throughout.
The upper bound (i) then follows exactly as in the proof of
Lemma~\ref{lem:mmd-frobenius}\,(i), where the assumption that the two matrices
are correlations is used only through $\tr(\bSigma_1-\bSigma_2)=0$.
For the lower bound, take the Gaussian kernel at $\gamma=\tilde\gamma$ and
recall from that same proof the representation
\[
  \mathrm{MMD}_{\tilde\gamma}^{2}
  \;=\;\E\Big[\big(e^{-a(\bm z)}-e^{-b(\bm z)}\big)^{2}\Big],
  \qquad
  a(\bm z):=\frac{\bm z^{\top}\bSigma_1\bm z}{\tilde\gamma^{2}},
  \qquad
  b(\bm z):=\frac{\bm z^{\top}\bSigma_2\bm z}{\tilde\gamma^{2}},
\]
where $\E$ is with respect to $\bm z\sim\cN(\bm0,\bI)$ here and below. Since
$\bSigma_1,\bSigma_2\succeq0$ we have $a,b\ge0$, so the mean value theorem gives
\[
  \big(e^{-a}-e^{-b}\big)^{2}\;\ge\;(a-b)^{2}e^{-2(a+b)}.
\]
If $\bSigma_1=\bSigma_2$ there is nothing to prove. Otherwise, with
$\bm N:=\bSigma_1-\bSigma_2$, the second-moment identity for Gaussian quadratic
forms from the proof of Lemma~\ref{lem:mmd-frobenius},
$\E(\bm z^{\top}\bm N\bm z)^{2}=2\lVert\bm N\rVert_F^{2}+(\tr\bm N)^{2}$, gives
\begin{equation}\label{eq:pac_frob_second_moment}
  \E(a-b)^{2}\;=\;\frac{2\lVert\bSigma_1-\bSigma_2\rVert_F^{2}}{\tilde\gamma^{4}}
  \;>\;0,
\end{equation}
since $\tr\bm N=0$, and
\[
  \mathrm d\nu(\bm z)
  \;:=\;\frac{\big(a(\bm z)-b(\bm z)\big)^{2}}{\E (a-b)^{2}}\;
  \frac{e^{-\lVert\bm z\rVert^{2}/2}}{(2\pi)^{d/2}}\;\mathrm d\bm z
\]
defines a probability measure on $\R^{d}$. Jensen's inequality for expectations with respect to $\nu$ and the convex map
$x\mapsto e^{-2x}$ gives
\begin{equation}
  \E_{\nu}\big[e^{-2(a+b)}\big]\;\ge\;\exp\big(-2\,\E_{\nu}[a+b]\big),
  \quad\text{i.e.}\quad
  \E\big[(a-b)^{2}e^{-2(a+b)}\big]\;\ge\;\E\left[(a-b)^{2}\right]
  \exp\Big(-2\,\frac{\E\big[(a+b)(a-b)^{2}\big]}{\E (a-b)^{2}}\Big).\label{eq:pac_lower_frob_bd}
\end{equation}
Moreover, writing $\lVert\cdot\rVert_{p}:=\big(\E\lvert\cdot\rvert^{p}\big)^{1/p}$ and noting that all moments are finite, since $a,b$ are quadratic forms in a Gaussian vector,
Hölder's inequality with exponents $3$ and $\tfrac32$, applied to $a+b$ and
$(a-b)^{2}$, gives
\[
  \E\big[(a+b)(a-b)^{2}\big]
  \;\le\;\lVert a+b\rVert_{3}\,\big\lVert (a-b)^{2}\big\rVert_{3/2}
  \;=\;\lVert a+b\rVert_{3}\,\lVert a-b\rVert_{3}^{2}.
\]
The last step holds because
$\big\lVert (a-b)^{2}\big\rVert_{3/2}=\big(\E\lvert a-b\rvert^{3}\big)^{2/3}
=\lVert a-b\rVert_{3}^{2}$.
Since $a-b=\bm z^{\top}(\bSigma_1-\bSigma_2)\bm z/\tilde\gamma^{2}$ is a
polynomial of degree two in the coordinates of the standard Gaussian vector
$\bm z$, Gaussian hypercontractivity
\citep[Theorem~5.10 and Remark~5.11]{Janson_1997} bounds its higher moments by
the second: for such polynomials $\lVert\cdot\rVert_{p}\le(p-1)\lVert\cdot\rVert_{2}$
for every $p\ge2$, and at $p=3$
\[
  \lVert a-b\rVert_3\;\le\;2\,\lVert a-b\rVert_2.
\]
By the mean value bound and \eqref{eq:pac_lower_frob_bd}, together with Hölder
and hypercontractivity, therefore,
\begin{align}
  \mathrm{MMD}_{\tilde\gamma}^{2}
  &\;\ge\;\E\big[(a-b)^{2}e^{-2(a+b)}\big]
  \;\ge\;\E\big[(a-b)^{2}\big]\,
  \exp\Big(-2\,\frac{\E\big[(a+b)(a-b)^{2}\big]}
                    {\E\big[(a-b)^{2}\big]}\Big) \nonumber\\[2pt]
  &\;\ge\;\E\big[(a-b)^{2}\big]\,
  \exp\Big(-2\,\frac{\lVert a+b\rVert_3\,\lVert a-b\rVert_3^{2}}
                    {\E\big[(a-b)^{2}\big]}\Big)
  \;\ge\;\E\big[(a-b)^{2}\big]\,e^{-8\lVert a+b\rVert_3}. \label{eq:pac_frob_chain}
\end{align}
It remains to bound $\lVert a+b\rVert_3$. Note that
$\bSigma_1+\bSigma_2$ is real symmetric, so there are an orthogonal $\bm U$ and a
diagonal $\bm D=\diag(\theta_1,\dots,\theta_d)$ with
$\bSigma_1+\bSigma_2=\bm U^{\top}\bm D\,\bm U$, and $\theta_i\ge0$ because
$\bSigma_1,\bSigma_2\succeq0$.
Write
\[
  a+b=\frac{\bm z^{\top}(\bSigma_1+\bSigma_2)\bm z}{\tilde\gamma^{2}}
  =\frac{1}{\tilde\gamma^{2}}\sum_{i=1}^{d}\theta_i\,(\bm U\bm z)_i^{2},
\]
and $\bm U\bm z\sim\cN(\bm0,\bI)$ since $\bm U$ is orthogonal, so $a+b$ has
the same law as $\tilde\gamma^{-2}\sum_i\theta_iz_i^{2}$.
Both $\bSigma_1$ and $\bSigma_2$ are correlation matrices, so
$\sum_i\theta_i=\tr(\bSigma_1+\bSigma_2)=2d$, and Minkowski's inequality gives
\[
  \lVert a+b\rVert_3
  \;\le\;\frac{1}{\tilde\gamma^{2}}\sum_{i=1}^{d}\theta_i\,
  \big\lVert z_i^{2}\big\rVert_3
  \;=\;\frac{2d}{\tilde\gamma^{2}}\,15^{1/3}
  \;=\;\tfrac12\,15^{1/3},
\]
since $\lVert z_i^{2}\rVert_3=(\E Z^{6})^{1/3}=15^{1/3}$ for $Z\sim\cN(0,1)$ and
$\tilde\gamma^{2}=4d$.
Since $4\cdot15^{1/3}<10$, \eqref{eq:pac_frob_chain} and
\eqref{eq:pac_frob_second_moment} therefore give
\[
  \mathrm{MMD}_{\tilde\gamma}^{2}
  \;\ge\;\E\big[(a-b)^{2}\big]\,e^{-8\lVert a+b\rVert_3}
  \;\ge\;e^{-10}\,\E\big[(a-b)^{2}\big]
  \;=\;\frac{2e^{-10}}{\tilde\gamma^{4}}\lVert\bSigma_1-\bSigma_2\rVert_F^{2}
  \;=\;\frac{e^{-10}}{8d^{2}}\lVert\bSigma_1-\bSigma_2\rVert_F^{2}.
\]
\end{proof}
\begin{theorem}[PAC bound in the original coordinates]\label{thm:pac}
Let $(\bmu_N,\bm\tau_N)$ and $\bm s_M$ be the empirical risk minimizers of the
two stages and $P_{N,M}$ the deployed prediction they define. Let $\gamma>0$,
$\eta_1,\eta_2>0$, and $\delta_1,\delta_2\in(0,1)$, and put
\[
  K:=\frac{32\,e^{10}\,\tau_{\max}^{2}\,d^{2}}{\gamma^{4}},
  \qquad
  c_1:=2Kc_2+\frac{4\sqrt{2d}}{\gamma\sqrt{e\,c_\star}} ,
\]
with $\rho_N(\eta_1)$, $\sigma_1$ and $L_1'$ as in
Lemma~\ref{lem:pac-stage1-rate}, $c_2$ and $\rho_M(\eta_2)$ as in
Proposition~\ref{prop:pac-rate}, $c_\star$ the curvature constant of
Lemma~\ref{lem:crps-curvature} and $L_2'$ the Lipschitz constant of
Lemma~\ref{lem:pac-regularity}. Under Assumptions~\ref{ass:target}
and~\ref{ass:class}, with probability at least $1-\delta_1-\delta_2$,
\begin{equation}\label{eq:pac-orig}
  \E_{\bX}\Big[\tfrac12\mathrm{MMD}^{2}_{\gamma}
  \big(Q(\bX),\,P_{N,M}(\bX)\big)\Big]
  \;\le\;
  \underbrace{K\,\bar\varepsilon_{\mathrm{app}}}_{\text{appr.}}
  +\underbrace{c_1\sqrt{\rho_N(\eta_1)}}_{\text{first stage}}
  +\underbrace{K\,\rho_M(\eta_2)}_{\text{second stage}} .
\end{equation}
In particular, let $\varepsilon>0$ and $\delta\in(0,1)$, and choose
$\delta_1=\delta_2=\delta/2$, $\eta_1=\varepsilon^{2}/(16c_1^{2}L_1')$ and
$\eta_2=\varepsilon/(16KL_2')$. Choose
\begin{align*}
  N&\ \ge\ \frac{128\,c_1^{4}\,\sigma_1^{2}}{\varepsilon^{4}}
  \left(q\log\Big(1+\frac{32\,W L_1 c_1^{2}L_1'}{\varepsilon^{2}}\Big)
        +\log\frac{4}{\delta}\right),\\[2pt]
  M&\ \ge\ \frac{32\,K^{2}}{\varepsilon^{2}}
  \left(q\log\Big(1+\frac{32\,W L_2 K L_2'}{\varepsilon}\Big)
        +\log\frac{4}{\delta}\right),
\end{align*}
where $q=\dim\mathcal W$ is the number of network parameters, $W$ the bound on
them and $L_1,L_2$ the parameter-Lipschitz constants of
Assumption~\ref{ass:class}\,(iii)--(iv) and
$L_1',L_2'$ the loss Lipschitz constants. Then
\[
  \E_{\bX}\Big[\tfrac12\mathrm{MMD}^{2}_{\gamma}
  \big(Q(\bX),P_{N,M}(\bX)\big)\Big]
  \;\le\;K\,\bar\varepsilon_{\mathrm{app}}+\varepsilon
\]
with probability at least $1-\delta$.
\end{theorem}
\begin{proof}
Fix $\bx$ and drop it from the notation. Let
$Q'':=\cN\big(\bmu_Q,\bD_QR_M\bD_Q\big)$ carry the exact marginals and the
fitted correlation. Since $\mathrm{MMD}_\gamma$ is a metric,
$(x+y)^{2}\le2x^{2}+2y^{2}$ gives
\[
  \tfrac12\mathrm{MMD}^{2}_{\gamma}(Q,P_{N,M})
  \;\le\;\mathrm{MMD}_\gamma^{2}(Q,Q'')+\mathrm{MMD}_\gamma^{2}(Q'',P_{N,M}) .
\]
\emph{Correlation error.} $Q$ and $Q''$ share the mean $\bmu_Q$, and their
covariances share the diagonal $\bm\tau_Q$ and hence their traces, so
Lemma~\ref{lem:pac-frobenius}\,(i) and
$\lVert\bD_Q\rVert_{\mathrm{op}}^{2}\le\tau_{\max}$ give
\[
  \mathrm{MMD}_\gamma^{2}(Q,Q'')
  \;\le\;\frac{2}{\gamma^{4}}\big\lVert\bD_Q(R_Q-R_M)\bD_Q\big\rVert_F^{2}
  \;\le\;\frac{2\tau_{\max}^{2}}{\gamma^{4}}\lVert R_Q-R_M\rVert_F^{2} .
\]
By Lemma~\ref{lem:pac-frobenius}\,(ii), applied to $R_Q$ and $R_M$, which
both have trace $d$,
\[
  \lVert R_Q-R_M\rVert_F^{2}
  \;\le\;16\,e^{10}d^{2}\cdot
  \tfrac12\mathrm{MMD}^{2}_{\tilde\gamma}
  \big(\tilde Q^{\star},\cN(\bm0,R_M)\big),
\]
using $\tilde\gamma^{4}=16d^{2}$. Hence
$\mathrm{MMD}_\gamma^{2}(Q,Q'')\le K\cdot
\tfrac12\mathrm{MMD}^{2}_{\tilde\gamma}\big(\tilde Q^{\star},\cN(\bm0,R_M)\big)$.
\emph{Marginal error.} $Q''$ and $P_{N,M}$ share the correlation $R_M$ and
differ only in the first-stage parameters. Couple
$\bm U:=\bmu_Q+\bD_Q\bm V\sim Q''$ and $\bm U':=\bmu_N+\bD_N\bm V\sim P_{N,M}$
through one $\bm V\sim\cN(\bm0,R_M)$, and let $(\bm U_2,\bm U_2')$ be an
independent copy of $(\bm U,\bm U')$. By the definition of
$\mathrm{MMD}_\gamma$ and the symmetry of the kernel,
\[
  \mathrm{MMD}_\gamma^{2}(Q'',P_{N,M})
  =\E\big[k_\gamma(\bm U,\bm U_2)-k_\gamma(\bm U,\bm U_2')\big]
  +\E\big[k_\gamma(\bm U',\bm U_2')-k_\gamma(\bm U',\bm U_2)\big]
  \le2L_\gamma\,\E\lVert\bm U-\bm U'\rVert ,
\]
where $L_\gamma:=\sqrt2/(\gamma\sqrt e)$ is the Lipschitz constant of
$k_\gamma$ in each argument. As $\diag(R_M)=\bm1$,
\[
  \E\lVert\bm U-\bm U'\rVert
  \le\lVert\bmu_Q-\bmu_N\rVert+\E\lVert(\bD_Q-\bD_N)\bm V\rVert
  \le\lVert\bmu_Q-\bmu_N\rVert+\lVert\sqrt{\bm\tau_Q}-\sqrt{\bm\tau_N}\rVert
  \le\sqrt{2d}\;\bar e ,
\]
so $\mathrm{MMD}_\gamma^{2}(Q'',P_{N,M})
\le\frac{4\sqrt d}{\gamma\sqrt e}\,\bar e$, with $\bar e$ the first-stage
error \eqref{eq:pac-stage1-orig}.
\emph{Conclusion.} Taking $\E_{\bX}$ and using
$\E_{\bX}\bar e(\bX)\le\bar e_N$ by Cauchy--Schwarz, the first bound of
Lemma~\ref{lem:pac-stage1-rate},
\[
  \bar e_N\;\le\;\sqrt{\frac{2\rho_N(\eta_1)}{c_\star}},
\]
shows that the marginal error contributes at most
$\frac{4\sqrt{2d}}{\gamma\sqrt{e\,c_\star}}\sqrt{\rho_N(\eta_1)}$. The
correlation error contributes $K$ times the left-hand side of
\eqref{eq:pac-rate}, which by Proposition~\ref{prop:pac-rate} is at most
$K\bar\varepsilon_{\mathrm{app}}+2Kc_2\sqrt{\rho_N(\eta_1)}+K\rho_M(\eta_2)$.
Both statements hold on the event of Proposition~\ref{prop:pac-rate}, which
contains that of Lemma~\ref{lem:pac-stage1-rate} and has probability at least
$1-\delta_1-\delta_2$; collecting the two $\sqrt{\rho_N(\eta_1)}$ terms into
$c_1$ proves \eqref{eq:pac-orig}.
For the sample complexity, $\delta_1=\delta_2=\delta/2$ turns both logarithms
into $\log\big(4\mathfrak N_j(\eta_j)/\delta\big)$, and Lemma~\ref{lem:covering}
bounds these by $q\log\big(1+2WL_j/\eta_j\big)+\log(4/\delta)$; the stated
conditions on $N$ and $M$ use the latter and are therefore at least as strong
as what the two results require. The choice of $\eta_2$ makes
$4KL_2'\eta_2=\varepsilon/4$, and the bound on $M$ makes the Hoeffding part of
$K\rho_M(\eta_2)$ at most $\varepsilon/4$; the choice of $\eta_1$ makes
$2L_1'\eta_1=t/2$ with $t:=\varepsilon^{2}/(4c_1^{2})$, and the bound on $N$
makes the deviation part of $\rho_N(\eta_1)$ at most $t/2$, so that
$c_1\sqrt{\rho_N(\eta_1)}\le\varepsilon/2$. In total the right-hand side of
\eqref{eq:pac-orig} is at most $K\bar\varepsilon_{\mathrm{app}}+\varepsilon$.
\end{proof}

\section{Toy example}
\label{app:toy_example}
\begin{figure}[ht]
    \centering
    \includegraphics[width=\linewidth]{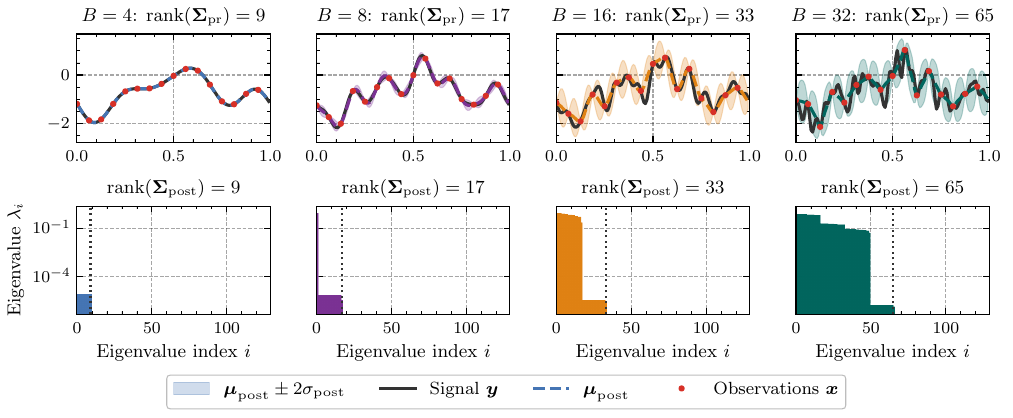}
    \caption{Realizations of the signal $\by$ for different band limits $B \in \{4,8,16,32\}$ (upper) and corresponding eigenvalues of the posterior covariance $\bSigma_\mathrm{post}$ (lower).}
    \label{fig:toy_example1_rank_visualization}
\end{figure}

The toy example considers a reconstruction problem on a periodic grid. We generate a discrete signal $\by \in \R^d$, which is distributed according to a Gaussian prior $\by \sim \cN(\bm 0, \bSigma_\mathrm{pr})$, where $\bSigma_\mathrm{pr}$ is chosen as circulant with a red power spectrum $\lambda_i \geq 0$. Then, we impose a band limit, i.e., we set $\lambda_i=0$ for all $|i| > B$ for some $B \in \mathbb{N}$, which implies $\mathrm{rank}(\bSigma_\mathrm{pr})=\min(2B{+}1,d)$.
Then, we observe a noisy subsample of the signal $\by$ as
\[
  \bx=\bS\by+\bm\eta\in\R^{n},
  \qquad \bm\eta\sim\cN(\bm0,\sigma^2\bI_n),
\]
where $\bS \in \R^{n \times d}$ selects every $m$-th coordinate with $n=d/m$ and $\sigma^2$ is the variance of the observation noise. Now the task is to reconstruct $\by$ given $\bx$ using a neural network.
Since prior and likelihood are Gaussian, the posterior of $\by$ given $\bx$ is also Gaussian. In particular, since $\by$ and $\bm\eta$ are independent, we have that
\begin{align*}
    \mathrm{Cov}(\by, \by) = \bSigma_\mathrm{pr}, \quad \mathrm{Cov}(\by, \bx) = \bE[\by (\bS \by + \bm\eta)^\top] = \bSigma_\mathrm{pr} \bS^\top, \quad \mathrm{Cov}(\bx, \bx) =  \bS \bSigma_\mathrm{pr} \bS^\top + \sigma^2 \bI_n.
\end{align*}
Then, via Gaussian conditioning, we obtain $\by\mid\bx\sim\cN(\bmu_\mathrm{post}(\bx),\bSigma_\mathrm{post})$, with
\begin{align*}
    \bmu_\mathrm{post}(\bx) &= \bSigma_\mathrm{pr} \bS^\top
        \bigl( \bS \bSigma_\mathrm{pr} \bS^\top + \sigma^2 \bI_n\bigr)^{-1}\bx, \\
    \bSigma_\mathrm{post} &= \bSigma_\mathrm{pr} - \bSigma_\mathrm{pr} \bS^\top
        \bigl( \bS \bSigma_\mathrm{pr} \bS^\top + \sigma^2 \bI_n\bigr)^{-1} \bS \bSigma_\mathrm{pr}.
\end{align*}
The posterior covariance $\bSigma_\mathrm{post}$ is independent of $\bx$, so the task is homoscedastic, and $\operatorname{rank}(\bSigma_\mathrm{post})=\min(2B{+}1,d)$. For this toy example, we use $d=128$, $m=8$, $\sigma^2=10^{-6}$ and varying $B$ to analyze the impact of the rank of the posterior on the estimation. \autoref{fig:toy_example1_rank_visualization} shows a realized signal for varying $B$, as well as the corresponding eigenvalues of the posterior covariance $\bSigma_\mathrm{post}$.
It can be seen that for $B=4$ the reconstruction task is deterministic up to the observation noise, which is not visible in the figure. For $B=8$ the uncertainty in the posterior increases, and for further increasing $B$, the posterior mean deviates and the standard deviation becomes visible. Thus, for large $B$, the signal cannot be exactly reconstructed based on the limited observations.

For the reconstruction, we train a linear map $\bx\mapsto\bmu_\phi$ and a learnable Cholesky factor $\bSigma_\phi=\bL\bL^\top$ with softplus-constrained diagonal. We train and evaluate on 4096 and 1024 samples, respectively, use a batch size of 256, and train for 5000 steps with Adam (learning rate $10^{-3}$ for $\bL$, $3\cdot10^{-2}$ for the mean, cosine annealing). We compare the Gaussian kernel score with $\gamma^2 = 2d\cdot\overline{\operatorname{diag}\bSigma_\mathrm{post}}$ against the log score, unconstrained and with an added floor $\varepsilon\bI$, $\varepsilon=10^{-6}$. Results are averaged over five seeds

\begin{table}[t]
\centering
\caption{Frobenius distance $\lVert\bSigma_\phi-\bSigma_\mathrm{post}\rVert_F$ between the learned covariance $\bSigma_\phi$ and the true posterior covariance $\bSigma_\mathrm{post}$ across different band limits $B$. The results are averaged across five seeds with the best model in bold.}
\label{tab:toy1_metrics}
\begin{tabular}{lcccc}
\toprule
 & $B = 4$ & $B = 8$ & $B = 16$ & $B = 32$ \\
\midrule
$S_k$ & $\mathbf{0.00357} \pm 3.7\cdot10^{-6}$ & $\mathbf{0.108} \pm 0.014$ & $\mathbf{0.618} \pm 0.011$ & $\mathbf{0.584} \pm 0.010$ \\
$S_\mathrm{log}$ ($\varepsilon = 10^{-6}$) & $1.15 \pm 0.26$ & $1.43 \pm 0.11$ & $2.07 \pm 0.22$ & $1.95 \pm 0.10$ \\
$S_\mathrm{log}$ & $0.953 \pm 0.13$ & $1.33 \pm 0.12$ & $2.08 \pm 0.22$ & $1.99 \pm 0.26$ \\
\bottomrule
\end{tabular}
\end{table}

\begin{figure}[ht]
    \centering
    \includegraphics[width=\linewidth]{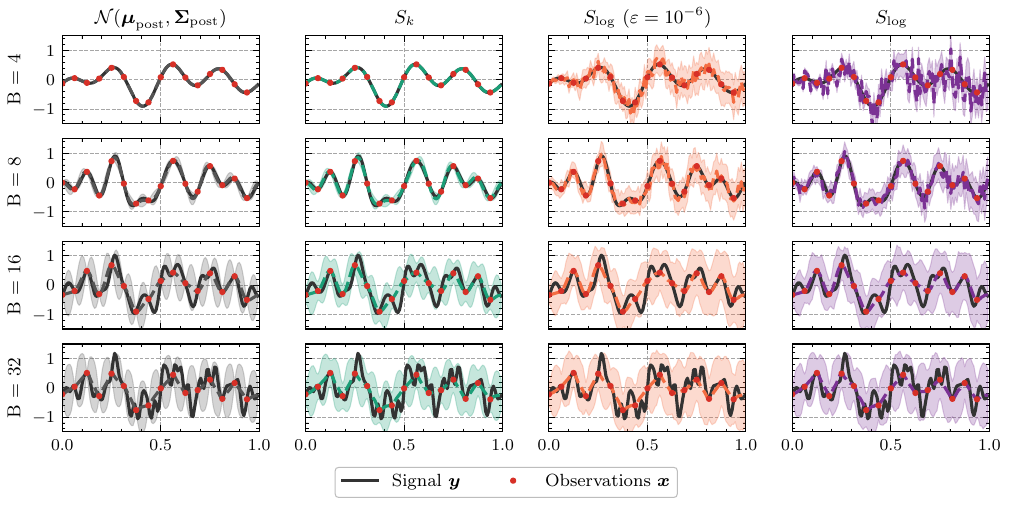}
    \caption{True posterior and different estimates for a selected realization and varying band limit $B\in\{4,8,16,32\}$. The ranks of the posterior covariance are identical to those in \autoref{fig:toy_example1_rank_visualization}.}
    \label{fig:toy_example1_full}
\end{figure}

\begin{figure}[ht]
    \centering
    \includegraphics[width=\linewidth]{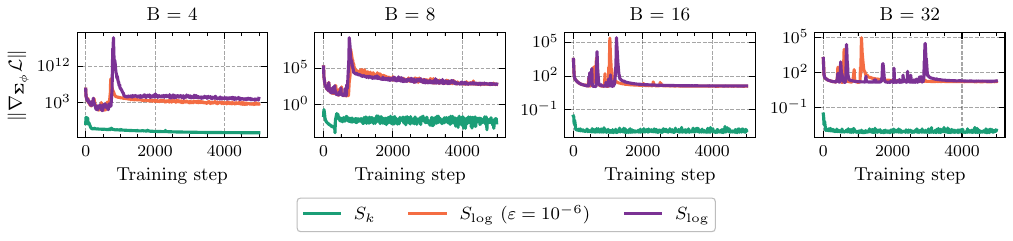}
    \caption{Gradient norm of the different loss functions with respect to the estimated covariance $\bSigma_\phi$, compared across different band limits $B$.}
    \label{fig:toy_example1_grad_norm}
\end{figure}

\autoref{tab:toy1_metrics} shows the Frobenius norm between the different estimates and the true posterior covariance, averaged across multiple seeds. Evidently, the kernel score obtains the best performance across all band limits.
\autoref{fig:toy_example1_full} shows visualizations of the true posterior and the corresponding estimates for varying band limit $B$. For large $B$, where $\bSigma_\mathrm{post}$ admits a larger rank, the estimates of the kernel score and the log score are visually indistinguishable, although a significant gap in the Frobenius norm remains. For small $B$, the log score cannot resolve the low-rank structure and leads to very noisy and incorrect estimates.
\autoref{fig:toy_example1_grad_norm} visualizes the gradient norms of the different loss functions with respect to the estimated $\bSigma_\phi$. Even with an added floor, the log score has a gradient norm several orders of magnitude larger than the kernel score, reaching values up to $10^{17}$. While this improves for increasing $B$, it remains much larger and only seems to stabilize later during training. This directly connects to Lemma~\ref{lemma:log_score_gradients} and highlights that the log score leads to poor results when parameterized with a neural network and optimized with stochastic gradient descent.

\section{Experiment details}
\label{app:experiments}
Every configuration is trained with $5$ independent random seeds, and we report
the averaged mean and standard deviation across test sets, so that
differences between methods can be judged against the run-to-run variability of the training procedure.
We use different backbone architectures and only vary the output head, depending on the predictive method. Thus, all methods share the same backbone (and generally, the number of parameters) and are directly comparable. We now further describe our evaluation protocol, as well as the individual datasets and backbones.

\subsection{Evaluation metrics}
\label{app:eval}

All methods are compared on a common set of evaluation metrics, mostly based on proper scoring rules, with lower values indicating better performance.
Below, we define each metric for a single prediction. Unless stated otherwise, reported values are averages over all cases of the respective test set.

Fix a test case $(\bx,\by) \in \cX \times \R^d$, where $\bx$ is the input and $\by$ the observed target, and let $P=P(\cdot\mid\bx)$ denote the predictive distribution returned by the model, with mean $\bmu$. To ensure a fair comparison against the sample-based method, metrics are evaluated from a simulated ensemble
\begin{equation}
\label{eq:eval:ensemble}
    \by^{(1)},\dots,\by^{(M)}\overset{\text{i.i.d.}}{\sim} P, \qquad M=100,
\end{equation}
using unbiased estimators, unless a closed form is available for all methods. For the deterministic baseline, $P$ is the point mass at $\bmu$; for the sample-based method, $\bmu$ and $\bSigma$ denote the ensemble mean and covariance of \eqref{eq:eval:ensemble}.

\paragraph{Mean squared error (MSE).}
The MSE evaluates the predictive mean $\bmu$,
\begin{equation}\label{eq:eval:mse}
  \mathrm{MSE}(P,\by) \;=\; \frac{1}{d}\,\lVert \bmu - \by \rVert_2^2 .
\end{equation}
It ignores the predicted uncertainty entirely and serves to confirm that differences in the probabilistic scores are not driven by differences in the accuracy of the location.

\paragraph{Continuous ranked probability score (CRPS).}
The CRPS \citep{GneitingRaftery2007} is a standard univariate proper scoring rule and is applied here to each marginal separately and averaged over the domain,
\begin{equation}\label{eq:eval:crps}
  \mathrm{CRPS}_M(P,\by)
  \;=\;
  \frac{1}{d}\sum_{i=1}^{d}
  \left(
    \frac{1}{M}\sum_{m=1}^{M} \big\lvert y^{(m)}_i - y_i \big\rvert
    \;-\;
    \frac{1}{2M(M-1)}\sum_{m=1}^{M}\sum_{m'\neq m}
      \big\lvert y^{(m)}_i - y^{(m')}_i \big\rvert
  \right) .
\end{equation}
It depends on $P$ only through its univariate margins and therefore measures the marginal fit. The first term rewards accuracy of the ensemble around the observation, the second term rewards spread and thereby prevents overconfident predictions, so that the score is minimized in expectation only when each marginal predictive distribution is calibrated.

\paragraph{Negative log-likelihood (NLL).}
For a Gaussian prediction $P=\cN(\bmu,\bSigma)$, the NLL per coordinate evaluated at the observation is
\begin{equation}\label{eq:eval:nll}
  \mathrm{NLL}(P,\by)
  \;=\; \frac{1}{2d}\Big(
        \log\det\!\big(2\pi\bSigma\big)
        + (\by-\bmu)^{\!\top}\bSigma^{-1}(\by-\bmu)
        \Big) .
\end{equation}
It is the only metric requiring an explicit density and is therefore not available for the deterministic and the sample-based baseline. It is strictly proper, rewards concentration of the predictive density at the observation, and is unbounded as $\bSigma$ approaches the boundary of the positive definite cone, which is precisely the behaviour analysed in Section~\ref{sec:nll}.

\paragraph{Energy score (ES).}
The energy score \citep{GneitingRaftery2007} is the multivariate generalisation
of the CRPS,
$\mathrm{ES}(P,\by) = \E\lVert \bY - \by\rVert_2
 - \tfrac{1}{2}\E\lVert \bY - \bY'\rVert_2$ with $\bY,\bY'\sim P$ independent,
estimated from \eqref{eq:eval:ensemble} by
\begin{equation}\label{eq:eval:es}
  \mathrm{ES}_M(P,\by)
  \;=\;
  \frac{1}{M}\sum_{m=1}^{M} \big\lVert \by^{(m)} - \by \big\rVert_2
  \;-\;
  \frac{1}{2M(M-1)}\sum_{m=1}^{M}\sum_{m'\neq m}
        \big\lVert \by^{(m)} - \by^{(m')} \big\rVert_2 .
\end{equation}
The pairwise normalisation $M(M-1)$ makes the estimator unbiased, so that its
expectation does not reward larger ensembles \citep{https://doi.org/10.1002/qj.2270}.
The energy score is strictly proper but is known to have limited sensitivity to
misspecification of the dependence structure.

\paragraph{Gaussian kernel score (KS).}
For the Gaussian kernel
$k(\bm{u},\bm{v}) = \exp\!\big(-\lVert \bm{u}-\bm{v}\rVert_2^2/\gamma^2\big)$
with bandwidth $\gamma > 0$, the associated kernel score
\citep{GneitingRaftery2007} is
$\mathrm{KS}(P,\by) = \tfrac{1}{2}\E\, k(\bY,\bY') - \E\, k(\bY,\by)
 + \tfrac{1}{2}k(\by,\by)$ with $\bY,\bY'\sim P$ independent,
estimated by
\begin{equation}\label{eq:eval:ks}
  \mathrm{KS}_M(P,\by) 
  \;=\;
  \frac{1}{2M(M-1)}\sum_{m=1}^{M}\sum_{m'\neq m} k\big(\by^{(m)},\by^{(m')}\big)
  \;-\;
  \frac{1}{M}\sum_{m=1}^{M} k\big(\by^{(m)},\by\big)
  \;+\; \frac{1}{2} ,
\end{equation}
using $k(\by,\by) = 1$. We evaluate at $\gamma^2=d$ (Appendix~\ref{app:score}).
It is the sample counterpart of one half of the squared MMD between $P$
and the point mass at $\by$, and is strictly proper.
In contrast to the NLL, it is bounded, which makes it comparable across all
training objectives even when a model has degenerate predicted covariance.

\paragraph{Variogram score (VS).}
The variogram score of order $\alpha$ \citep{VariogramBasedProperScoringRulesforProbabilisticForecastsofMultivariateQuantities} compares observed and predicted increments,
\begin{equation}\label{eq:eval:vs}
  \mathrm{VS}^{(\alpha)}(P,\by)
  \;=\;
  \frac{1}{\sum_{i\neq j} w_{ij}}\sum_{i\neq j} w_{ij}
  \Big(
    \lvert y_i - y_j \rvert^{\alpha}
    \;-\;
    \E\big\lvert Y_i - Y_j \big\rvert^{\alpha}
  \Big)^{2}, \qquad \bY\sim P.
\end{equation}
 It is proper but not strictly proper. Being a function of increments only, it is invariant to a common shift of the field and hence blind to a constant bias in the mean, which is what makes it a targeted measure of the predicted dependence structure.
We select $\alpha=1$, for which $Y_i-Y_j\sim\cN(\mu_i-\mu_j,\,\Sigma_{ii}+\Sigma_{jj}-2\Sigma_{ij})$ under a Gaussian $P=\cN(\bmu,\bSigma)$, so that $\E\lvert Y_i-Y_j\rvert$ is the mean of a folded normal distribution and available in closed form. Following \citet{VariogramBasedProperScoringRulesforProbabilisticForecastsofMultivariateQuantities}, the weights $w_{ij}$ are inversely proportional to the distance between coordinates where a physical distance is available and uniform otherwise. 

\subsection{Calibration}
\label{app:calibration}
While the previous metrics provide average performance scores, this does not reveal how a method fails, whether through a biased mean, a misspecified marginal scale, or a dependence structure that is too weak or too strong.
We therefore consider evaluation beyond pure score comparison by analyzing the calibration of the predictive distribution. In particular, we consider \emph{probabilistic calibration}: the observation should behave like a draw from its own predictive distribution \citep{gneiting2005calibrated, 10.1214/23-EJS2180}. This type of calibration can easily be analyzed by checking whether the probability integral transform (PIT) of the prediction is uniformly distributed.

Let $(\bx_t,\by_t)_{t=1}^{T}$ denote the evaluation set with $\by_t \in \R^d$,
and let $F_t = \Gauss(\bmu_t,\bSigma_t)$ be the associated predictive
distribution, with marginal variances
$\sigma_{t,i}^2 = (\bSigma_t)_{ii}$.
All diagnostics test the null hypothesis $\by_t \sim F_t$.
First, we consider marginal calibration, which isolates the pointwise
fit.  Since the margins of $F_t$ are Gaussian, the PIT is available in closed form,
\begin{equation}\label{eq:calib:pit}
  p_{t,i}
  \;=\; \Phi\!\left(\frac{y_{t,i}-\mu_{t,i}}{\sigma_{t,i}}\right),
  \qquad i = 1,\dots,d,
\end{equation}
and is uniform on $[0,1]$ under calibration.  We pool $\{p_{t,i}\}$ over cases
and components and inspect the resulting histogram.
This diagnostic depends on $\bSigma_t$ only through its diagonal and verifies that all marginals are calibrated. However, it cannot discriminate between different dependence structures, which requires incorporating the full covariance matrix.

To check multivariate calibration, we follow \cite{https://doi.org/10.1002/qj.4647} and use rank histograms, which can be thought of as a discrete version of the PIT histogram.
The rank histogram rests on the observation that if the observation and $M$
ensemble members are exchangeable, the rank of the observation within the pooled
sample is uniform on $\{1,\dots,M+1\}$. By analyzing the corresponding histogram, one can find systematic departures from uniformity, which indicate specific forms of miscalibration.
Since $F_t$ is available in closed form, we simulate
$\bz_t^{(1)},\dots,\bz_t^{(M)} \overset{\text{iid}}{\sim} F_t$ with $M = 1000$,
which are exchangeable with $\by_t$ under the null.
As there is no canonical order on $\R^d$, ranking requires a
\emph{pre-rank function} \citep{https://doi.org/10.1002/qj.4647} $\prerank_t : \R^d \to \R$, giving
\begin{equation}\label{eq:calib:rank}
  r_t
  \;=\;
  1 + \sum_{m=1}^{M}
  \mathbf{1}\!\left\{\prerank_t(\bz_t^{(m)}) < \prerank_t(\by_t)\right\}
  \;\sim\; \Unif\{1,\dots,M+1\} .
\end{equation}
Here, uniformity is necessary but not sufficient, since a pre-rank function
reduces the field to a scalar summary.  We therefore use three \emph{simple}
pre-rank functions, chosen so that the shape of the histogram maps onto an
identifiable deficiency, whose histograms confound location, scale and dependence. The three are stated for a generic $\bm{v} \in \R^d$ and applied to $\by_t$ and
to each $\bz_t^{(m)}$ alike.

\paragraph{Mean calibration.}
The spatial average
\begin{equation}\label{eq:calib:rho-mean}
  \prerank^{\mathrm{mean}}(\bm{v}) \;=\; \tfrac{1}{d}\,\bone^{\!\top}\bm{v}
\end{equation}
targets the aggregate level of the field. Here, we can utilize a closed form, since ${\prerank^{\mathrm{mean}}(\bm{v}) \sim
 \Gauss\big(d^{-1}\bone^{\!\top}\bmu_t,\;
           d^{-2}\bone^{\!\top}\bSigma_t\bone\big)}$ for $\bm{v}\sim F_t$,        
so \eqref{eq:calib:rank} may be replaced by the exact PIT.
The only quantity required is the scalar $\bone^{\!\top}\bSigma_t\bone$, which
is inexpensive for all covariance families considered.

\paragraph{Scale calibration.}
The empirical variance of a single realisation,
\begin{equation}\label{eq:calib:rho-var}
  \prerank^{\mathrm{var}}(\bm{v})
  \;=\; \tfrac{1}{d}\sum_{i=1}^{d}
        \big(v_i - \prerank^{\mathrm{mean}}(\bm{v})\big)^2 ,
\end{equation}
quantifies the amount of spatial contrast present in the field and is invariant
to its overall level, making it complementary to
\eqref{eq:calib:rho-mean}.  For fixed margins, stronger predicted correlation
reduces the expected contrast, so this pre-rank penalizes predictions whose
realisations are systematically too smooth or too rough.

\paragraph{Dependence calibration.}
A variogram-type statistic over a set $\mathcal{P}_h$ of index pairs at spatial
lag $h$,
\begin{equation}\label{eq:calib:rho-vg}
  \prerank^{\mathrm{vg}}_h(\bm{v})
  \;=\;
  \frac{1}{|\mathcal{P}_h|}
  \sum_{(i,j)\in\mathcal{P}_h}
  w_{ij}\,\big\lvert v_i - v_j \big\rvert^{\alpha},
  \qquad \alpha \in (0,2],
\end{equation}
measures the strength of dependence at a prescribed scale.  Being a function of
increments only, it is exactly invariant to constant shifts and hence
insensitive to bias in the mean, which separates it cleanly from
\eqref{eq:calib:rho-mean}.  We follow \cite{https://doi.org/10.1002/qj.4647} and use $\alpha = 1$ and set $w_{ij} = 1$, except for EUPPBench, where we weight each pair by the inverse distance $w_{ij} = \lVert \bm{s}_i - \bm{s}_j \rVert^{-1}$ between the two locations.

To make visualization more concise, we do not plot the rank histograms
directly, but their ECDF, i.e.,\ $F_u(u)$ against the
normalized rank $u$.  A perfect diagonal is consistent with calibration in terms
of the selected pre-rank function only.
A curve crossing the diagonal from above indicates that the observed summary
falls too often in the extreme ranks, i.e.,\ underdispersion in that aspect, while a crossing from below indicates overdispersion.
A curve staying on one side indicates that the observed summary is systematically smaller (above) or larger (below) than its forecast counterpart.  Concretely, a one-sided curve in $\prerank^{\mathrm{mean}}$ indicates a bias in the domain-averaged field and a crossing from above that $\bone^{\!\top}\bSigma_t\bone$ is too small, which for calibrated margins means correlations that are too weak. A curve below the diagonal in $\prerank^{\mathrm{var}}$ indicates that observed fields carry more spatial contrast than simulated ones, i.e.\ that the forecast is too smooth, and
a one-sided curve in $\prerank^{\mathrm{vg}}_h$ that dependence at lag $h$ is
over- or underestimated.
Since uniformity is necessary but not sufficient, compensating errors may cancel in a summary, and we thus report the three pre-rank diagnostics
jointly alongside the marginal one.

\subsection{Task descriptions}
\label{app:exp_description}

All four tasks follow the same experimental protocol; only the data, the backbone and the hyperparameters in Table~\ref{tab:hyperparams} differ. Every method shares a task-specific feature extractor (the \emph{trunk}) and differs only in its probabilistic head and loss: the single-stage baselines attach one head for the mean and marginal variances and, where applicable, a second head for the low-rank factors, and are trained with either the log score or the kernel score; \score{} attaches the same first head and a second head for the spectral core, and is trained in two stages, the marginals with the CRPS and the core with the kernel score, using the bandwidths of Appendix~\ref{app:score}. Each method receives the same total epoch budget; for \score{} it is split between the two stages, and during the second stage the trunk and the marginal head are frozen so that only the core head is trained. Unless stated otherwise, optimization uses AdamW \citep{loshchilov2018decoupled} with a cosine learning-rate schedule and three warmup epochs, all inputs and targets are standardized with statistics computed on the training portion only, and evaluation uses the best validation checkpoint. Sample-based scores are computed from $M=100$ predictive samples, and every configuration is run with five seeds(1$-5$).

\begin{table}[h]
\centering
\caption{Training configurations for the four tasks. ``Epochs'' gives the total budget and, in parentheses, the marginal/core split for \score{}.}
\label{tab:hyperparams}
\small
\adjustbox{max width=\linewidth}{%
\begin{tabular}{lllll}
\toprule
 & Time series & Surface temp.\ & Depth & Post-processing \\
\midrule
Trunk          & PatchTST \citep{nie2023a}          & U-Cast \citep{cachay2026ucast}    & DepthAnything v2 \citep{depth_anything_v2} & GNN \citep{feik2024graphneuralnetworksspatial} \\
Trunk params   & $81.8\,$k -- $0.9\,$M                & $65.9$\,M       & $22.1$\,M (frozen) $+\,2.7$\,M & $10.1\,$M \\
Epochs         & $100$ ($75/25$)   & $30$ ($20/10$)  & $75$ ($50/25$)  & $40$ ($30/10$) \\
Batch size     & $128$ ($32$ ILI, Electricity; $24$ Traffic) & $16$  & $8$             & $8$ \\
Learning rate  & $10^{-4}$ / $2.5\cdot10^{-3}$ (ILI) & $3\cdot10^{-4}$ ($7\cdot10^{-5}$ core) & $10^{-3}$ & $2\cdot10^{-4}$ \\
Weight decay   & $0.05$ ($0$ Weather, Electricity, Traffic) & $0.1$        & $0.05$          & $0.05$ \\
\bottomrule
\end{tabular}
}
\end{table}

\paragraph{Time series benchmark.}
We follow the long-term forecasting protocol of PatchTST \citep{nie2023a} on the eight data sets summarized in Table~\ref{tab:ts_datasets}. Each example consists of a contiguous look-back window of length $L$ followed immediately by the prediction horizon $T$. For the ETT data sets we study both multivariate forecasting of all variables and univariate forecasting of the target variable \texttt{OT}; ILI, Weather, Electricity and Traffic are evaluated in the multivariate setting only. Electricity and Traffic are the two largest data sets, recording the hourly consumption of 321 clients and the hourly road occupancy of 862 sensors, so that the output dimension $d=CT$ reaches $61{,}632$ and $165{,}504$ at $T=192$. The backbone is PatchTST, an encoder-only Transformer whose channel-independent design processes each channel as a separate univariate series while sharing all weights across channels; accordingly, all predictive heads model each channel independently. We use the per-data-set configurations of \citet{nie2023a} (patch length $P=16$ and stride $S=8$ for $L=336$, $P=24$ and $S=2$ for ILI, both yielding 42 patches; $(d_{\mathrm{model}},H,d_{\mathrm{ff}})=(16,4,128)$ with dropout $0.3$ for ETTh1/2 and ILI, and $(128,16,256)$ with dropout $0.2$ for ETTm1/2, Weather, Electricity and Traffic). For the ETT data sets and ILI, optimization follows the common protocol. For Weather we instead follow the PatchTST recipe for that data set: the learning rate is held constant for three epochs and then multiplied by $0.9$ after every epoch, without weight decay and with early stopping on the validation loss (patience 20); for \score{} this recipe is used for the marginal fit, while the core fit keeps the cosine schedule with warmup. For Electricity and Traffic we likewise follow the PatchTST recipes for these data sets, a one-cycle schedule with 20 warmup epochs over the 100-epoch budget, no weight decay, early stopping on the validation loss (patience 10) and batch sizes of 32 and 24 (Table~\ref{tab:hyperparams}); again, \score{} uses this recipe for the marginal fit and the cosine schedule with warmup for the core fit.

\begin{table}[h]
\centering
\caption{Time-series data sets. Splits are chronological. ETT uses 12/4/4 months, the other data sets use 70/10/20\,\%.}
\label{tab:ts_datasets}
\small
\begin{tabular}{lcccccc}
\toprule
Data set & Variables & Resolution & Observations & Split & $L$ & $T$ \\
\midrule
ETTh1, ETTh2 & 7  & hourly  & $17{,}420$ & 12/4/4 m  & 336 & $\{96,192\}$ \\
ETTm1, ETTm2 & 7  & 15 min  & $69{,}680$ & 12/4/4 m  & 336 & $\{96,192\}$ \\
Weather      & 21 & 10 min  & $52{,}696$ & 70/10/20  & 336 & $\{96,192\}$ \\
Electricity  & 321 & hourly & $26{,}304$ & 70/10/20  & 336 & $\{96,192\}$ \\
Traffic      & 862 & hourly & $17{,}544$ & 70/10/20  & 336 & $\{96,192\}$ \\
ILI          & 7  & weekly  & $966$      & 70/10/20  & 104 & $\{24,36\}$ \\
\bottomrule
\end{tabular}
\end{table}

\paragraph{Surface temperature prediction.}
We use ERA5 reanalysis data \citep{hersbachERA5GlobalReanalysis2020} as provided by WeatherBench2 \citep{rasp2024weatherbench2benchmarkgeneration}. Similar to \citet{Uncertaintyquantificationfordatadrivenweathermodels, pacchiardiProbabilisticForecastingGenerative2021}, we consider a European domain at $0.25^{\circ}$ resolution, $35^{\circ}$--$75^{\circ}$N and $12.5^{\circ}$W--$42.5^{\circ}$E, i.e.\ a $160\times220$ grid with $d=35{,}200$ target coordinates, and one observation per day at 12\,UTC. The data are split by year into training (1990--2014), validation (2015--2019) and test (2020--2022). The input stacks the last two daily observations of five dynamic variables (the 10\,m wind components U10 and V10, temperature at 2\,m and 850\,hPa, and geopotential at 500\,hPa) with the land-sea mask and the surface geopotential, giving twelve channels. The target is the change of the standardized 2\,m temperature over a lead time of three days, so the predictive distribution is over a single channel on the full grid. As trunk we use the U-Net of U-Cast \citep{cachay2026ucast}, a recent model that matches state-of-the-art weather forecasters at a fraction of their size, reduced to 96 base channels because of the smaller number of input channels. The head for the mean and marginal variances has $0.17$\,M parameters; the low-rank models and \score{} add a second head with $0.5$\,M parameters. In the second stage of \score{}, the core head is trained at the reduced learning rate given in Table~\ref{tab:hyperparams}.

\paragraph{Monocular depth estimation.}
We predict a depth map from a single RGB image on NYU Depth~v2 \citep{Silberman:ECCV12}, following the protocol of \citet{liu2023single}. Training uses the $24{,}231$ synchronized image--depth pairs of the standard split, with a fixed $10\%$ held out for validation, and evaluation uses the $654$ densely annotated images of the test split of \citet{eigen2014depth}. All frames are restricted to the standard evaluation crop of the $480\times640$ images, i.e.\ $d=426\times560=238{,}560$ pixels. Depth is stored in millimetres, capped at $10$\,m and rescaled to $[0,1]$. The raw training depth contains projection holes (about $14\%$ of the pixels within the crop), which are filled once by nearest-neighbour interpolation at dataset construction; the test depth maps are the official densely in-painted ones. The trunk is the Depth Anything~V2 architecture \citep{depth_anything_v2}, i.e.\ a frozen DINOv2 ViT-S/14 encoder \citep{oquab2023dinov2} and a trainable DPT decoder \citep{Ranftl2021}, whose scalar output layer is replaced by two convolutional heads of $0.05$\,M parameters each. Inputs are padded to a multiple of the patch size and the prediction is cropped back to the original resolution.

\paragraph{Graph-based temperature post-processing.}
The final task shows how our approach extends to non-equidistant domains and the corresponding unitary transforms. We consider post-processing of two-meter temperature across central Europe: the model takes an ensemble of numerical weather prediction (NWP) forecasts of multiple variables as input and predicts the observed temperature at a set of stations at a lead time of 72\,h. We use the EUPPBench data \citep{essd-15-2635-2023}, which pairs ECMWF medium-range ensemble forecasts with station observations at 122 European stations from 1997 to 2018 and comprises 730 daily operational forecasts with 51 members as well as 4180 reforecasts (NWP runs for past dates) with 11 members, each with 31 variables. We follow the data configuration and the graph neural network of \citet{feik2024graphneuralnetworksspatial}, training on reforecasts from 1997--2013 and evaluating on reforecasts from 2014--2017 and on operational forecasts from 2017--2018, but replace their marginal model by a multivariate Gaussian with the covariance parameterizations under study. We use a reduced set of $d=115$ stations, as the remaining stations have more than 50\% missing values.

The task requires a different transform than the previous ones. Let $G=(V,E,w)$ be a graph on the $d$ output coordinates, e.g.\ a $k$-nearest-neighbour graph with Gaussian edge weights $w_{ij}=\exp(-\mathrm{dist}(i,j)^{2}/2\varepsilon^{2})$ built from the Euclidean distance between stations. With $\bW=(w_{ij})$ and degree matrix $\bDelta_w=\diag(\bW\bm1)$, the symmetric normalized Laplacian $\bL=\bI-\bDelta_w^{-1/2}\bW\bDelta_w^{-1/2}=\bU\bLambda\bU^{*}$ is Hermitian positive semi-definite, so $\bU$ is unitary and $\bB:=\bU^{*}$ defines the \emph{graph Fourier transform}. A diagonal core $\hat\bSigma=\diag\big(g_{\btheta}(\bLambda)\big)$ amounts to filtering the graph spectrum; we use the Mat\'ern filter
\[
  g_{\btheta}(\lambda)=\Big(\tfrac{2\nu}{\kappa^{2}}+\lambda\Big)^{-\nu},
  \qquad \btheta=(\nu,\kappa)\in\R_{>0}^{2}.
\]
Unlike the DFT, a general graph is not vertex-transitive and $|U_{ik}|^{2}\neq 1/d$, so the implied marginal variances
\[
  \Delta_i(\btheta)=\sum_{k}g_{\btheta}(\lambda_k)\,|U_{ik}|^{2}
\]
are non-constant and the normalization must be carried out explicitly. It is available in closed form, however, and merges with $\bD$ into the single diagonal $\bA:=\bD\,\bm\Delta(\btheta)^{-1/2}$, giving
\[
  \bSigma_P=\bA\,\bU\,g_{\btheta}(\bLambda)\,\bU^{*}\bA ,
\]
so the staged objective of Proposition~\ref{prop:two-stage-opt} applies verbatim: $\hat\bD$ carries the marginals exactly and $g_{\btheta}$ is fit to the correlation structure alone. There is no computational gain in this case, since the resulting matrices are dense and their evaluation has the same complexity as a Cholesky decomposition.

\section{Detailed results}
\label{app:results}

\subsection{Numerical results}
The full results are given in Tables~\ref{tab:etth1}--\ref{tab:euppbench}. The log score is reported per coordinate; lower is better for all scores, and we report the mean and standard deviation over five seeds. The best model within one standard deviation is highlighted in bold. The average rank is computed per seed over the scores available for all methods and then averaged over seeds. When two methods are equally good, the corresponding ranks are averaged. A dash marks a score that is not defined for a method.

\paragraph{Time series benchmark.}
On the time series benchmark (Tables~\ref{tab:etth1}--\ref{tab:traffic}), \score{} attains the best average rank on 20 of the 24 tasks, with a mean rank of $2.2$ against $3.1$ for the strongest baseline, LorD trained with the kernel score. At fixed covariance, replacing the NLL by the kernel score improves LorD on at least 19 of the 24 tasks on each of MSE, CRPS, ES, VS and KS, 105 of the 120 task-score pairs in total. Only the log score, which the NLL-trained model optimizes directly, goes the other way. On the three high-dimensional data sets, Weather, Electricity and Traffic, with $d$ between $2{,}016$ and $165{,}504$, LorD trained with the NLL deteriorates sharply on Weather and collapses on Electricity and Traffic, while trained with the kernel score it is on par with the diagonal model (Section~\ref{sec:nll}). With the kernel score as the loss function, \score{} achieves a lower CRPS than the diagonal and low-rank models on 23 of the 24 tasks and a lower variogram score on 22 and 21 of the 24 tasks. The CRPS gain comes from the first stage, which fits the marginals with the CRPS, and the variogram gain from the second stage, which adds the correlation along the horizon. ES and KS change by less than $1\%$ in the median, as they are less sensitive to the dependence structure. These results support the assumption of the circulant core, since the target is a window of $T$ consecutive values whose forecast errors persist and follow the cycles of the series, so their correlation depends on the lag between two lead times rather than on the position in the window. The generative sampler is outperformed by \score{} on CRPS and VS on every task, and \score{} runs at close to the cost of the diagonal model (Table~\ref{tab:runtime-ms-per-step}). On Weather and Electricity, the log score of every covariance trained with the kernel score, \score{} included, explodes, which the bounded kernel score does not penalize (Theorem~\ref{thm:ks-lipschitz}).

\begin{table*}[ht]
\centering
\caption{Test scores on ETTh1.}
\label{tab:etth1}
\small
\adjustbox{max width=\linewidth}{%
\begin{tabular}{llrrrrrrr}
\toprule
Method & Objective & MSE $\downarrow$ & CRPS $\downarrow$ & ES $\downarrow$ & VS $\downarrow$ & KS $\downarrow$ & NLL/$d$ $\downarrow$ & $\varnothing$ Rank $\downarrow$ \\
\midrule
\multicolumn{9}{l}{\textbf{Univariate, $T = 96$}}\\
\multicolumn{9}{l}{\quad\emph{Physical space}}\\
\qquad Det & MSE & $\mathbf{0.053 \pm 0.001}$ & $0.176 \pm 0.001$ & $2.041 \pm 0.010$ & $0.034 \pm 0.000$ & $0.050 \pm 0.001$ & - & $7.3$ \\
\qquad Diag & NLL & $0.053 \pm 0.000$ & $\mathbf{0.128 \pm 0.001}$ & $1.528 \pm 0.006$ & $0.031 \pm 0.001$ & $0.049 \pm 0.000$ & $+0.043 \pm 0.011$ & $5.2$ \\
\qquad Diag & KS & $\mathbf{0.052 \pm 0.000}$ & $\mathbf{0.127 \pm 0.000}$ & $1.502 \pm 0.006$ & $0.026 \pm 0.000$ & $\mathbf{0.048 \pm 0.000}$ & $+0.113 \pm 0.012$ & $3.8$ \\
\qquad LorD $r{=}1$ & NLL & $0.067 \pm 0.001$ & $0.148 \pm 0.002$ & $1.696 \pm 0.016$ & $0.025 \pm 0.000$ & $0.061 \pm 0.001$ & $-0.363 \pm 0.016$ & $6.2$ \\
\qquad LorD $r{=}1$ & KS & $\mathbf{0.052 \pm 0.000}$ & $0.128 \pm 0.000$ & $1.499 \pm 0.005$ & $0.024 \pm 0.000$ & $\mathbf{0.048 \pm 0.000}$ & $-0.262 \pm 0.015$ & $3.2$ \\
\qquad Chol & NLL & $0.071 \pm 0.000$ & $0.153 \pm 0.001$ & $1.742 \pm 0.009$ & $0.024 \pm 0.000$ & $0.064 \pm 0.000$ & $\mathbf{-1.146 \pm 0.013}$ & $6.3$ \\
\qquad Chol & KS & $\mathbf{0.053 \pm 0.001}$ & $0.129 \pm 0.001$ & $1.500 \pm 0.009$ & $0.024 \pm 0.000$ & $\mathbf{0.049 \pm 0.001}$ & - & $3.6$ \\
\qquad SB & KS & $0.055 \pm 0.002$ & $0.139 \pm 0.006$ & $1.654 \pm 0.053$ & $0.031 \pm 0.000$ & $0.051 \pm 0.002$ & - & $6.2$ \\
\multicolumn{9}{l}{\quad\emph{Spectral space}}\\
\qquad SCORE & CRPS+KS & $\mathbf{0.052 \pm 0.000}$ & $\mathbf{0.127 \pm 0.000}$ & $\mathbf{1.478 \pm 0.006}$ & $\mathbf{0.023 \pm 0.000}$ & $\mathbf{0.048 \pm 0.000}$ & $-1.009 \pm 0.006$ & $\mathbf{1.9}$ \\
\cmidrule(l){1-9}
\multicolumn{9}{l}{\textbf{Univariate, $T = 192$}}\\
\multicolumn{9}{l}{\quad\emph{Physical space}}\\
\qquad Det & MSE & $0.068 \pm 0.001$ & $0.202 \pm 0.001$ & $3.368 \pm 0.016$ & $0.050 \pm 0.001$ & $0.064 \pm 0.001$ & - & $7.7$ \\
\qquad Diag & NLL & $0.068 \pm 0.001$ & $0.147 \pm 0.001$ & $2.492 \pm 0.013$ & $0.035 \pm 0.001$ & $0.063 \pm 0.001$ & $+0.196 \pm 0.018$ & $5.1$ \\
\qquad Diag & KS & $\mathbf{0.065 \pm 0.000}$ & $\mathbf{0.144 \pm 0.001}$ & $2.436 \pm 0.008$ & $0.032 \pm 0.000$ & $\mathbf{0.060 \pm 0.000}$ & $+0.273 \pm 0.030$ & $3.0$ \\
\qquad LorD $r{=}1$ & NLL & $0.076 \pm 0.000$ & $0.160 \pm 0.000$ & $2.658 \pm 0.008$ & $0.034 \pm 0.000$ & $0.070 \pm 0.000$ & $-0.115 \pm 0.010$ & $6.5$ \\
\qquad LorD $r{=}1$ & KS & $\mathbf{0.065 \pm 0.000}$ & $0.146 \pm 0.000$ & $2.440 \pm 0.007$ & $0.032 \pm 0.000$ & $\mathbf{0.060 \pm 0.000}$ & $+0.022 \pm 0.037$ & $3.5$ \\
\qquad Chol & NLL & $0.078 \pm 0.000$ & $0.162 \pm 0.000$ & $2.676 \pm 0.006$ & $0.032 \pm 0.000$ & $0.071 \pm 0.000$ & $\mathbf{-1.100 \pm 0.010}$ & $6.2$ \\
\qquad Chol & KS & $\mathbf{0.065 \pm 0.001}$ & $\mathbf{0.145 \pm 0.001}$ & $\mathbf{2.419 \pm 0.013}$ & $0.033 \pm 0.000$ & $\mathbf{0.060 \pm 0.000}$ & - & $3.0$ \\
\qquad SB & KS & $0.069 \pm 0.001$ & $0.154 \pm 0.001$ & $2.664 \pm 0.018$ & $0.044 \pm 0.001$ & $0.063 \pm 0.001$ & - & $6.6$ \\
\multicolumn{9}{l}{\quad\emph{Spectral space}}\\
\qquad SCORE & CRPS+KS & $\mathbf{0.066 \pm 0.001}$ & $\mathbf{0.145 \pm 0.001}$ & $\mathbf{2.410 \pm 0.013}$ & $\mathbf{0.031 \pm 0.000}$ & $\mathbf{0.060 \pm 0.001}$ & $-1.035 \pm 0.037$ & $\mathbf{2.2}$ \\
\cmidrule(l){1-9}
\multicolumn{9}{l}{\textbf{Multivariate, $T = 96$}}\\
\multicolumn{9}{l}{\quad\emph{Physical space}}\\
\qquad Det & MSE & $0.377 \pm 0.002$ & $0.400 \pm 0.001$ & $15.577 \pm 0.043$ & $0.567 \pm 0.003$ & $0.305 \pm 0.001$ & - & $6.6$ \\
\qquad Diag & NLL & $\mathbf{0.370 \pm 0.003}$ & $0.286 \pm 0.001$ & $11.069 \pm 0.038$ & $0.507 \pm 0.004$ & $0.257 \pm 0.001$ & $+0.660 \pm 0.005$ & $3.0$ \\
\qquad Diag & KS & $\mathbf{0.369 \pm 0.002}$ & $0.298 \pm 0.001$ & $11.052 \pm 0.028$ & $0.529 \pm 0.002$ & $0.257 \pm 0.001$ & $+0.795 \pm 0.017$ & $3.8$ \\
\qquad LorD $r{=}1$ & NLL & $0.378 \pm 0.006$ & $0.287 \pm 0.002$ & $11.167 \pm 0.082$ & $0.514 \pm 0.007$ & $0.260 \pm 0.002$ & $+0.617 \pm 0.009$ & $3.5$ \\
\qquad LorD $r{=}1$ & KS & $0.369 \pm 0.001$ & $0.297 \pm 0.001$ & $11.041 \pm 0.026$ & $0.527 \pm 0.002$ & $\mathbf{0.256 \pm 0.001}$ & $+0.758 \pm 0.023$ & $3.7$ \\
\qquad SB & KS & $0.388 \pm 0.006$ & $0.332 \pm 0.018$ & $12.875 \pm 0.726$ & $0.567 \pm 0.006$ & $0.289 \pm 0.012$ & - & $6.2$ \\
\multicolumn{9}{l}{\quad\emph{Spectral space}}\\
\qquad SCORE & CRPS+KS & $\mathbf{0.366 \pm 0.001}$ & $\mathbf{0.280 \pm 0.001}$ & $\mathbf{10.971 \pm 0.019}$ & $\mathbf{0.495 \pm 0.002}$ & $\mathbf{0.254 \pm 0.001}$ & $\mathbf{+0.153 \pm 0.008}$ & $\mathbf{1.2}$ \\
\cmidrule(l){1-9}
\multicolumn{9}{l}{\textbf{Multivariate, $T = 192$}}\\
\multicolumn{9}{l}{\quad\emph{Physical space}}\\
\qquad Det & MSE & $0.415 \pm 0.001$ & $0.422 \pm 0.001$ & $23.251 \pm 0.024$ & $0.622 \pm 0.001$ & $0.332 \pm 0.001$ & - & $6.5$ \\
\qquad Diag & NLL & $0.415 \pm 0.003$ & $0.311 \pm 0.003$ & $16.662 \pm 0.063$ & $0.566 \pm 0.009$ & $0.281 \pm 0.001$ & $+0.790 \pm 0.016$ & $3.5$ \\
\qquad Diag & KS & $\mathbf{0.409 \pm 0.002}$ & $0.319 \pm 0.001$ & $16.544 \pm 0.052$ & $0.590 \pm 0.005$ & $0.278 \pm 0.001$ & $+0.876 \pm 0.005$ & $3.6$ \\
\qquad LorD $r{=}1$ & NLL & $0.427 \pm 0.002$ & $0.314 \pm 0.003$ & $16.847 \pm 0.024$ & $0.577 \pm 0.008$ & $0.284 \pm 0.001$ & $+0.749 \pm 0.024$ & $4.2$ \\
\qquad LorD $r{=}1$ & KS & $\mathbf{0.410 \pm 0.002}$ & $0.318 \pm 0.001$ & $16.538 \pm 0.044$ & $0.586 \pm 0.003$ & $0.278 \pm 0.001$ & $+0.855 \pm 0.008$ & $3.4$ \\
\qquad SB & KS & $\mathbf{0.455 \pm 0.062}$ & $0.344 \pm 0.008$ & $18.568 \pm 0.290$ & $\mathbf{0.794 \pm 0.308}$ & $0.306 \pm 0.004$ & - & $5.2$ \\
\multicolumn{9}{l}{\quad\emph{Spectral space}}\\
\qquad SCORE & CRPS+KS & $\mathbf{0.406 \pm 0.002}$ & $\mathbf{0.299 \pm 0.001}$ & $\mathbf{16.423 \pm 0.032}$ & $\mathbf{0.541 \pm 0.003}$ & $\mathbf{0.275 \pm 0.001}$ & $\mathbf{+0.290 \pm 0.027}$ & $\mathbf{1.3}$ \\
\bottomrule
\end{tabular}}
\end{table*}

\begin{table*}[ht]
\centering
\caption{Test scores on ETTh2.}
\label{tab:etth2}
\small
\adjustbox{max width=\linewidth}{%
\begin{tabular}{llrrrrrrr}
\toprule
Method & Objective & MSE $\downarrow$ & CRPS $\downarrow$ & ES $\downarrow$ & VS $\downarrow$ & KS $\downarrow$ & NLL/$d$ $\downarrow$ & $\varnothing$ Rank $\downarrow$ \\
\midrule
\multicolumn{9}{l}{\textbf{Univariate, $T = 96$}}\\
\multicolumn{9}{l}{\quad\emph{Physical space}}\\
\qquad Det & MSE & $0.125 \pm 0.002$ & $0.282 \pm 0.002$ & $3.295 \pm 0.021$ & $0.103 \pm 0.002$ & $0.115 \pm 0.001$ & - & $7.5$ \\
\qquad Diag & NLL & $\mathbf{0.120 \pm 0.000}$ & $\mathbf{0.194 \pm 0.000}$ & $2.354 \pm 0.003$ & $0.103 \pm 0.001$ & $\mathbf{0.105 \pm 0.000}$ & $+0.351 \pm 0.004$ & $3.7$ \\
\qquad Diag & KS & $0.122 \pm 0.001$ & $0.197 \pm 0.001$ & $2.362 \pm 0.010$ & $0.095 \pm 0.001$ & $\mathbf{0.106 \pm 0.001}$ & $+0.400 \pm 0.009$ & $4.3$ \\
\qquad LorD $r{=}1$ & NLL & $0.157 \pm 0.002$ & $0.227 \pm 0.003$ & $2.631 \pm 0.025$ & $0.091 \pm 0.000$ & $0.129 \pm 0.002$ & $+0.040 \pm 0.004$ & $6.1$ \\
\qquad LorD $r{=}1$ & KS & $0.124 \pm 0.002$ & $0.200 \pm 0.002$ & $2.352 \pm 0.018$ & $0.088 \pm 0.001$ & $\mathbf{0.106 \pm 0.002}$ & $+0.090 \pm 0.006$ & $3.7$ \\
\qquad Chol & NLL & $0.204 \pm 0.011$ & $0.268 \pm 0.009$ & $3.075 \pm 0.108$ & $0.106 \pm 0.003$ & $0.163 \pm 0.009$ & $\mathbf{-0.690 \pm 0.373}$ & $7.2$ \\
\qquad Chol & KS & $0.126 \pm 0.003$ & $0.200 \pm 0.003$ & $\mathbf{2.347 \pm 0.029}$ & $\mathbf{0.087 \pm 0.002}$ & $\mathbf{0.106 \pm 0.002}$ & $(+4.162 \pm 9.108) \times 10^{13}$ & $4.0$ \\
\qquad SB & KS & $0.127 \pm 0.003$ & $0.210 \pm 0.008$ & $2.541 \pm 0.084$ & $0.097 \pm 0.003$ & $\mathbf{0.108 \pm 0.003}$ & - & $5.1$ \\
\multicolumn{9}{l}{\quad\emph{Spectral space}}\\
\qquad SCORE & CRPS+KS & $0.123 \pm 0.001$ & $0.197 \pm 0.001$ & $\mathbf{2.321 \pm 0.010}$ & $\mathbf{0.085 \pm 0.001}$ & $\mathbf{0.104 \pm 0.001}$ & $\mathbf{-0.529 \pm 0.018}$ & $\mathbf{2.7}$ \\
\cmidrule(l){1-9}
\multicolumn{9}{l}{\textbf{Univariate, $T = 192$}}\\
\multicolumn{9}{l}{\quad\emph{Physical space}}\\
\qquad Det & MSE & $\mathbf{0.160 \pm 0.002}$ & $0.323 \pm 0.002$ & $5.356 \pm 0.031$ & $0.132 \pm 0.002$ & $0.145 \pm 0.002$ & - & $7.5$ \\
\qquad Diag & NLL & $\mathbf{0.157 \pm 0.002}$ & $\mathbf{0.224 \pm 0.002}$ & $3.840 \pm 0.022$ & $0.115 \pm 0.001$ & $\mathbf{0.133 \pm 0.001}$ & $+0.491 \pm 0.008$ & $3.8$ \\
\qquad Diag & KS & $\mathbf{0.159 \pm 0.002}$ & $\mathbf{0.228 \pm 0.002}$ & $3.859 \pm 0.029$ & $0.108 \pm 0.000$ & $\mathbf{0.134 \pm 0.002}$ & $+0.551 \pm 0.014$ & $4.0$ \\
\qquad LorD $r{=}1$ & NLL & $0.184 \pm 0.001$ & $0.251 \pm 0.001$ & $4.147 \pm 0.012$ & $0.105 \pm 0.000$ & $0.151 \pm 0.001$ & $+0.209 \pm 0.006$ & $5.9$ \\
\qquad LorD $r{=}1$ & KS & $0.161 \pm 0.001$ & $0.230 \pm 0.001$ & $3.840 \pm 0.017$ & $0.103 \pm 0.000$ & $0.134 \pm 0.001$ & $+0.225 \pm 0.009$ & $3.7$ \\
\qquad Chol & NLL & $0.241 \pm 0.001$ & $0.297 \pm 0.001$ & $4.875 \pm 0.013$ & $0.123 \pm 0.001$ & $0.192 \pm 0.001$ & $\mathbf{-0.428 \pm 0.015}$ & $7.1$ \\
\qquad Chol & KS & $\mathbf{0.161 \pm 0.002}$ & $0.229 \pm 0.002$ & $\mathbf{3.813 \pm 0.030}$ & $\mathbf{0.101 \pm 0.001}$ & $\mathbf{0.133 \pm 0.002}$ & $(+2.392 \pm 5.349) \times 10^{28}$ & $3.7$ \\
\qquad SB & KS & $\mathbf{0.160 \pm 0.004}$ & $0.240 \pm 0.007$ & $4.124 \pm 0.098$ & $0.126 \pm 0.004$ & $\mathbf{0.133 \pm 0.003}$ & - & $5.6$ \\
\multicolumn{9}{l}{\quad\emph{Spectral space}}\\
\qquad SCORE & CRPS+KS & $\mathbf{0.158 \pm 0.002}$ & $\mathbf{0.227 \pm 0.001}$ & $\mathbf{3.792 \pm 0.020}$ & $\mathbf{0.101 \pm 0.000}$ & $\mathbf{0.131 \pm 0.001}$ & $-0.099 \pm 0.387$ & $\mathbf{3.2}$ \\
\cmidrule(l){1-9}
\multicolumn{9}{l}{\textbf{Multivariate, $T = 96$}}\\
\multicolumn{9}{l}{\quad\emph{Physical space}}\\
\qquad Det & MSE & $0.284 \pm 0.001$ & $0.345 \pm 0.001$ & $12.645 \pm 0.023$ & $0.395 \pm 0.001$ & $0.226 \pm 0.001$ & - & $5.5$ \\
\qquad Diag & NLL & $0.416 \pm 0.035$ & $0.322 \pm 0.015$ & $11.411 \pm 0.558$ & $0.514 \pm 0.051$ & $0.265 \pm 0.017$ & $+0.732 \pm 0.017$ & $6.1$ \\
\qquad Diag & KS & $\mathbf{0.278 \pm 0.002}$ & $\mathbf{0.251 \pm 0.001}$ & $\mathbf{9.078 \pm 0.025}$ & $\mathbf{0.355 \pm 0.001}$ & $\mathbf{0.194 \pm 0.001}$ & $+0.712 \pm 0.031$ & $2.4$ \\
\qquad LorD $r{=}1$ & NLL & $0.361 \pm 0.005$ & $0.297 \pm 0.003$ & $10.520 \pm 0.091$ & $0.454 \pm 0.007$ & $0.237 \pm 0.003$ & $+0.578 \pm 0.013$ & $4.8$ \\
\qquad LorD $r{=}1$ & KS & $\mathbf{0.277 \pm 0.001}$ & $\mathbf{0.250 \pm 0.001}$ & $\mathbf{9.055 \pm 0.021}$ & $\mathbf{0.354 \pm 0.001}$ & $\mathbf{0.193 \pm 0.001}$ & $+0.642 \pm 0.043$ & $2.3$ \\
\qquad SB & KS & $\mathbf{0.284 \pm 0.006}$ & $0.309 \pm 0.019$ & $11.413 \pm 0.658$ & $0.388 \pm 0.008$ & $0.220 \pm 0.007$ & - & $4.5$ \\
\multicolumn{9}{l}{\quad\emph{Spectral space}}\\
\qquad SCORE & CRPS+KS & $\mathbf{0.280 \pm 0.002}$ & $\mathbf{0.250 \pm 0.000}$ & $\mathbf{9.084 \pm 0.024}$ & $\mathbf{0.356 \pm 0.001}$ & $\mathbf{0.194 \pm 0.001}$ & $\mathbf{+0.057 \pm 0.031}$ & $\mathbf{1.9}$ \\
\cmidrule(l){1-9}
\multicolumn{9}{l}{\textbf{Multivariate, $T = 192$}}\\
\multicolumn{9}{l}{\quad\emph{Physical space}}\\
\qquad Det & MSE & $\mathbf{0.349 \pm 0.010}$ & $0.386 \pm 0.004$ & $20.085 \pm 0.258$ & $0.483 \pm 0.010$ & $0.271 \pm 0.005$ & - & $5.3$ \\
\qquad Diag & NLL & $0.509 \pm 0.016$ & $0.360 \pm 0.006$ & $17.923 \pm 0.289$ & $0.580 \pm 0.018$ & $0.299 \pm 0.005$ & $+0.840 \pm 0.015$ & $6.1$ \\
\qquad Diag & KS & $\mathbf{0.344 \pm 0.002}$ & $\mathbf{0.284 \pm 0.001}$ & $\mathbf{14.526 \pm 0.049}$ & $\mathbf{0.435 \pm 0.001}$ & $\mathbf{0.231 \pm 0.001}$ & $+0.843 \pm 0.039$ & $2.6$ \\
\qquad LorD $r{=}1$ & NLL & $0.391 \pm 0.002$ & $0.314 \pm 0.002$ & $15.651 \pm 0.085$ & $0.495 \pm 0.009$ & $0.255 \pm 0.003$ & $+0.661 \pm 0.010$ & $4.3$ \\
\qquad LorD $r{=}1$ & KS & $\mathbf{0.343 \pm 0.002}$ & $\mathbf{0.283 \pm 0.001}$ & $\mathbf{14.523 \pm 0.056}$ & $\mathbf{0.435 \pm 0.002}$ & $\mathbf{0.231 \pm 0.001}$ & $+0.730 \pm 0.026$ & $2.3$ \\
\qquad SB & KS & $\mathbf{0.352 \pm 0.008}$ & $0.348 \pm 0.040$ & $18.315 \pm 2.077$ & $0.478 \pm 0.020$ & $0.264 \pm 0.016$ & - & $4.9$ \\
\multicolumn{9}{l}{\quad\emph{Spectral space}}\\
\qquad SCORE & CRPS+KS & $\mathbf{0.347 \pm 0.002}$ & $\mathbf{0.284 \pm 0.000}$ & $\mathbf{14.583 \pm 0.034}$ & $\mathbf{0.438 \pm 0.002}$ & $\mathbf{0.232 \pm 0.001}$ & $\mathbf{+0.107 \pm 0.018}$ & $\mathbf{2.0}$ \\
\bottomrule
\end{tabular}}
\end{table*}

\begin{table*}[ht]
\centering
\caption{Test scores on ETTm1.}
\label{tab:ettm1}
\small
\adjustbox{max width=\linewidth}{%
\begin{tabular}{llrrrrrrr}
\toprule
Method & Objective & MSE $\downarrow$ & CRPS $\downarrow$ & ES $\downarrow$ & VS $\downarrow$ & KS $\downarrow$ & NLL/$d$ $\downarrow$ & $\varnothing$ Rank $\downarrow$ \\
\midrule
\multicolumn{9}{l}{\textbf{Univariate, $T = 96$}}\\
\multicolumn{9}{l}{\quad\emph{Physical space}}\\
\qquad Det & MSE & $\mathbf{0.026 \pm 0.000}$ & $0.122 \pm 0.001$ & $1.409 \pm 0.004$ & $0.015 \pm 0.000$ & $\mathbf{0.025 \pm 0.000}$ & - & $5.1$ \\
\qquad Diag & NLL & $\mathbf{0.026 \pm 0.000}$ & $\mathbf{0.087 \pm 0.000}$ & $1.048 \pm 0.002$ & $0.017 \pm 0.001$ & $\mathbf{0.025 \pm 0.000}$ & $-0.390 \pm 0.007$ & $3.6$ \\
\qquad Diag & KS & $\mathbf{0.026 \pm 0.000}$ & $0.089 \pm 0.000$ & $1.060 \pm 0.003$ & $0.013 \pm 0.000$ & $\mathbf{0.025 \pm 0.000}$ & $-0.196 \pm 0.039$ & $3.9$ \\
\qquad LorD $r{=}1$ & NLL & $0.029 \pm 0.001$ & $0.093 \pm 0.002$ & $1.073 \pm 0.015$ & $0.012 \pm 0.000$ & $0.027 \pm 0.001$ & $-0.920 \pm 0.007$ & $4.8$ \\
\qquad LorD $r{=}1$ & KS & $\mathbf{0.026 \pm 0.000}$ & $0.090 \pm 0.000$ & $1.042 \pm 0.005$ & $\mathbf{0.011 \pm 0.000}$ & $\mathbf{0.025 \pm 0.000}$ & $-0.663 \pm 0.044$ & $3.0$ \\
\qquad SB & KS & $\mathbf{0.026 \pm 0.000}$ & $0.121 \pm 0.000$ & $1.408 \pm 0.004$ & $0.015 \pm 0.000$ & $\mathbf{0.025 \pm 0.000}$ & - & $5.1$ \\
\multicolumn{9}{l}{\quad\emph{Spectral space}}\\
\qquad SCORE & CRPS+KS & $\mathbf{0.026 \pm 0.000}$ & $0.088 \pm 0.000$ & $\mathbf{1.024 \pm 0.002}$ & $\mathbf{0.011 \pm 0.000}$ & $\mathbf{0.025 \pm 0.000}$ & $\mathbf{-1.587 \pm 0.031}$ & $\mathbf{2.1}$ \\
\cmidrule(l){1-9}
\multicolumn{9}{l}{\textbf{Univariate, $T = 192$}}\\
\multicolumn{9}{l}{\quad\emph{Physical space}}\\
\qquad Det & MSE & $\mathbf{0.039 \pm 0.000}$ & $0.151 \pm 0.001$ & $2.477 \pm 0.011$ & $0.023 \pm 0.000$ & $0.038 \pm 0.000$ & - & $5.7$ \\
\qquad Diag & NLL & $\mathbf{0.039 \pm 0.000}$ & $\mathbf{0.109 \pm 0.000}$ & $1.848 \pm 0.003$ & $0.026 \pm 0.000$ & $0.037 \pm 0.000$ & $-0.122 \pm 0.006$ & $3.8$ \\
\qquad Diag & KS & $\mathbf{0.039 \pm 0.000}$ & $0.111 \pm 0.000$ & $1.843 \pm 0.007$ & $0.019 \pm 0.000$ & $0.037 \pm 0.000$ & $+0.159 \pm 0.009$ & $3.8$ \\
\qquad LorD $r{=}1$ & NLL & $0.042 \pm 0.001$ & $0.115 \pm 0.003$ & $1.878 \pm 0.053$ & $\mathbf{0.017 \pm 0.002}$ & $0.040 \pm 0.002$ & $-0.652 \pm 0.027$ & $4.2$ \\
\qquad LorD $r{=}1$ & KS & $\mathbf{0.039 \pm 0.000}$ & $0.111 \pm 0.001$ & $1.819 \pm 0.013$ & $0.017 \pm 0.000$ & $0.037 \pm 0.000$ & $-0.419 \pm 0.057$ & $3.1$ \\
\qquad SB & KS & $\mathbf{0.039 \pm 0.000}$ & $0.150 \pm 0.001$ & $2.456 \pm 0.009$ & $0.023 \pm 0.000$ & $0.038 \pm 0.000$ & - & $5.5$ \\
\multicolumn{9}{l}{\quad\emph{Spectral space}}\\
\qquad SCORE & CRPS+KS & $\mathbf{0.039 \pm 0.000}$ & $\mathbf{0.109 \pm 0.000}$ & $\mathbf{1.793 \pm 0.006}$ & $\mathbf{0.016 \pm 0.000}$ & $\mathbf{0.036 \pm 0.000}$ & $\mathbf{-1.458 \pm 0.022}$ & $\mathbf{1.6}$ \\
\cmidrule(l){1-9}
\multicolumn{9}{l}{\textbf{Multivariate, $T = 96$}}\\
\multicolumn{9}{l}{\quad\emph{Physical space}}\\
\qquad Det & MSE & $0.290 \pm 0.001$ & $0.343 \pm 0.001$ & $13.276 \pm 0.018$ & $0.408 \pm 0.003$ & $0.238 \pm 0.001$ & - & $6.4$ \\
\qquad Diag & NLL & $\mathbf{0.281 \pm 0.007}$ & $\mathbf{0.243 \pm 0.004}$ & $\mathbf{9.490 \pm 0.130}$ & $\mathbf{0.364 \pm 0.007}$ & $\mathbf{0.206 \pm 0.004}$ & $+0.664 \pm 0.129$ & $2.8$ \\
\qquad Diag & KS & $\mathbf{0.289 \pm 0.004}$ & $0.254 \pm 0.002$ & $9.577 \pm 0.053$ & $0.390 \pm 0.004$ & $0.208 \pm 0.001$ & $+0.756 \pm 0.054$ & $4.4$ \\
\qquad LorD $r{=}1$ & NLL & $\mathbf{0.285 \pm 0.009}$ & $0.245 \pm 0.004$ & $\mathbf{9.572 \pm 0.131}$ & $\mathbf{0.372 \pm 0.005}$ & $0.210 \pm 0.004$ & $+0.515 \pm 0.096$ & $3.2$ \\
\qquad LorD $r{=}1$ & KS & $\mathbf{0.287 \pm 0.002}$ & $0.249 \pm 0.001$ & $\mathbf{9.465 \pm 0.020}$ & $0.378 \pm 0.002$ & $\mathbf{0.204 \pm 0.001}$ & $+0.805 \pm 0.055$ & $3.2$ \\
\qquad SB & KS & $0.311 \pm 0.013$ & $0.260 \pm 0.005$ & $9.963 \pm 0.231$ & $0.410 \pm 0.014$ & $0.214 \pm 0.005$ & - & $5.9$ \\
\multicolumn{9}{l}{\quad\emph{Spectral space}}\\
\qquad SCORE & CRPS+KS & $\mathbf{0.283 \pm 0.003}$ & $\mathbf{0.239 \pm 0.001}$ & $\mathbf{9.429 \pm 0.045}$ & $\mathbf{0.368 \pm 0.004}$ & $\mathbf{0.203 \pm 0.001}$ & $\mathbf{-0.231 \pm 0.080}$ & $\mathbf{2.0}$ \\
\cmidrule(l){1-9}
\multicolumn{9}{l}{\textbf{Multivariate, $T = 192$}}\\
\multicolumn{9}{l}{\quad\emph{Physical space}}\\
\qquad Det & MSE & $\mathbf{0.331 \pm 0.004}$ & $0.369 \pm 0.001$ & $20.393 \pm 0.106$ & $0.477 \pm 0.003$ & $0.271 \pm 0.002$ & - & $6.2$ \\
\qquad Diag & NLL & $0.335 \pm 0.002$ & $0.272 \pm 0.002$ & $14.923 \pm 0.066$ & $0.434 \pm 0.002$ & $0.241 \pm 0.001$ & $+0.964 \pm 0.069$ & $4.0$ \\
\qquad Diag & KS & $\mathbf{0.330 \pm 0.003}$ & $0.277 \pm 0.001$ & $14.676 \pm 0.061$ & $0.456 \pm 0.003$ & $0.235 \pm 0.001$ & $+0.872 \pm 0.041$ & $3.8$ \\
\qquad LorD $r{=}1$ & NLL & $\mathbf{0.329 \pm 0.009}$ & $0.264 \pm 0.003$ & $\mathbf{14.689 \pm 0.199}$ & $\mathbf{0.433 \pm 0.019}$ & $\mathbf{0.236 \pm 0.005}$ & $+0.673 \pm 0.124$ & $2.8$ \\
\qquad LorD $r{=}1$ & KS & $\mathbf{0.330 \pm 0.005}$ & $0.275 \pm 0.001$ & $\mathbf{14.641 \pm 0.084}$ & $0.455 \pm 0.005$ & $0.234 \pm 0.002$ & $+0.874 \pm 0.020$ & $3.5$ \\
\qquad SB & KS & $0.361 \pm 0.008$ & $0.284 \pm 0.003$ & $15.169 \pm 0.161$ & $0.516 \pm 0.038$ & $0.240 \pm 0.004$ & - & $5.9$ \\
\multicolumn{9}{l}{\quad\emph{Spectral space}}\\
\qquad SCORE & CRPS+KS & $\mathbf{0.325 \pm 0.003}$ & $\mathbf{0.259 \pm 0.001}$ & $\mathbf{14.503 \pm 0.062}$ & $\mathbf{0.425 \pm 0.006}$ & $\mathbf{0.230 \pm 0.001}$ & $\mathbf{-0.182 \pm 0.038}$ & $\mathbf{1.7}$ \\
\bottomrule
\end{tabular}}
\end{table*}

\begin{table*}[ht]
\centering
\caption{Test scores on ETTm2.}
\label{tab:ettm2}
\small
\adjustbox{max width=\linewidth}{%
\begin{tabular}{llrrrrrrr}
\toprule
Method & Objective & MSE $\downarrow$ & CRPS $\downarrow$ & ES $\downarrow$ & VS $\downarrow$ & KS $\downarrow$ & NLL/$d$ $\downarrow$ & $\varnothing$ Rank $\downarrow$ \\
\midrule
\multicolumn{9}{l}{\textbf{Univariate, $T = 96$}}\\
\multicolumn{9}{l}{\quad\emph{Physical space}}\\
\qquad Det & MSE & $0.065 \pm 0.000$ & $0.187 \pm 0.001$ & $2.194 \pm 0.008$ & $0.053 \pm 0.001$ & $0.061 \pm 0.000$ & - & $5.4$ \\
\qquad Diag & NLL & $\mathbf{0.064 \pm 0.001}$ & $\mathbf{0.133 \pm 0.002}$ & $1.676 \pm 0.022$ & $0.071 \pm 0.004$ & $\mathbf{0.058 \pm 0.001}$ & $-0.081 \pm 0.026$ & $3.6$ \\
\qquad Diag & KS & $\mathbf{0.064 \pm 0.000}$ & $0.138 \pm 0.000$ & $1.664 \pm 0.003$ & $0.059 \pm 0.000$ & $0.058 \pm 0.000$ & $+0.240 \pm 0.018$ & $3.9$ \\
\qquad LorD $r{=}1$ & NLL & $0.068 \pm 0.001$ & $0.147 \pm 0.002$ & $1.757 \pm 0.016$ & $0.058 \pm 0.001$ & $0.068 \pm 0.002$ & $-0.627 \pm 0.033$ & $5.2$ \\
\qquad LorD $r{=}1$ & KS & $\mathbf{0.064 \pm 0.001}$ & $0.137 \pm 0.001$ & $\mathbf{1.610 \pm 0.011}$ & $\mathbf{0.051 \pm 0.001}$ & $\mathbf{0.057 \pm 0.000}$ & $+0.119 \pm 0.105$ & $2.9$ \\
\qquad SB & KS & $\mathbf{0.065 \pm 0.001}$ & $0.184 \pm 0.001$ & $2.158 \pm 0.009$ & $\mathbf{0.052 \pm 0.001}$ & $0.060 \pm 0.001$ & - & $4.6$ \\
\multicolumn{9}{l}{\quad\emph{Spectral space}}\\
\qquad SCORE & CRPS+KS & $\mathbf{0.063 \pm 0.001}$ & $\mathbf{0.134 \pm 0.001}$ & $\mathbf{1.589 \pm 0.015}$ & $\mathbf{0.050 \pm 0.001}$ & $\mathbf{0.056 \pm 0.001}$ & $\mathbf{-1.526 \pm 0.066}$ & $\mathbf{1.8}$ \\
\cmidrule(l){1-9}
\multicolumn{9}{l}{\textbf{Univariate, $T = 192$}}\\
\multicolumn{9}{l}{\quad\emph{Physical space}}\\
\qquad Det & MSE & $0.095 \pm 0.001$ & $0.232 \pm 0.002$ & $3.866 \pm 0.015$ & $0.076 \pm 0.001$ & $0.087 \pm 0.001$ & - & $5.4$ \\
\qquad Diag & NLL & $\mathbf{0.094 \pm 0.001}$ & $\mathbf{0.169 \pm 0.003}$ & $2.951 \pm 0.047$ & $0.099 \pm 0.008$ & $\mathbf{0.084 \pm 0.002}$ & $+0.232 \pm 0.035$ & $4.2$ \\
\qquad Diag & KS & $\mathbf{0.093 \pm 0.001}$ & $0.171 \pm 0.002$ & $2.899 \pm 0.020$ & $0.080 \pm 0.002$ & $\mathbf{0.082 \pm 0.001}$ & $+0.460 \pm 0.036$ & $4.0$ \\
\qquad LorD $r{=}1$ & NLL & $0.097 \pm 0.003$ & $0.175 \pm 0.002$ & $2.957 \pm 0.037$ & $0.078 \pm 0.002$ & $0.088 \pm 0.002$ & $-0.127 \pm 0.047$ & $4.2$ \\
\qquad LorD $r{=}1$ & KS & $\mathbf{0.093 \pm 0.001}$ & $\mathbf{0.170 \pm 0.002}$ & $\mathbf{2.839 \pm 0.024}$ & $\mathbf{0.070 \pm 0.002}$ & $\mathbf{0.081 \pm 0.001}$ & $+0.226 \pm 0.099$ & $2.6$ \\
\qquad SB & KS & $\mathbf{0.094 \pm 0.001}$ & $0.230 \pm 0.003$ & $3.822 \pm 0.033$ & $0.075 \pm 0.001$ & $0.086 \pm 0.001$ & - & $5.3$ \\
\multicolumn{9}{l}{\quad\emph{Spectral space}}\\
\qquad SCORE & CRPS+KS & $\mathbf{0.093 \pm 0.000}$ & $\mathbf{0.167 \pm 0.001}$ & $\mathbf{2.812 \pm 0.018}$ & $\mathbf{0.067 \pm 0.001}$ & $\mathbf{0.081 \pm 0.001}$ & $\mathbf{-1.240 \pm 0.097}$ & $\mathbf{1.9}$ \\
\cmidrule(l){1-9}
\multicolumn{9}{l}{\textbf{Multivariate, $T = 96$}}\\
\multicolumn{9}{l}{\quad\emph{Physical space}}\\
\qquad Det & MSE & $0.168 \pm 0.003$ & $0.257 \pm 0.003$ & $9.571 \pm 0.092$ & $0.233 \pm 0.003$ & $0.142 \pm 0.002$ & - & $7.0$ \\
\qquad Diag & NLL & $\mathbf{0.160 \pm 0.001}$ & $0.190 \pm 0.001$ & $6.915 \pm 0.029$ & $0.223 \pm 0.002$ & $\mathbf{0.126 \pm 0.001}$ & $+0.576 \pm 0.043$ & $4.0$ \\
\qquad Diag & KS & $\mathbf{0.159 \pm 0.002}$ & $\mathbf{0.184 \pm 0.001}$ & $\mathbf{6.767 \pm 0.035}$ & $\mathbf{0.212 \pm 0.001}$ & $\mathbf{0.124 \pm 0.001}$ & $+43.003 \pm 0.014$ & $\mathbf{2.6}$ \\
\qquad LorD $r{=}1$ & NLL & $\mathbf{0.160 \pm 0.001}$ & $0.189 \pm 0.001$ & $6.839 \pm 0.013$ & $0.218 \pm 0.002$ & $\mathbf{0.125 \pm 0.000}$ & $\mathbf{+0.472 \pm 0.042}$ & $3.3$ \\
\qquad LorD $r{=}1$ & KS & $\mathbf{0.160 \pm 0.002}$ & $\mathbf{0.184 \pm 0.001}$ & $\mathbf{6.784 \pm 0.037}$ & $\mathbf{0.212 \pm 0.002}$ & $\mathbf{0.125 \pm 0.001}$ & $+42.871 \pm 0.123$ & $2.7$ \\
\qquad SB & KS & $0.163 \pm 0.001$ & $0.191 \pm 0.002$ & $7.417 \pm 0.028$ & $0.220 \pm 0.001$ & $0.128 \pm 0.001$ & - & $5.3$ \\
\multicolumn{9}{l}{\quad\emph{Spectral space}}\\
\qquad SCORE & CRPS+KS & $\mathbf{0.161 \pm 0.001}$ & $\mathbf{0.183 \pm 0.001}$ & $\mathbf{6.800 \pm 0.030}$ & $\mathbf{0.212 \pm 0.001}$ & $\mathbf{0.125 \pm 0.001}$ & $(+2.349 \pm 4.820) \times 10^{7}$ & $3.0$ \\
\cmidrule(l){1-9}
\multicolumn{9}{l}{\textbf{Multivariate, $T = 192$}}\\
\multicolumn{9}{l}{\quad\emph{Physical space}}\\
\qquad Det & MSE & $0.223 \pm 0.004$ & $0.295 \pm 0.003$ & $15.515 \pm 0.141$ & $0.305 \pm 0.004$ & $0.180 \pm 0.003$ & - & $6.6$ \\
\qquad Diag & NLL & $0.220 \pm 0.001$ & $0.223 \pm 0.001$ & $11.450 \pm 0.080$ & $0.299 \pm 0.007$ & $0.162 \pm 0.002$ & $+1.157 \pm 0.230$ & $4.8$ \\
\qquad Diag & KS & $0.214 \pm 0.002$ & $0.214 \pm 0.001$ & $11.129 \pm 0.078$ & $0.276 \pm 0.002$ & $0.157 \pm 0.002$ & $+67.366 \pm 0.028$ & $2.9$ \\
\qquad LorD $r{=}1$ & NLL & $0.216 \pm 0.002$ & $0.222 \pm 0.002$ & $11.347 \pm 0.109$ & $0.298 \pm 0.008$ & $0.161 \pm 0.002$ & $\mathbf{+0.717 \pm 0.035}$ & $4.1$ \\
\qquad LorD $r{=}1$ & KS & $\mathbf{0.210 \pm 0.000}$ & $\mathbf{0.212 \pm 0.000}$ & $\mathbf{10.969 \pm 0.013}$ & $\mathbf{0.272 \pm 0.001}$ & $\mathbf{0.153 \pm 0.000}$ & $+67.270 \pm 0.051$ & $\mathbf{1.5}$ \\
\qquad SB & KS & $0.221 \pm 0.003$ & $0.220 \pm 0.001$ & $12.021 \pm 0.037$ & $0.289 \pm 0.003$ & $0.161 \pm 0.001$ & - & $5.2$ \\
\multicolumn{9}{l}{\quad\emph{Spectral space}}\\
\qquad SCORE & CRPS+KS & $0.212 \pm 0.001$ & $\mathbf{0.212 \pm 0.001}$ & $11.037 \pm 0.038$ & $\mathbf{0.273 \pm 0.001}$ & $0.155 \pm 0.001$ & $(+8.006 \pm 13.424) \times 10^{6}$ & $2.8$ \\
\bottomrule
\end{tabular}}
\end{table*}

\begin{table*}[ht]
\centering
\caption{Test scores on National illness (multivariate, 7 series).}
\label{tab:illness}
\small
\adjustbox{max width=\linewidth}{%
\begin{tabular}{llrrrrrrr}
\toprule
Method & Objective & MSE $\downarrow$ & CRPS $\downarrow$ & ES $\downarrow$ & VS $\downarrow$ & KS $\downarrow$ & NLL/$d$ $\downarrow$ & $\varnothing$ Rank $\downarrow$ \\
\midrule
\multicolumn{9}{l}{\textbf{$T = 24$}}\\
\multicolumn{9}{l}{\quad\emph{Physical space}}\\
\qquad Det & MSE & $2.078 \pm 0.128$ & $0.932 \pm 0.051$ & $16.523 \pm 0.817$ & $2.430 \pm 0.236$ & $0.668 \pm 0.034$ & - & $6.3$ \\
\qquad Diag & NLL & $\mathbf{1.752 \pm 0.039}$ & $\mathbf{0.631 \pm 0.025}$ & $\mathbf{11.384 \pm 0.338}$ & $\mathbf{1.789 \pm 0.040}$ & $\mathbf{0.445 \pm 0.021}$ & $+3.526 \pm 0.604$ & $\mathbf{1.8}$ \\
\qquad Diag & KS & $2.177 \pm 0.155$ & $0.748 \pm 0.058$ & $12.846 \pm 0.762$ & $2.161 \pm 0.145$ & $\mathbf{0.461 \pm 0.021}$ & $\mathbf{+2.409 \pm 0.290}$ & $3.8$ \\
\qquad LorD $r{=}1$ & NLL & $2.136 \pm 0.103$ & $0.735 \pm 0.015$ & $12.828 \pm 0.182$ & $2.200 \pm 0.068$ & $\mathbf{0.461 \pm 0.005}$ & $\mathbf{+2.568 \pm 0.550}$ & $4.0$ \\
\qquad LorD $r{=}1$ & KS & $2.066 \pm 0.040$ & $0.721 \pm 0.017$ & $12.563 \pm 0.245$ & $2.055 \pm 0.078$ & $\mathbf{0.466 \pm 0.013}$ & $\mathbf{+2.498 \pm 0.248}$ & $3.7$ \\
\qquad SB & KS & $1.974 \pm 0.092$ & $0.711 \pm 0.037$ & $12.782 \pm 0.495$ & $2.024 \pm 0.089$ & $0.507 \pm 0.023$ & - & $4.3$ \\
\multicolumn{9}{l}{\quad\emph{Spectral space}}\\
\qquad SCORE & CRPS+KS & $1.962 \pm 0.081$ & $0.696 \pm 0.031$ & $12.369 \pm 0.340$ & $1.964 \pm 0.093$ & $\mathbf{0.475 \pm 0.022}$ & $\mathbf{+2.948 \pm 2.796}$ & $3.8$ \\
\cmidrule(l){1-9}
\multicolumn{9}{l}{\textbf{$T = 36$}}\\
\multicolumn{9}{l}{\quad\emph{Physical space}}\\
\qquad Det & MSE & $\mathbf{2.011 \pm 0.193}$ & $0.948 \pm 0.053$ & $20.728 \pm 1.291$ & $2.698 \pm 0.276$ & $0.723 \pm 0.041$ & - & $5.8$ \\
\qquad Diag & NLL & $\mathbf{1.701 \pm 0.135}$ & $\mathbf{0.633 \pm 0.031}$ & $\mathbf{14.232 \pm 0.569}$ & $\mathbf{2.071 \pm 0.251}$ & $\mathbf{0.477 \pm 0.015}$ & $+3.214 \pm 0.965$ & $\mathbf{2.7}$ \\
\qquad Diag & KS & $2.136 \pm 0.120$ & $0.766 \pm 0.035$ & $16.459 \pm 0.592$ & $2.468 \pm 0.183$ & $0.510 \pm 0.017$ & $+2.392 \pm 0.218$ & $4.7$ \\
\qquad LorD $r{=}1$ & NLL & $\mathbf{1.897 \pm 0.143}$ & $0.708 \pm 0.015$ & $15.267 \pm 0.345$ & $2.392 \pm 0.093$ & $\mathbf{0.469 \pm 0.006}$ & $+2.708 \pm 0.609$ & $3.4$ \\
\qquad LorD $r{=}1$ & KS & $1.954 \pm 0.035$ & $0.718 \pm 0.021$ & $15.603 \pm 0.324$ & $2.296 \pm 0.067$ & $0.502 \pm 0.019$ & $+2.344 \pm 0.331$ & $4.0$ \\
\qquad SB & KS & $\mathbf{1.985 \pm 0.332}$ & $\mathbf{0.839 \pm 0.235}$ & $\mathbf{18.276 \pm 4.376}$ & $\mathbf{2.410 \pm 0.440}$ & $0.636 \pm 0.151$ & - & $4.1$ \\
\multicolumn{9}{l}{\quad\emph{Spectral space}}\\
\qquad SCORE & CRPS+KS & $\mathbf{1.821 \pm 0.062}$ & $\mathbf{0.675 \pm 0.033}$ & $15.285 \pm 0.419$ & $\mathbf{2.170 \pm 0.057}$ & $0.515 \pm 0.019$ & $\mathbf{+1.072 \pm 0.188}$ & $2.8$ \\
\bottomrule
\end{tabular}}
\end{table*}

\begin{table*}[ht]
\centering
\caption{Test scores on Weather (multivariate, 21 series).}
\label{tab:weather}
\small
\adjustbox{max width=\linewidth}{%
\begin{tabular}{llrrrrrrr}
\toprule
Method & Objective & MSE $\downarrow$ & CRPS $\downarrow$ & ES $\downarrow$ & VS $\downarrow$ & KS $\downarrow$ & NLL/$d$ $\downarrow$ & $\varnothing$ Rank $\downarrow$ \\
\midrule
\multicolumn{9}{l}{\textbf{$T = 96$}}\\
\multicolumn{9}{l}{\quad\emph{Physical space}}\\
\qquad Det & MSE & $\mathbf{0.147 \pm 0.001}$ & $0.196 \pm 0.001$ & $15.990 \pm 0.036$ & $0.227 \pm 0.001$ & $0.130 \pm 0.000$ & - & $5.3$ \\
\qquad Diag & NLL & $0.189 \pm 0.014$ & $0.182 \pm 0.003$ & $13.428 \pm 0.276$ & $0.509 \pm 0.106$ & $0.148 \pm 0.004$ & $\mathbf{+3.731 \pm 1.901}$ & $5.2$ \\
\qquad Diag & KS & $\mathbf{0.146 \pm 0.001}$ & $0.158 \pm 0.001$ & $\mathbf{11.686 \pm 0.040}$ & $0.221 \pm 0.002$ & $\mathbf{0.121 \pm 0.001}$ & $+243.613 \pm 9.678$ & $3.0$ \\
\qquad LorD $r{=}1$ & NLL & $0.182 \pm 0.008$ & $0.183 \pm 0.005$ & $13.524 \pm 0.368$ & $0.527 \pm 0.100$ & $0.152 \pm 0.007$ & $\mathbf{+3.601 \pm 1.687}$ & $5.2$ \\
\qquad LorD $r{=}1$ & KS & $\mathbf{0.145 \pm 0.001}$ & $0.155 \pm 0.000$ & $\mathbf{11.701 \pm 0.028}$ & $0.220 \pm 0.001$ & $\mathbf{0.121 \pm 0.001}$ & $+242.577 \pm 10.717$ & $2.8$ \\
\qquad SB & KS & $\mathbf{0.179 \pm 0.050}$ & $0.156 \pm 0.005$ & $\mathbf{12.239 \pm 0.811}$ & $\mathbf{0.420 \pm 0.259}$ & $\mathbf{0.123 \pm 0.002}$ & - & $3.2$ \\
\multicolumn{9}{l}{\quad\emph{Spectral space}}\\
\qquad SCORE & CRPS+KS & $\mathbf{0.146 \pm 0.000}$ & $\mathbf{0.138 \pm 0.000}$ & $\mathbf{11.681 \pm 0.027}$ & $\mathbf{0.207 \pm 0.001}$ & $\mathbf{0.121 \pm 0.000}$ & $+311.086 \pm 20.259$ & $\mathbf{2.6}$ \\
\cmidrule(l){1-9}
\multicolumn{9}{l}{\textbf{$T = 192$}}\\
\multicolumn{9}{l}{\quad\emph{Physical space}}\\
\qquad Det & MSE & $0.192 \pm 0.001$ & $0.240 \pm 0.002$ & $26.367 \pm 0.074$ & $0.285 \pm 0.001$ & $0.168 \pm 0.001$ & - & $5.4$ \\
\qquad Diag & NLL & $0.238 \pm 0.030$ & $0.216 \pm 0.004$ & $21.435 \pm 0.828$ & $0.538 \pm 0.207$ & $0.178 \pm 0.006$ & $\mathbf{+4.318 \pm 0.434}$ & $5.2$ \\
\qquad Diag & KS & $0.192 \pm 0.000$ & $0.193 \pm 0.001$ & $\mathbf{19.239 \pm 0.039}$ & $0.274 \pm 0.002$ & $0.155 \pm 0.000$ & $+290.302 \pm 11.746$ & $3.3$ \\
\qquad LorD $r{=}1$ & NLL & $0.223 \pm 0.005$ & $0.214 \pm 0.001$ & $21.239 \pm 0.264$ & $0.524 \pm 0.078$ & $0.179 \pm 0.003$ & $\mathbf{+3.349 \pm 0.755}$ & $5.2$ \\
\qquad LorD $r{=}1$ & KS & $0.191 \pm 0.001$ & $0.190 \pm 0.001$ & $\mathbf{19.187 \pm 0.036}$ & $0.272 \pm 0.003$ & $\mathbf{0.154 \pm 0.000}$ & $+300.584 \pm 15.108$ & $2.8$ \\
\qquad SB & KS & $\mathbf{0.204 \pm 0.018}$ & $0.188 \pm 0.002$ & $\mathbf{19.439 \pm 0.249}$ & $\mathbf{0.369 \pm 0.169}$ & $\mathbf{0.155 \pm 0.001}$ & - & $3.1$ \\
\multicolumn{9}{l}{\quad\emph{Spectral space}}\\
\qquad SCORE & CRPS+KS & $\mathbf{0.189 \pm 0.000}$ & $\mathbf{0.171 \pm 0.000}$ & $\mathbf{19.237 \pm 0.038}$ & $\mathbf{0.257 \pm 0.001}$ & $\mathbf{0.154 \pm 0.000}$ & $+366.681 \pm 32.371$ & $\mathbf{2.2}$ \\
\bottomrule
\end{tabular}}
\end{table*}

\begin{table*}[ht]
\centering
\caption{Test scores on Electricity (multivariate, 321 series).}
\label{tab:electricity}
\small
\adjustbox{max width=\linewidth}{%
\begin{tabular}{llrrrrrrr}
\toprule
Method & Objective & MSE $\downarrow$ & CRPS $\downarrow$ & ES $\downarrow$ & VS $\downarrow$ & KS $\downarrow$ & NLL/$d$ $\downarrow$ & $\varnothing$ Rank $\downarrow$ \\
\midrule
\multicolumn{9}{l}{\textbf{$T = 96$}}\\
\multicolumn{9}{l}{\quad\emph{Physical space}}\\
\qquad Det & MSE & $0.130 \pm 0.000$ & $0.224 \pm 0.000$ & $61.609 \pm 0.064$ & $0.196 \pm 0.000$ & $0.119 \pm 0.000$ & - & $4.7$ \\
\qquad Diag & NLL & $0.133 \pm 0.001$ & $0.174 \pm 0.002$ & $45.995 \pm 1.015$ & $0.245 \pm 0.035$ & $0.120 \pm 0.005$ & $\mathbf{+0.587 \pm 0.039}$ & $4.1$ \\
\qquad Diag & KS & $0.130 \pm 0.001$ & $0.184 \pm 0.004$ & $\mathbf{44.000 \pm 0.127}$ & $0.203 \pm 0.005$ & $\mathbf{0.112 \pm 0.001}$ & $+228.388 \pm 156.785$ & $3.3$ \\
\qquad LorD $r{=}1$ & NLL & $0.363 \pm 0.070$ & $0.609 \pm 0.083$ & $138.623 \pm 17.551$ & $2.135 \pm 0.397$ & $0.386 \pm 0.019$ & $\mathbf{+0.568 \pm 0.021}$ & $6.1$ \\
\qquad LorD $r{=}1$ & KS & $0.130 \pm 0.001$ & $0.180 \pm 0.003$ & $44.049 \pm 0.093$ & $0.205 \pm 0.009$ & $0.112 \pm 0.000$ & $+248.586 \pm 168.724$ & $3.4$ \\
\qquad SB & KS & $0.133 \pm 0.001$ & $0.179 \pm 0.009$ & $\mathbf{46.126 \pm 3.845}$ & $0.192 \pm 0.005$ & $0.115 \pm 0.003$ & - & $3.7$ \\
\multicolumn{9}{l}{\quad\emph{Spectral space}}\\
\qquad SCORE & CRPS+KS & $\mathbf{0.128 \pm 0.000}$ & $\mathbf{0.162 \pm 0.000}$ & $\mathbf{43.827 \pm 0.058}$ & $\mathbf{0.178 \pm 0.000}$ & $\mathbf{0.111 \pm 0.000}$ & $(+2.609 \pm 3.831) \times 10^{4}$ & $\mathbf{1.9}$ \\
\cmidrule(l){1-9}
\multicolumn{9}{l}{\textbf{$T = 192$}}\\
\multicolumn{9}{l}{\quad\emph{Physical space}}\\
\qquad Det & MSE & $0.149 \pm 0.001$ & $0.241 \pm 0.000$ & $93.734 \pm 0.186$ & $0.221 \pm 0.001$ & $0.136 \pm 0.001$ & - & $5.1$ \\
\qquad Diag & NLL & $0.150 \pm 0.000$ & $0.186 \pm 0.001$ & $69.049 \pm 0.715$ & $0.265 \pm 0.016$ & $0.133 \pm 0.003$ & $\mathbf{+0.989 \pm 0.071}$ & $4.2$ \\
\qquad Diag & KS & $0.148 \pm 0.000$ & $0.193 \pm 0.001$ & $\mathbf{66.774 \pm 0.090}$ & $0.224 \pm 0.005$ & $0.126 \pm 0.000$ & $+121.481 \pm 29.306$ & $3.4$ \\
\qquad LorD $r{=}1$ & NLL & $0.367 \pm 0.023$ & $0.621 \pm 0.016$ & $197.512 \pm 4.880$ & $1.952 \pm 0.236$ & $0.391 \pm 0.003$ & $\mathbf{+0.972 \pm 0.086}$ & $6.1$ \\
\qquad LorD $r{=}1$ & KS & $0.148 \pm 0.001$ & $0.191 \pm 0.001$ & $\mathbf{66.831 \pm 0.193}$ & $0.224 \pm 0.004$ & $0.127 \pm 0.001$ & $+68.602 \pm 15.897$ & $3.4$ \\
\qquad SB & KS & $0.150 \pm 0.001$ & $0.187 \pm 0.001$ & $67.205 \pm 0.193$ & $0.212 \pm 0.001$ & $0.127 \pm 0.001$ & - & $3.1$ \\
\multicolumn{9}{l}{\quad\emph{Spectral space}}\\
\qquad SCORE & CRPS+KS & $\mathbf{0.145 \pm 0.001}$ & $\mathbf{0.174 \pm 0.000}$ & $\mathbf{66.626 \pm 0.111}$ & $\mathbf{0.200 \pm 0.001}$ & $\mathbf{0.125 \pm 0.000}$ & $(+4.339 \pm 6.268) \times 10^{4}$ & $\mathbf{1.8}$ \\
\bottomrule
\end{tabular}}
\end{table*}

\clearpage
\begin{table*}[ht]
\centering
\caption{Test scores on Traffic (multivariate, 862 series).}
\label{tab:traffic}
\small
\adjustbox{max width=\linewidth}{%
\begin{tabular}{llrrrrrrr}
\toprule
Method & Objective & MSE $\downarrow$ & CRPS $\downarrow$ & ES $\downarrow$ & VS $\downarrow$ & KS $\downarrow$ & NLL/$d$ $\downarrow$ & $\varnothing$ Rank $\downarrow$ \\
\midrule
\multicolumn{9}{l}{\textbf{$T = 96$}}\\
\multicolumn{9}{l}{\quad\emph{Physical space}}\\
\qquad Det & MSE & $0.400 \pm 0.001$ & $0.278 \pm 0.001$ & $179.013 \pm 0.136$ & $0.667 \pm 0.001$ & $0.322 \pm 0.000$ & - & $4.9$ \\
\qquad Diag & NLL & $0.405 \pm 0.001$ & $0.219 \pm 0.002$ & $128.521 \pm 0.326$ & $0.667 \pm 0.006$ & $0.273 \pm 0.001$ & $\mathbf{+0.321 \pm 0.012}$ & $3.6$ \\
\qquad Diag & KS & $0.401 \pm 0.001$ & $0.254 \pm 0.001$ & $\mathbf{127.611 \pm 0.221}$ & $0.698 \pm 0.002$ & $\mathbf{0.270 \pm 0.001}$ & $+0.838 \pm 0.001$ & $3.7$ \\
\qquad LorD $r{=}1$ & NLL & $0.429 \pm 0.012$ & $0.359 \pm 0.070$ & $179.878 \pm 32.056$ & $2.292 \pm 1.292$ & $0.377 \pm 0.036$ & $\mathbf{+0.290 \pm 0.029}$ & $6.0$ \\
\qquad LorD $r{=}1$ & KS & $0.402 \pm 0.001$ & $0.253 \pm 0.001$ & $\mathbf{127.676 \pm 0.234}$ & $0.698 \pm 0.002$ & $\mathbf{0.270 \pm 0.001}$ & $+0.797 \pm 0.003$ & $3.5$ \\
\qquad SB & KS & $0.399 \pm 0.005$ & $0.236 \pm 0.012$ & $\mathbf{127.214 \pm 0.404}$ & $0.683 \pm 0.014$ & $\mathbf{0.269 \pm 0.001}$ & - & $3.3$ \\
\multicolumn{9}{l}{\quad\emph{Spectral space}}\\
\qquad SCORE & CRPS+KS & $\mathbf{0.381 \pm 0.000}$ & $\mathbf{0.187 \pm 0.000}$ & $\mathbf{127.540 \pm 0.186}$ & $\mathbf{0.614 \pm 0.001}$ & $\mathbf{0.271 \pm 0.001}$ & $+1.332 \pm 0.089$ & $\mathbf{2.2}$ \\
\cmidrule(l){1-9}
\multicolumn{9}{l}{\textbf{$T = 192$}}\\
\multicolumn{9}{l}{\quad\emph{Physical space}}\\
\qquad Det & MSE & $0.415 \pm 0.001$ & $0.283 \pm 0.001$ & $259.284 \pm 0.396$ & $0.693 \pm 0.002$ & $0.334 \pm 0.001$ & - & $5.3$ \\
\qquad Diag & NLL & $0.418 \pm 0.003$ & $0.221 \pm 0.004$ & $185.288 \pm 0.640$ & $0.705 \pm 0.040$ & $0.281 \pm 0.001$ & $\mathbf{+0.406 \pm 0.017}$ & $3.2$ \\
\qquad Diag & KS & $0.415 \pm 0.000$ & $0.260 \pm 0.001$ & $\mathbf{184.419 \pm 0.083}$ & $0.725 \pm 0.003$ & $\mathbf{0.279 \pm 0.000}$ & $+0.874 \pm 0.002$ & $3.5$ \\
\qquad LorD $r{=}1$ & NLL & $0.499 \pm 0.040$ & $0.301 \pm 0.029$ & $212.771 \pm 8.583$ & $0.812 \pm 0.049$ & $0.340 \pm 0.013$ & $\mathbf{+0.390 \pm 0.016}$ & $5.8$ \\
\qquad LorD $r{=}1$ & KS & $0.415 \pm 0.000$ & $0.259 \pm 0.000$ & $\mathbf{184.372 \pm 0.105}$ & $0.723 \pm 0.001$ & $\mathbf{0.279 \pm 0.000}$ & $+0.851 \pm 0.007$ & $3.3$ \\
\qquad SB & KS & $0.417 \pm 0.002$ & $0.244 \pm 0.011$ & $\mathbf{184.294 \pm 0.384}$ & $0.716 \pm 0.007$ & $\mathbf{0.279 \pm 0.001}$ & - & $3.3$ \\
\multicolumn{9}{l}{\quad\emph{Spectral space}}\\
\qquad SCORE & CRPS+KS & $\mathbf{0.398 \pm 0.001}$ & $\mathbf{0.193 \pm 0.000}$ & $185.191 \pm 0.319$ & $\mathbf{0.639 \pm 0.001}$ & $0.283 \pm 0.001$ & $+1.424 \pm 0.085$ & $\mathbf{2.8}$ \\
\bottomrule
\end{tabular}}
\end{table*}

\paragraph{Surface temperature prediction.}
On ERA5 (Table~\ref{tab:era5}), \score{} attains the best or tied-best mean on every metric and the best average rank. On MSE and CRPS it coincides with the NLL-trained diagonal model. By Lemma~\ref{lem:corr-core}, these two scores are those of the first stage of \score{}, a CRPS-trained diagonal model on the same backbone that uses 20 of the 30 epochs, so the first stage matches the NLL-trained marginals at two thirds of the budget. The gain from the correlation core shows in the scores that are sensitive to the dependence structure: relative to the NLL-trained diagonal model, the variogram score improves by $7\%$ and the log score by $1.3$ nats per coordinate, whereas the ES and KS differences are within about one standard deviation, in line with their weaker sensitivity to dependence (Appendix~\ref{app:eval}). Among the baselines, the generative sampler improves CRPS and ES over the deterministic model but not MSE. LorD trained with the NLL collapses. Its mean remains accurate, but its CRPS deteriorates by a factor of four and its VS by more than a factor of sixty, while its log score is better than that of the diagonal model, which indicates a degenerate covariance rather than a poor mean. Trained with the kernel score, LorD is on par with the KS-trained diagonal model but behind the NLL-trained one on CRPS and VS. The assumptions of \score{} are well justified on this task, since temperature anomalies on a regular grid have a distance-dependent, approximately stationary correlation, which a circulant core captures.

\begin{table*}[ht]
\centering
\caption{Test scores on ERA5 surface temperature.}
\label{tab:era5}
\small
\adjustbox{max width=\linewidth}{%
\begin{tabular}{llrrrrrrr}
\toprule
Method & Objective & MSE $\downarrow$ & CRPS $\downarrow$ & ES $\downarrow$ & VS $\downarrow$ & KS $\downarrow$ & NLL/$d$ $\downarrow$ & $\varnothing$ Rank $\downarrow$ \\
\midrule
\multicolumn{9}{l}{\emph{Physical space}}\\
Det & MSE & $\mathbf{0.080 \pm 0.001}$ & $0.206 \pm 0.001$ & $51.881 \pm 0.197$ & $0.086 \pm 0.001$ & $0.076 \pm 0.001$ & - & $5.6$ \\
Diag & NLL & $\mathbf{0.079 \pm 0.001}$ & $\mathbf{0.144 \pm 0.001}$ & $\mathbf{36.701 \pm 0.155}$ & $0.072 \pm 0.001$ & $\mathbf{0.072 \pm 0.001}$ & $-0.052 \pm 0.004$ & $2.3$ \\
Diag & KS & $0.080 \pm 0.000$ & $0.156 \pm 0.002$ & $37.047 \pm 0.089$ & $0.081 \pm 0.002$ & $0.074 \pm 0.000$ & $+0.442 \pm 0.147$ & $4.1$ \\
LorD $r{=}1$ & NLL & $0.082 \pm 0.001$ & $0.631 \pm 0.037$ & $153.696 \pm 10.228$ & $4.605 \pm 0.579$ & $0.383 \pm 0.008$ & $-0.113 \pm 0.002$ & $5.9$ \\
LorD $r{=}1$ & KS & $\mathbf{0.080 \pm 0.001}$ & $0.155 \pm 0.003$ & $36.981 \pm 0.189$ & $0.079 \pm 0.002$ & $\mathbf{0.073 \pm 0.001}$ & $+0.487 \pm 0.315$ & $3.8$ \\
SB & KS & $0.081 \pm 0.000$ & $0.154 \pm 0.002$ & $36.993 \pm 0.079$ & $0.080 \pm 0.002$ & $0.073 \pm 0.000$ & - & $4.1$ \\
\midrule
\multicolumn{9}{l}{\emph{Spectral space}}\\
SCORE & CRPS+KS & $\mathbf{0.079 \pm 0.000}$ & $\mathbf{0.144 \pm 0.000}$ & $\mathbf{36.544 \pm 0.056}$ & $\mathbf{0.067 \pm 0.000}$ & $\mathbf{0.072 \pm 0.000}$ & $\mathbf{-1.334 \pm 0.046}$ & $\mathbf{1.6}$ \\
\bottomrule
\end{tabular}}
\end{table*}

\FloatBarrier

\paragraph{Monocular depth estimation.}
On NYU Depth~v2 (Table~\ref{tab:depth}), \score{} performs poorly across scores, ranking behind the diagonal model on every sample-based score. Part of the gap is already in the marginals. The first stage of \score{}, trained with the CRPS for 50 epochs, does not reach the diagonal model trained with the log score for 75 epochs, with CRPS $0.0204$ against $0.0192$, and the second stage cannot recover this by construction (Lemma~\ref{lem:corr-core}). The remaining gap on ES and VS means that the learned correlation is worse than no correlation at all, which points to the assumptions of \score{}. Depth maps are not stationary. Their correlation is created at object boundaries, so a single circulant kernel couples pixels across edges that are unrelated. Moreover, the log score of \score{} explodes. A plausible explanation for this is that a core trained with the bounded kernel score has no incentive to keep the predicted covariance well conditioned (Theorem~\ref{thm:ks-lipschitz}), while the log score penalizes directions of small predicted variance without bound (Section~\ref{sec:nll}). 

\begin{table*}[ht]
\centering
\caption{Test scores on NYU Depth v2. For readability MSE, CRPS, VS, KS in units
of $10^{-3}$.}
\label{tab:depth}
\small
\adjustbox{max width=\linewidth}{%
\begin{tabular}{llrrrrrrr}
\toprule
Method & Objective & MSE $\downarrow$ & CRPS $\downarrow$ & ES $\downarrow$ & VS $\downarrow$ & KS $\downarrow$ & NLL/$d$ $\downarrow$ & $\varnothing$ Rank $\downarrow$ \\
\midrule
\multicolumn{9}{l}{\emph{Physical space}}\\
Det & MSE & $\mathbf{1.847 \pm 0.018}$ & $24.983 \pm 0.170$ & $17.365 \pm 0.078$ & $2.299 \pm 0.018$ & $\mathbf{1.839 \pm 0.018}$ & - & $4.2$ \\
Diag & NLL & $\mathbf{1.798 \pm 0.041}$ & $\mathbf{19.241 \pm 0.334}$ & $\mathbf{12.746 \pm 0.167}$ & $\mathbf{2.151 \pm 0.053}$ & $\mathbf{1.787 \pm 0.041}$ & $+1.050 \pm 0.722$ & $\mathbf{2.2}$ \\
Diag & KS & $\mathbf{1.838 \pm 0.031}$ & $\mathbf{19.202 \pm 0.214}$ & $13.098 \pm 0.119$ & $2.290 \pm 0.045$ & $\mathbf{1.829 \pm 0.031}$ & $-1.686 \pm 0.061$ & $2.8$ \\
LorD $r{=}1$ & NLL & $63.015 \pm 6.329$ & $216.119 \pm 22.704$ & $112.163 \pm 11.859$ & $30.344 \pm 12.902$ & $130.097 \pm 23.341$ & $-0.058 \pm 1.080$ & $6.6$ \\
LorD $r{=}1$ & KS & $\mathbf{1.818 \pm 0.017}$ & $\mathbf{19.018 \pm 0.075}$ & $12.998 \pm 0.060$ & $2.264 \pm 0.022$ & $\mathbf{1.809 \pm 0.016}$ & $\mathbf{-1.849 \pm 0.045}$ & $2.6$ \\
SB & KS & $\mathbf{1.853 \pm 0.038}$ & $25.052 \pm 0.271$ & $17.335 \pm 0.143$ & $2.275 \pm 0.029$ & $\mathbf{1.844 \pm 0.038}$ & - & $4.4$ \\
\midrule
\multicolumn{9}{l}{\emph{Spectral space}}\\
SCORE & CRPS+KS & $1.907 \pm 0.021$ & $20.444 \pm 0.103$ & $13.735 \pm 0.079$ & $2.306 \pm 0.033$ & $1.897 \pm 0.022$ & $(+3.709 \pm 7.977) \times 10^{3}$ & $4.6$ \\
\bottomrule
\end{tabular}}
\end{table*}

\FloatBarrier

\paragraph{Graph-based temperature post-processing.}
This task is the only one where we consider a different transform, namely the Graph Fourier transform. While \score{} naturally extends, the resulting predictive distribution requires a Cholesky decomposition for loss optimization, thus offering no additional computational benefits. Therefore, the results validate our method against fully expressive alternatives that can recover the correct structure, i.e., Cholesky decomposition and the sample-based model. \autoref{tab:euppbench} shows that the latter takes the best rank and leads to the best performance overall, while the Cholesky decomposition obtains the best score in terms of the VS and NLL. Yet, our method remains competitive and ranks third overall, with the individual metrics being close to the sample-based method and better or comparable to the Cholesky method. However, here LorD trained with the kernel score outperforms \score{} and ranks second overall. However, as shown in \autoref{fig:main_prediction_example}, LorD estimates many correlations between stations that are far away, and thus are most likely not actually correlated for the specific prediction. \score{} on the other hand, provides more clustered and local dependencies.

\begin{table*}[ht]
\centering
\caption{Test scores on EUPPBench station post-processing.}
\label{tab:euppbench}
\small
\adjustbox{max width=\linewidth}{%
\begin{tabular}{llrrrrrrr}
\toprule
Method & Objective & MSE $\downarrow$ & CRPS $\downarrow$ & ES $\downarrow$ & VS $\downarrow$ & KS $\downarrow$ & NLL/$d$ $\downarrow$ & $\varnothing$ Rank $\downarrow$ \\
\midrule
\multicolumn{9}{l}{\emph{Physical space}}\\
Det & MSE & $2.144 \pm 0.019$ & $1.127 \pm 0.005$ & $15.078 \pm 0.070$ & $2.101 \pm 0.005$ & $0.811 \pm 0.003$ & - & $8.2$ \\
Diag & NLL & $2.147 \pm 0.011$ & $0.801 \pm 0.002$ & $10.800 \pm 0.026$ & $1.942 \pm 0.029$ & $0.479 \pm 0.001$ & $+1.736 \pm 0.004$ & $6.0$ \\
Diag & KS & $\mathbf{2.100 \pm 0.020}$ & $0.793 \pm 0.003$ & $10.653 \pm 0.045$ & $1.867 \pm 0.027$ & $0.476 \pm 0.001$ & $+1.750 \pm 0.005$ & $4.5$ \\
LorD $r{=}1$ & NLL & $2.267 \pm 0.049$ & $0.830 \pm 0.017$ & $11.073 \pm 0.202$ & $1.757 \pm 0.009$ & $0.478 \pm 0.002$ & $+1.582 \pm 0.003$ & $5.8$ \\
LorD $r{=}1$ & KS & $\mathbf{2.106 \pm 0.026}$ & $\mathbf{0.795 \pm 0.006}$ & $\mathbf{10.614 \pm 0.061}$ & $1.796 \pm 0.016$ & $\mathbf{0.473 \pm 0.001}$ & $+1.669 \pm 0.003$ & $3.4$ \\
Chol & NLL & $2.335 \pm 0.099$ & $0.836 \pm 0.019$ & $11.129 \pm 0.225$ & $\mathbf{1.714 \pm 0.020}$ & $0.476 \pm 0.002$ & $\mathbf{+1.339 \pm 0.009}$ & $5.1$ \\
Chol & KS & $\mathbf{2.109 \pm 0.030}$ & $0.797 \pm 0.005$ & $\mathbf{10.636 \pm 0.077}$ & $1.783 \pm 0.012$ & $\mathbf{0.474 \pm 0.001}$ & $+331.553 \pm 412.970$ & $4.3$ \\
SB & KS & $\mathbf{2.095 \pm 0.007}$ & $\mathbf{0.788 \pm 0.001}$ & $\mathbf{10.553 \pm 0.018}$ & $1.788 \pm 0.005$ & $\mathbf{0.473 \pm 0.000}$ & - & $\mathbf{2.6}$ \\
\midrule
\multicolumn{9}{l}{\emph{Spectral space}}\\
SCORE & CRPS+KS & $2.136 \pm 0.009$ & $0.799 \pm 0.002$ & $10.708 \pm 0.021$ & $1.775 \pm 0.006$ & $0.475 \pm 0.000$ & $+1.495 \pm 0.005$ & $4.0$ \\
\bottomrule
\end{tabular}}
\end{table*}

\FloatBarrier
\subsection{Runtime analysis}
Here, we further analyze the runtime of the different methods and show that our method outperforms the baselines in terms of the compute-performance tradeoff. 
Similar to \autoref{fig:performance_runtime}, \autoref{fig:runtime_ts} and \autoref{fig:runtime_other} show the increase in computational cost versus the performance improvement against the diagonal baseline, per dataset. In addition, \autoref{tab:runtime-ms-per-step} shows the average runtime per task in ms per gradient step. Here, all methods use the kernel score as an objective to ensure a fair runtime comparison. Since this objective outperforms the log score significantly (compare \autoref{tab:performance_significance}), it does not favor the performance of our method.
The figure shows that our method generally obtains the best performance improvement while having the smallest computational increase against the diagonal baseline. These are particularly strong for VS and NLL, but consistent across all metrics. While for some datasets, namely Depth, EUPPBench, and Weather, the gains are not across all metrics and are even negative for some, our method is still the fastest among the baselines that model dependence. This is also highlighted in \autoref{tab:runtime-ms-per-step}. While our method has a compute advantage due to the second stage with frozen weights, the table also shows that using only the first stage, \score{} is comparable to the deterministic or diagonal method. Further, the ablations in \ref{app:ablations} highlight that the second stage improves performance even under decreasing epoch budgets.

%
%
\begin{figure}[ht]
    \centering
    \includegraphics[width=\linewidth]{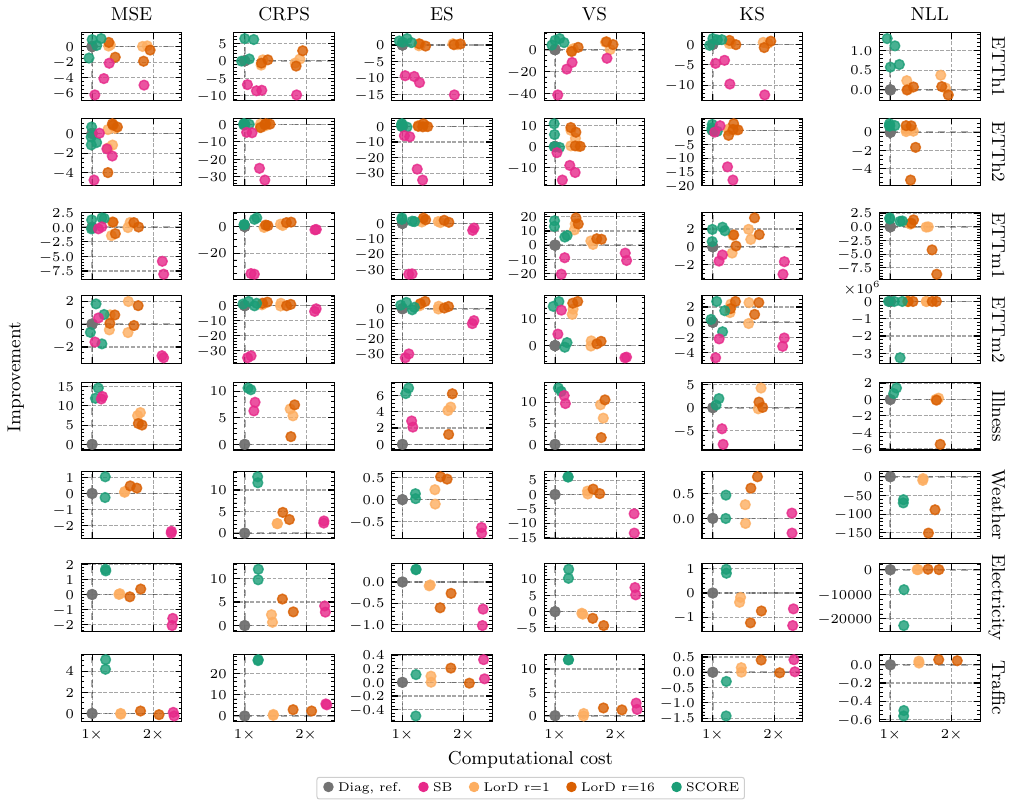}
    \caption{Comparison of relative improvement and computational cost against the diagonal Gaussian baseline for the time series datasets. All methods are trained on the Gaussian kernel score. Each point represents the median performance over five seeds for one task, i.e., uni-/multivariate and prediction horizon $T$.}
    \label{fig:runtime_ts}
\end{figure}

\begin{figure}[ht]
    \centering
    \includegraphics[width=\linewidth]{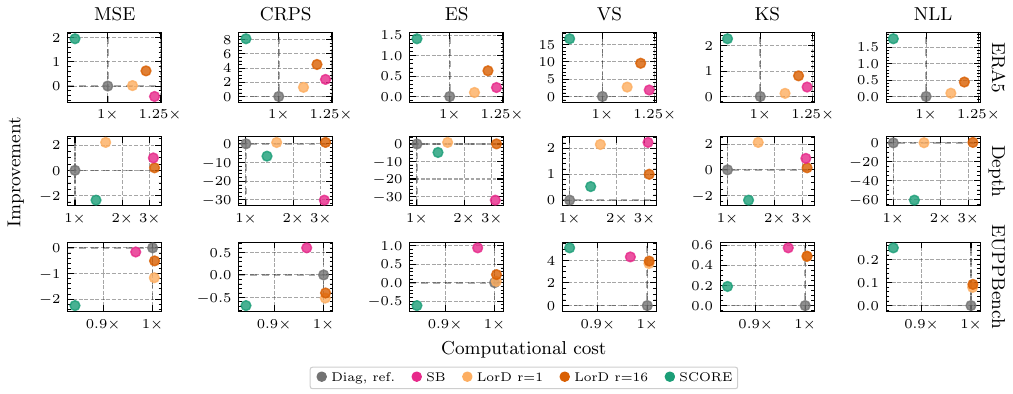}
    \caption{Comparison of relative improvement and computational cost against the diagonal Gaussian baseline for the remaining datasets. All methods are trained on the Gaussian kernel score. Each point represents the median performance over five seeds.}    \label{fig:runtime_other}
\end{figure}

\begin{table*}
  \centering
  \caption{Training cost of every predictive family on every benchmark task, in milliseconds per training step. Mean $\pm$ standard deviation over 5 seeds on the exact same compute requirements, epoch budget and equal batch size within a task. For our SCORE method, the stages are averaged by the number of epochs used, respectively. The fastest method overall is highlighted in italics, while the fastest method that models a correlation structure is highlighted in bold.}
  \label{tab:runtime-ms-per-step}
\adjustbox{max width=\linewidth}{%
  \begin{tabular}{l cccccccc}
    \toprule
    Task & Det & Diag & LorD $r{=}1$ & LorD $r{=}16$ & Chol & SB & SCORE & \begin{tabular}[b]{@{}c@{}}SCORE\\(stage 1)\end{tabular} \\
    \midrule
    ETTh1 (multiv.), $T{=}96$ & $\mathit{12 \pm 0.088}$ & $20 \pm 0.989$ & $36 \pm 0.208$ & $37 \pm 0.162$ & -- & $35 \pm 0.206$ & $\mathbf{21 \pm 0.189}$ & $21 \pm 0.227$ \\
    ETTh1 (multiv.), $T{=}192$ & $\mathit{24 \pm 0.422}$ & $33 \pm 0.217$ & $41 \pm 0.461$ & $43 \pm 0.898$ & -- & $40 \pm 0.087$ & $\mathbf{33 \pm 0.336}$ & $33 \pm 0.013$ \\
    ETTh1 (univ.), $T{=}96$ & $\mathit{11 \pm 0.125}$ & $19 \pm 0.372$ & $34 \pm 0.278$ & $35 \pm 0.247$ & $33 \pm 0.089$ & $22 \pm 0.308$ & $\mathbf{21 \pm 0.334}$ & $20 \pm 0.362$ \\
    ETTh1 (univ.), $T{=}192$ & $\mathit{24 \pm 0.242}$ & $34 \pm 0.972$ & $42 \pm 0.098$ & $41 \pm 1.012$ & $51 \pm 0.474$ & $35 \pm 0.202$ & $\mathbf{34 \pm 1.099}$ & $33 \pm 1.645$ \\
    ETTh2 (multiv.), $T{=}96$ & $\mathit{24 \pm 0.066}$ & $34 \pm 0.992$ & $41 \pm 0.512$ & $42 \pm 0.772$ & -- & $40 \pm 0.444$ & $\mathbf{33 \pm 0.049}$ & $33 \pm 0.265$ \\
    ETTh2 (multiv.), $T{=}192$ & $\mathit{24 \pm 0.367}$ & $32 \pm 1.256$ & $41 \pm 0.852$ & $42 \pm 0.376$ & -- & $40 \pm 0.110$ & $\mathbf{34 \pm 0.402}$ & $33 \pm 0.542$ \\
    ETTh2 (univ.), $T{=}96$ & $\mathit{24 \pm 0.323}$ & $34 \pm 1.389$ & $42 \pm 1.108$ & $42 \pm 0.332$ & $40 \pm 0.610$ & $35 \pm 0.915$ & $\mathbf{34 \pm 1.540}$ & $33 \pm 1.712$ \\
    ETTh2 (univ.), $T{=}192$ & $\mathit{23 \pm 0.028}$ & $33 \pm 0.049$ & $42 \pm 1.473$ & $41 \pm 0.714$ & $52 \pm 0.485$ & $36 \pm 0.444$ & $\mathbf{33 \pm 1.099}$ & $32 \pm 1.004$ \\
    ETTm1 (multiv.), $T{=}96$ & $\mathit{43 \pm 0.240}$ & $52 \pm 0.156$ & $78 \pm 0.202$ & $84 \pm 0.351$ & -- & $115 \pm 0.829$ & $\mathbf{58 \pm 0.079}$ & $62 \pm 0.095$ \\
    ETTm1 (multiv.), $T{=}192$ & $\mathit{44 \pm 0.132}$ & $52 \pm 0.094$ & $79 \pm 0.751$ & $88 \pm 0.217$ & -- & $116 \pm 0.438$ & $\mathbf{59 \pm 0.523}$ & $63 \pm 0.604$ \\
    ETTm1 (univ.), $T{=}96$ & $\mathit{24 \pm 0.050}$ & $33 \pm 0.098$ & $42 \pm 1.182$ & $42 \pm 1.591$ & -- & $37 \pm 0.348$ & $\mathbf{33 \pm 1.145}$ & $33 \pm 1.025$ \\
    ETTm1 (univ.), $T{=}192$ & $\mathit{24 \pm 0.356}$ & $33 \pm 0.169$ & $41 \pm 0.070$ & $43 \pm 0.805$ & -- & $36 \pm 0.346$ & $\mathbf{33 \pm 0.512}$ & $33 \pm 0.499$ \\
    ETTm2 (multiv.), $T{=}96$ & $\mathit{44 \pm 0.189}$ & $52 \pm 0.309$ & $78 \pm 0.167$ & $84 \pm 0.693$ & -- & $115 \pm 0.745$ & $\mathbf{58 \pm 0.157}$ & $62 \pm 0.157$ \\
    ETTm2 (multiv.), $T{=}192$ & $\mathit{44 \pm 0.205}$ & $52 \pm 0.196$ & $79 \pm 0.800$ & $88 \pm 0.250$ & -- & $116 \pm 0.583$ & $\mathbf{59 \pm 0.621}$ & $63 \pm 0.659$ \\
    ETTm2 (univ.), $T{=}96$ & $\mathit{25 \pm 0.385}$ & $33 \pm 0.053$ & $41 \pm 0.846$ & $42 \pm 0.265$ & -- & $36 \pm 0.557$ & $\mathbf{34 \pm 1.091}$ & $33 \pm 0.873$ \\
    ETTm2 (univ.), $T{=}192$ & $\mathit{24 \pm 0.388}$ & $34 \pm 1.623$ & $42 \pm 0.457$ & $42 \pm 0.552$ & -- & $36 \pm 0.889$ & $\mathbf{34 \pm 0.395}$ & $34 \pm 0.226$ \\
    Illness, $T{=}24$ & $\mathit{18 \pm 0.446}$ & $27 \pm 2.737$ & $48 \pm 1.243$ & $50 \pm 1.018$ & -- & $32 \pm 0.968$ & $\mathbf{30 \pm 1.086}$ & $29 \pm 1.352$ \\
    Illness, $T{=}36$ & $\mathit{11 \pm 1.039}$ & $26 \pm 4.785$ & $48 \pm 2.059$ & $47 \pm 1.360$ & -- & $30 \pm 2.105$ & $\mathbf{29 \pm 1.575}$ & $28 \pm 2.417$ \\
    Weather, $T{=}96$ & $\mathit{125 \pm 0.775}$ & $132 \pm 0.950$ & $191 \pm 1.115$ & $202 \pm 2.475$ & -- & $323 \pm 1.299$ & $\mathbf{153 \pm 0.741}$ & $169 \pm 0.972$ \\
    Weather, $T{=}192$ & $\mathit{124 \pm 0.884}$ & $133 \pm 1.001$ & $192 \pm 1.523$ & $220 \pm 1.185$ & -- & $325 \pm 0.677$ & $\mathbf{154 \pm 0.736}$ & $171 \pm 0.933$ \\
    Electricity, $T{=}96$ & $\mathit{478 \pm 2.673}$ & $486 \pm 2.578$ & $658 \pm 3.645$ & $746 \pm 2.115$ & -- & $1202 \pm 2.311$ & $\mathbf{568 \pm 1.521}$ & $642 \pm 1.653$ \\
    Electricity, $T{=}192$ & $\mathit{474 \pm 2.831}$ & $485 \pm 0.582$ & $664 \pm 1.499$ & $840 \pm 0.764$ & -- & $1205 \pm 2.273$ & $\mathbf{565 \pm 0.915}$ & $636 \pm 0.936$ \\
    Traffic, $T{=}96$ & $\mathit{314 \pm 1.657}$ & $322 \pm 0.575$ & $445 \pm 0.812$ & $557 \pm 0.526$ & -- & $807 \pm 1.332$ & $\mathbf{376 \pm 0.971}$ & $422 \pm 1.136$ \\
    Traffic, $T{=}192$ & $\mathit{315 \pm 1.116}$ & $326 \pm 1.316$ & $451 \pm 0.557$ & $694 \pm 0.817$ & -- & $824 \pm 1.250$ & $\mathbf{379 \pm 0.497}$ & $426 \pm 0.653$ \\
    ERA5 $t_{2\mathrm{m}}$ & $637 \pm 5.115$ & $641 \pm 3.880$ & $713 \pm 4.156$ & $755 \pm 4.311$ & -- & $783 \pm 2.407$ & {\boldmath$\mathit{558 \pm 1.041}$} & $697 \pm 1.657$ \\
    NYU depth (BTS) & $\mathit{187 \pm 0.301}$ & $194 \pm 0.232$ & $304 \pm 0.154$ & $622 \pm 0.394$ & -- & $611 \pm 1.174$ & $\mathbf{264 \pm 0.313}$ & $285 \pm 0.411$ \\
    EUPPBench & $272 \pm 0.911$ & $287 \pm 1.133$ & $289 \pm 0.478$ & $289 \pm 1.659$ & $288 \pm 1.548$ & $277 \pm 0.808$ & {\boldmath$\mathit{243 \pm 0.789}$} & $274 \pm 0.917$ \\
    \bottomrule
  \end{tabular}
  }
\end{table*}

\subsection{Calibration results}
Here, we provide additional results regarding the calibration of our method and the baselines, as described in \ref{app:calibration}. To ensure that the results mainly reflect the covariance parameterization, we compare all methods using the kernel score as an objective. Since the kernel score beats the NLL in performance significantly, this is to no additional advantage for our method. In general, we only evaluate the calibration for predictive tasks with a single channel, i.e., we omit the multivariate time series task, as one would need to analyze each channel separately.

\autoref{fig:calibration_mean_scale} shows diagnostics of the marginal, mean, and scale calibration for the different methods and datasets. In general, almost all methods admit a good marginal calibration. However, \score{} overall shows the best mean and scale calibration, similar to the Cholesky approximation, which however is only available for the low-dimensional tasks. The sampling-based method collapses to a point prediction on multiple datasets and thus admits only poor calibration in these cases.
Similarly, \autoref{fig:calibration_dependence} shows the calibration diagnostics for the dependence structure across multiple lags $h$. Here, none of the methods are well calibrated, especially for very small lags. Still, \score{} admits better-calibrated predictions compared to the other methods, especially for larger lags and for the ETTh and EUPPBench datasets. This generally highlights the difficulty of learning a calibrated approximate dependence structure across complex domains, as no method is able to truly recover it.

\begin{figure}
    \centering
    \includegraphics[width=\linewidth]{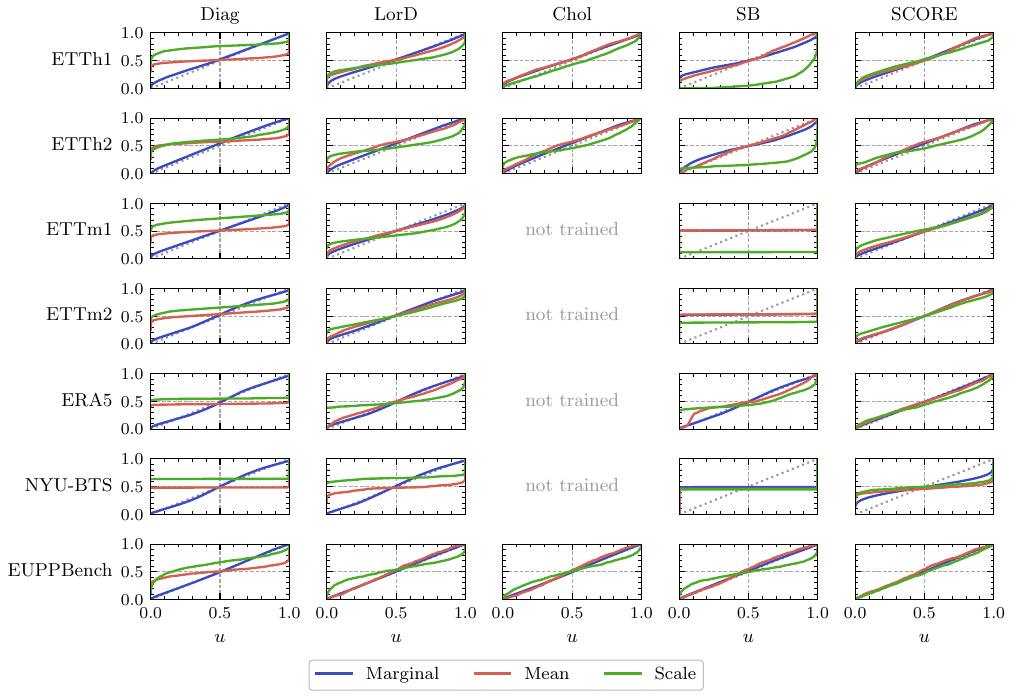}
    \caption{Empirical CDFs of the PIT values of the marginal, location and scale pre-rank functions for the different predictive methods. A calibrated prediction lies on the diagonal. Results are aggregated over five seeds.}
    \label{fig:calibration_mean_scale}
\end{figure}

\begin{figure}
    \centering
    \includegraphics[width=\linewidth]{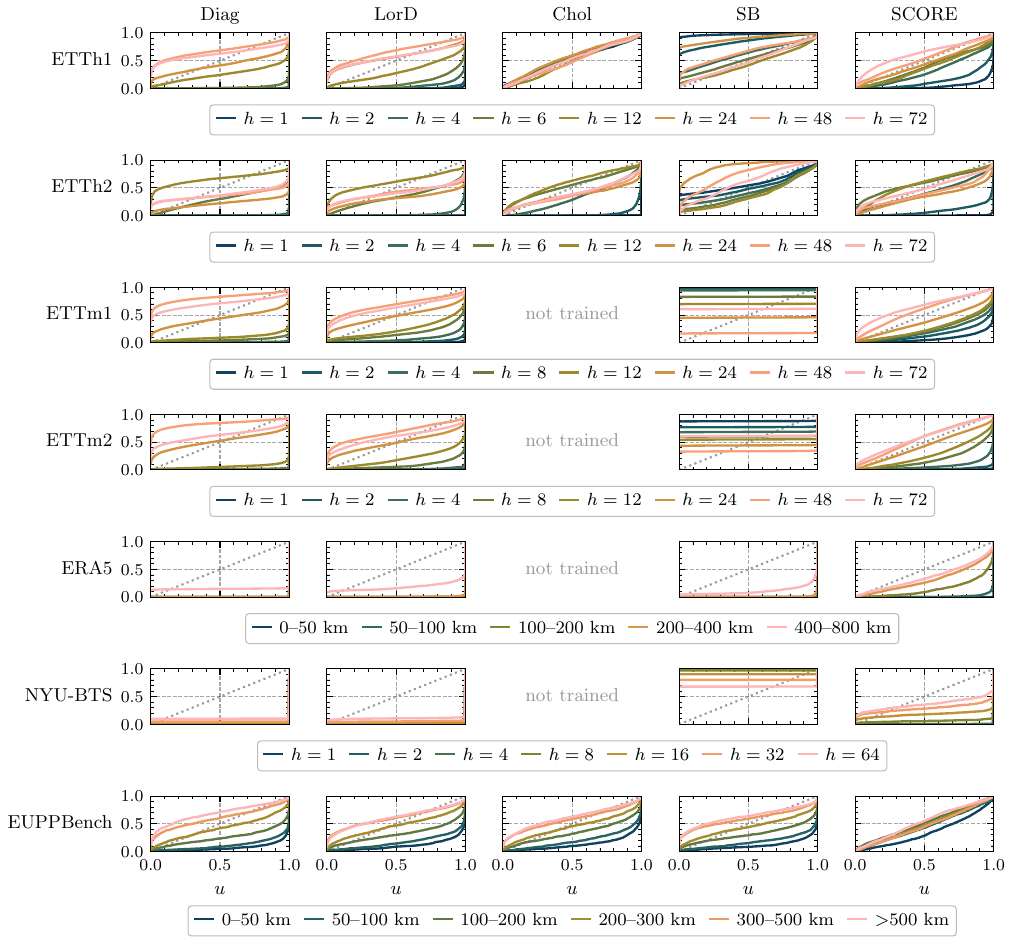}
    \caption{Empirical CDFs of the PIT values of the variogram pre-rank function with respect to different spatial (or temporal) lags $h$ for the different predictive methods. A calibrated prediction lies on the diagonal. Results are aggregated over five seeds.}
    \label{fig:calibration_dependence}
\end{figure}

\FloatBarrier
\subsection{Ablation studies}
\label{app:ablations}

\paragraph{Log score vs.\ kernel score.}
Here we provide details on the results in \autoref{tab:performance_significance}, which shows that training with the kernel score instead of the log score yields better performance in the large majority of comparisons, on four evaluation metrics.
For each covariance parameterization $a$ and task $t$, we compare the seed-averaged scores $\bar s^{\mathrm{KS}}_m(a,t)$ and $\bar s^{\mathrm{NLL}}_m(a,t)$ of the two arms, which differ only in the objective.
We do this for each metric $m$ other than the two training losses, giving $167$ comparisons per metric (all $(a,t)$ pairs for which both arms were trained).
Since metrics are not on a common scale across tasks, we aggregate signs rather than magnitudes, which is scale-free and conservative.
Since the tasks of one data set share data and hyperparameters, we treat the $C = 11$ data sets as clusters and assume independence between data sets.
With $\mathcal{T}_c$ the comparisons of data set $c$, let
\[
  S_c = \sum_{(a,t) \in \mathcal{T}_c} \operatorname{sign}\bigl(\bar s^{\mathrm{NLL}}_m(a,t) - \bar s^{\mathrm{KS}}_m(a,t)\bigr),
  \qquad S = \sum_{c} S_c ,
\]
where ties contribute zero (none occurred).
Under the sharp null hypothesis that the objective has no effect, the labels KS and NLL are exchangeable within each data set.
Swapping them maps $S_c \mapsto -S_c$ and keeps any within-cluster dependence intact, so $(S_c)_c \overset{d}{=} (\varepsilon_c S_c)_c$ for all $\varepsilon \in \{-1,1\}^C$.
This is a cluster-level sign-flip randomization test \citep[Ch.~15]{lehmann2005testing}, the sign-test analogue of the clustered signed-rank tests of \citet{rosner2006wilcoxon, datta2008signed}; see also \citet{hemerik2018exact}.
The exact two-sided $p$-value is
\[
  p = 2^{-C} \, \Bigl|\Bigl\{\varepsilon : \Bigl|\sum_{c} \varepsilon_c S_c\Bigr| \geq |S| \Bigr\}\Bigr| \;\geq\; 2/2^{C} \approx 0.001 .
\]
This floor is reached on MSE, CRPS and ES, where every data set has a kernel-score majority, and remains significant after a Holm correction over the four metrics.
On VS, EUPPBench favors the log score ($S_c = -5$), which gives $p = 10/2048 \approx 0.005$.

\paragraph{Loss and low-rank comparison.}
This ablation aims to disentangle the performance gains of the different loss functions from our proposed covariance approximation. Here we focus on the analysis of the low-rank method and simultaneously assess the impact of the size of the rank $r$.
To achieve this, we train the low-rank method with increasing ranks $r \in \{1,2,4,8,16\}$ for both the NLL and the kernel score, again across five different seeds. We then analyze, for each metric, the relative performance to our method, as shown in \autoref{fig:rank_ablation}. Using the kernel score as a loss function leads to significantly better performance for the LorD approximation across all metrics. However, the performance does not seem to depend on increasing rank $r$. Using the NLL as an objective leads to a much higher relative loss, but it decreases with larger $r$, although not for the other metrics.
\begin{figure}
    \centering
    \includegraphics[width=\linewidth]{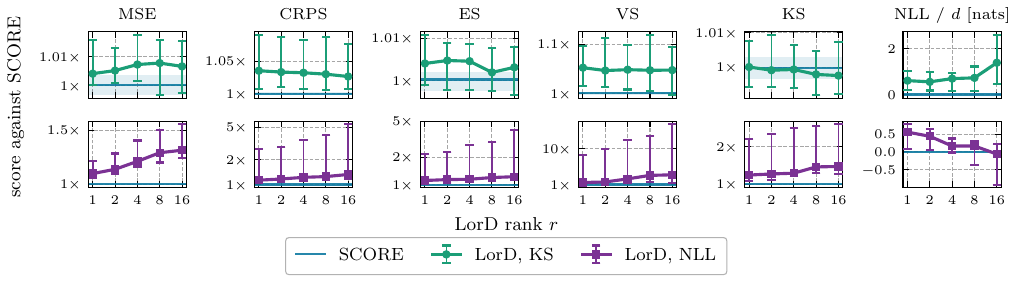}
    \caption{Relative performance of the low-rank method against ours with increasing rank $r$.}
    \label{fig:rank_ablation}
\end{figure}

\paragraph{Impact of stage 2 training and epoch budget.}
Here, we provide an additional ablation to verify that our staged training approach brings additional benefits compared to simply training the first stage. At the same time, we vary the total training budget to make sure that the additional performance gains do not depend on the chosen number of epochs. For computational reasons, we only provide this ablation for the univariate time series and the EUPPBench task, which nevertheless present a clear picture, as can be seen in \autoref{fig:stage1_stage2_comparison}.
\begin{figure}
    \centering
    \includegraphics[width=\linewidth]{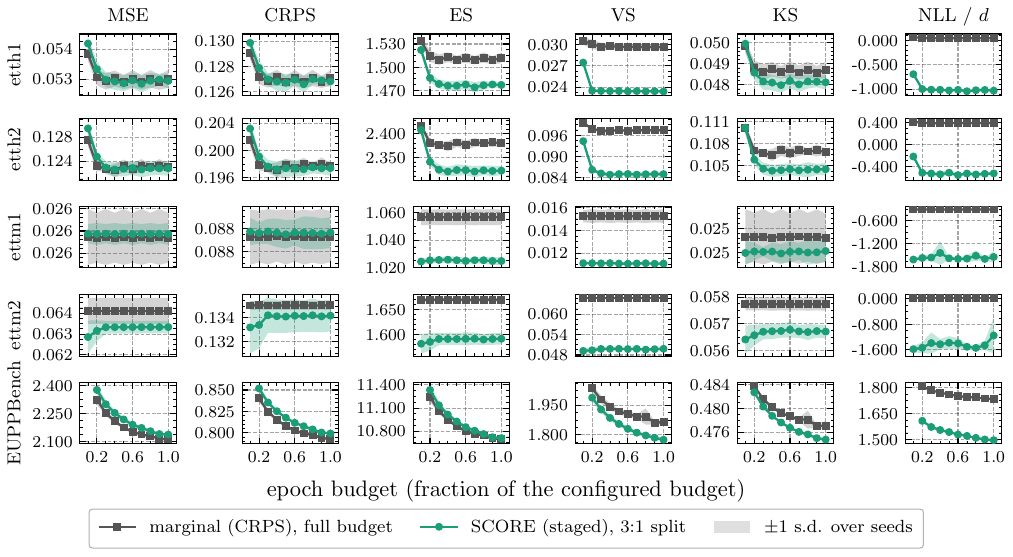}
    \caption{Performance of staged training (3:1 split) vs. first-stage only using the same computational budget.}
    \label{fig:stage1_stage2_comparison}
\end{figure}
These results highlight that, especially for the multivariate metrics, the second-stage training significantly increases performance beyond the marginal first stage. Furthermore, this is irrespective of the training budget, since even for very few epochs, where our model leads to poor marginal performance, the corresponding multivariate metrics are significantly lower compared to training only the first stage.
\end{document}